\documentclass[11 pt]{article}
\usepackage{graphicx} 
\usepackage{subcaption}
\usepackage{url}
\usepackage{tikz-cd}
\usepackage{tikz}
\usepackage{xcolor}
\usetikzlibrary{decorations.pathmorphing, calc}
\usepackage{amsmath}
\usepackage{amsthm}
\usepackage{amssymb}
\usepackage[rgb,dvipsnames]{xcolor}
\usepackage{titling}
\usepackage{diffcoeff,amssymb}
\usepackage{comment}
\usepackage{algorithm}
\usepackage{algpseudocode}
\usepackage{pifont}
\usepackage{amssymb}
\usepackage{mathtools}
\usepackage{mathrsfs}
\usepackage{dsfont}
\usepackage{bm}
\usepackage[title]{appendix}
\usepackage{thmtools,thm-restate}
\usepackage[utf8]{inputenc}
\usepackage[T1]{fontenc}
\usepackage[numbers]{natbib}
\usepackage[left=2.4cm,right=2.4cm,top=2.3cm,bottom=2cm]{geometry}
\usepackage{hyperref}
\usepackage{cleveref}

\newcommand{\probP}{\mathbb{P}}

\def \V {{\bm V}}

\def \c {{\bm c}}
\def \b {{\bm b}}

\def \e {{\bm e}}

\def \M {{\bm M}}
\def \a {{\bm a}}

\def \v {{\bm v}}
\def \w {{\bm w}}
\def \u {{\bm u}}
\def \x {{\bm w}}
\def \y {{\bm y}}
\def \z {{\bm x}}  
\def \h {{\bm h}}
\def \A {{\bm A}}

\def \r {{\bm r}}
\def \W {{\bm W}}

\def \X {{\bm X}}

\def \Z {{\bm X}}

\newcommand{\norm}[1]{\ensuremath{\left\|#1\right\|}}

\newtheorem{theorem}{Theorem}
\newtheorem{proposition}{Proposition}
\newtheorem{remark}{Remark}%
\newtheorem{corollary}{Corollary}
\newtheorem{definition}{Definition}
\newtheorem{lemma}{\bf Lemma}
\numberwithin{equation}{section}
\usepackage[parfill]{parskip}
\newcommand{\RSnew}[1]{{#1}}

\title{The Measure of Deception: An Analysis of Data Forging in Machine Unlearning}

\author{
  Rishabh Dixit\thanks{Department of Mathematics, UC San Diego 
  (ridixit@ucsd.edu).}
  \and
  Yuan Hui\thanks{Department of Mathematics, UC San Diego 
  (yuhui@ucsd.edu).}
  \and
  Rayan Saab\thanks{Department of Mathematics and Hal{\i}c{\i}o{g}lu Data Science Institute, UC San Diego 
  (rsaab@ucsd.edu).}
  \thanks{Authors listed in alphabetical order.}
}
\date{\vspace{-4ex}}

\begin{document}

\maketitle
\begin{abstract}
    Motivated by privacy regulations and the need to mitigate the effects of harmful data, machine unlearning seeks to modify trained models so that they effectively ``forget'' designated data. A key challenge in verifying unlearning is \emph{forging}---adversarially crafting data that mimics the gradient of a target point, thereby creating the appearance of unlearning without actually removing information. To capture this phenomenon, we consider the collection of data points whose gradients approximate a target gradient within tolerance $\epsilon$---which we call an $\epsilon$-forging set---and develop a framework for its analysis. For linear regression and one-layer neural networks, we show that the Lebesgue measure of this set is small. It scales on the order of $\epsilon$, and when $\epsilon$ is small enough, $\epsilon^d$. More generally, under mild regularity assumptions, we prove that the forging set measure decays as $\epsilon^{(d-r)/2}$, where $d$ is the data dimension and $r<d$ is the \RSnew{dimension of vector space of right singular vectors corresponding to ``small'' singular values} of a variation matrix defined by the model gradients. Extensions to batch SGD and almost-everywhere smooth loss functions yield the same asymptotic scaling. In addition, we establish probability bounds showing that, under non-degenerate data distributions, the likelihood of randomly sampling a forging point is vanishingly small. These results provide evidence that adversarial forging is fundamentally limited and that false unlearning claims can, in principle, be detected.
\end{abstract}

\section{Introduction}
Modern machine learning increasingly faces the requirement to forget specific training data---whether due to legal mandates such as the GDPR’s “right to be forgotten” \cite{mantelero2013eu} or user privacy requests. A widely adopted response to this challenge is machine unlearning \cite{bourtoule2021machine}\cite{nguyen2022survey}\cite{neel2021descent}\cite{gupta2021adaptive}\cite{cao2015towards}, which aims to modify a trained model as if certain data had never been seen. On the other hand, most of the existing machine unlearning algorithms rarely achieve true data erasure. Instead, they provide approximate guarantees---only ensuring that the updated model’s distribution resembles that of a model retrained without the data \cite{pawelczyk2024machine}\cite{thudi2022necessity}\cite{chowdhury2024towards}. As a result, fully retraining, with the target data removed from the training set, remains the rigorous solution in general. Since retraining a model from scratch is often prohibitively expensive, it creates a natural temptation to ``forge" a training trajectory, crafting an altered sequence that appears to comply with unlearning requests while leaving the final model largely unchanged \cite{thudi2022necessity}. 

From the perspective of a model trainer, the motivation to forge can be considerable. Reconstructing a trajectory that does not truly remove the targeted data but closely replicates the original gradient updates offers several advantages. First, the model’s utility is preserved, avoiding any degradation in performance due to stochastic retraining variability. Second, the computational cost of forging may be negligible compared to full retraining, especially in large-scale deep learning contexts where retraining costs can be immense. Thus, forging may appear to be a low-risk, high-reward, albeit unethical alternative to principled unlearning.

To further illustrate the incentives for forging, consider non-convex learning problems, which are ubiquitous in deep neural networks. Even minor changes in the training data can lead to qualitatively different models, as the optimization may converge to different local minima. This effect is particularly pronounced if the data to be removed occupies a meaningful subregion of the data space, such as a specific class or cluster, rather than being more uniformly distributed. In such cases, retraining without that data could easily yield a model that differs significantly from the original.

These two factors---the strong incentive to avoid retraining and the high likelihood of model drift due to principled unlearning---make forging a compelling albeit unethical strategy. While prior work has demonstrated that it is often possible to construct forged mini-batches that replicate original gradients with high precision \cite{thudi2022necessity}, we show that the set of such forging batches is vanishingly small in data space. That is, although forging is algorithmically feasible, it is statistically brittle: the probability of encountering forging batches under realistic data distributions is exceedingly low. Our work establishes the first quantitative framework for gradient-based data forging, thereby deepening the understanding of this phenomenon beyond recent results \cite{suliman2024data}. This has significant implications, both for the auditability of unlearning processes and for the potential to defend against deceptive forgeries---an area previously thought to be highly challenging \cite{zhang2024verification}.
Since the measure of forging batches (or a forging data point) is vanishingly small under any non-degenerate data distribution, an adversary attempting to forge must rely on highly atypical data points that deviate from the natural distribution. In a real-world unlearning audit, such deviations could be identified  through statistical distribution tests on purported training batches, for example. In effect, our results imply that gradient-forging attacks---while technically possible---require distributional anomalies that are inherently easy to identify, offering a potential line of defense previously  considered out of reach.

\subsection{Problem Setup}
To formalize \emph{data forging}, we consider a model trained to minimize an empirical loss function \( f(\bm{w}; \bm{x}) \), where \( \bm{w} \in \mathbb{R}^n \) denotes the model parameters and \( \bm{x} \in \mathbb{R}^d \) is a data point. Given a dataset \( D \), standard training via stochastic gradient descent (SGD) produces a sequence of iterates
\[
\bm{w}_{k+1} = \bm{w}_k - h_k \cdot \frac{1}{|B_k|} \sum_{\bm{x} \in B_k} \nabla_{\bm{w}} f(\bm{w}_k; \bm{x}),
\]
where \( B_k \subset D \) denotes the mini-batch used at step \( k \) and \( h_k \) is the learning rate. Suppose that a particular data point \( \bm{x}^* \in D \) must be removed (e.g., due to a deletion request). Instead of retraining from scratch on \( D \setminus \{ \bm{x}^* \} \), a model trainer may attempt to \emph{forge} a new sequence of mini-batches \( \{\widetilde{B}_k\} \), each disjoint from \( \bm{x}^* \), such that the resulting forged trajectory
\[
\widetilde{\bm{w}}_{k+1} = \widetilde{\bm{w}}_k - h_k \cdot \frac{1}{|\widetilde{B}_k|} \sum_{\bm{x} \in \widetilde{B}_k} \nabla_{\bm{w}} f(\widetilde{\bm{w}}_k; \bm{x})
\]
satisfies \( \|\widetilde{\bm{w}}_{k} - \bm{w}_{k}\| \leq \delta \) for all \( k \), with some small tolerance \( \delta \). A common strategy is \emph{gradient matching}, where each forged batch is selected to approximate the gradient of the original batch:
\[
\left\| \frac{1}{|\widetilde{B}_k|} \sum_{\bm{x} \in \widetilde{B}_k} \nabla_{\bm{w}} f(\RSnew{\widetilde{\bm{w}}}_k, \bm{x}) - \frac{1}{|B_k|} \sum_{\bm{x} \in B_k} \nabla_{\bm{w}} f(\bm{w}_k, \bm{x}) \right\| \leq \epsilon.
\]
with \( \epsilon \ll 1 \), ensuring that the forged update closely tracks the original trajectory. In particular, when the batch size is set to one, the \emph{gradient matching} condition reduces to the \emph{one-step forging problem} where one seeks \( \widetilde{\bm{x}}\neq \bm{x}^* \) while satisfying: 
\begin{equation}\label{grad-match}
    \left\| \nabla_{\bm{w}} f(\bm{w}_k, \bm{x}^*) - \nabla_{\bm{w}} f(\bm{w}_k, \widetilde{\bm{x}}) \right\| \leq \epsilon.
\end{equation}
The alternative mini-batches $\{\widetilde{B}_k\}$ or the data point $\widetilde{\bm x}$ need not belong to the original dataset $D$. In principle, a forger can choose the forging data from anywhere in the ambient space that contains the data distribution. Throughout the analysis, we condition on the original data and the model trajectory, even if they are obtained from SGD, since the forging process takes place entirely after training has concluded.

\subsection{Related Work and Contributions}
\paragraph{Related Work.} Previous work has primarily focused on developing unlearning algorithms with an emphasis on practical efficiency. In recent years, however, increasing studies have been focusing on the certification and verification of these methods. Thudi et al.~\cite{thudi2022necessity} argue that formally proving the absence of a specific data point after a claimed unlearning process is unrealistic, unless the process is subject to external scrutiny, such as an audit. This stems from a common assumption in the literature: that the model should not change significantly when the data is modified. As a result, it is often possible to construct an alternative dataset that produces a similar model, which renders exact verification of data removal infeasible. 

Baluta et al.~\cite{baluta2023unforgeability}  consider forging under a fixed-point model of computation and demonstrate that exact forging under that model is unrealistic. They show that even small floating-point errors can be amplified over the course of training, making precise replication infeasible. On the other hand, to fully understand the implications of forging in machine unlearning, we establish a quantifiable framework—one that supports rigorous analysis of more advanced, model-driven forgery attacks that are less reliant on numerical precision. Suliman et al.~\cite{suliman2024data} similarly argue that forging is both difficult and empirically detectable. Their results show that errors introduced by greedily constructed forged batches typically exceed those caused by benign sources of randomness during training. Their theoretical analysis in the setting of logistic regression provides insight into why forging is inherently challenging. This paper generalizes the analysis beyond logistic regression and extends insights to a broader class of models, aiming to establish a unified theoretical foundation for analyzing forging in modern architectures, including deep neural networks and large language models (LLMs).

\paragraph{Motivations.} Successful forging can offer two main advantages. First, by replicating a model’s trajectory, a forger can preserve the model’s utility and avoid the cost of retraining---especially when retraining is impractical. For instance, if the loss function exhibits local convexity, a small change in the data trajectory leaves the model nearly unchanged. Second, when the loss landscape is complex and sensitive to data, a forger can craft an alternative point that mimics the effect of the original data. This scheme keeps the model from drifting toward a different local optimum and helps maintain its original behavior. We present examples highlighting both motivations in \Cref{MOT}. Forging, therefore, poses a serious threat to genuine unlearning. One main goal of this paper is to deepen the theoretical understanding of forging, with the hope that this can assist in detecting forgery attempts and strengthening the robustness of unlearning algorithms.

\paragraph{Our Contributions.}
We develop a measure-theoretic framework for analyzing $\epsilon$-forging sets---the collection of data points whose gradients replicate the original update within tolerance $\epsilon$.  Beginning with linear regression, we show that the Lebesgue measure of the forging set scales on the order of $\epsilon$ (\Cref{prop:lr-eps-forging}), and establish the same scaling law for one-layer neural networks (\Cref{prop: nn-eps-forging}). We then generalize to smooth loss functions and, under mild regularity assumptions on the loss landscape and model gradient, prove in \Cref{seconvarthm1} that the forging set measure is bounded by $\epsilon^{(d-r)/2}$, where $d$ is the data dimension and $r<d$ is the \RSnew{dimension of vector space of right singular vectors corresponding to ``small" singular values} of a certain variation matrix introduced in our analysis. For simple problems such as linear regression, we show that \RSnew{$r = 2$} (Appendix~\ref{degenerateassumptionsec}). Applying the same reasoning, we extend these bounds to batch SGD. Finally, by invoking measure regularity, we obtain a general result for almost-everywhere smooth loss functions (\Cref{aeforgethm1}), which also yields an $\epsilon^{(d-r)/2}$ scaling provided $\epsilon$ is sufficiently small and satisfies a cover separation condition (Lemma~\ref{coverlemma00}).

In addition, under a non-degeneracy assumption on the data distribution, we show that the probability of randomly sampling a forging point is vanishingly small unless the data are adversarially engineered. We provide probability bounds in both simple settings (Corollaries~\ref{corlr-p}, \ref{cornn-p}) and general settings (Theorems~\ref{seconvarthm1p}, \ref{aeforgethm2}). Thus, our results not only align with empirical findings on forgery detectability \cite{baluta2023unforgeability,suliman2024data}, but also provide a rigorous quantitative framework that sheds light on some fundamental limitations of forgery-based attacks in unlearning.

\paragraph{Paper Organization.}
\Cref{NOTE} introduces the notation used throughout. \Cref{MOT} presents the motivation for studying forging-type adversarial attacks, illustrated with concrete examples. \Cref{case} analyzes the forging set in two fundamental settings---linear regression and one-layer neural networks. \Cref{generalforgeanalysismainsec} develops the general framework for smooth loss functions, and \Cref{forgeinbatch} extends the analysis to batch SGD. \Cref{sectionaeforge} extends the results further to almost-everywhere smooth loss functions. \Cref{conclusion} summarizes our findings and outlines directions for future work. The Appendix provides detailed proofs and additional technical material.

\subsection{Notation}\label{NOTE}
We use $\bm x\in\mathbb{R}^d$ to denote a data point and $y\in\mathbb{R}$ to denote its associated label. A collection of such samples is denoted by $D$. For a vector $\bm v \in \mathbb{R}^{d}$, we use $v_j \in \mathbb{R}$ to denote its $j$-th entry and $\|\bm v\| = \|\bm v\|_2 = \sqrt{\sum_{j} v_j^2}$. The standard basis vector in $\mathbb{R}^d$ with a $1$ in the $i-$th entry and zeros elsewhere is \(\bm{e}_i\). $\bm 1$ denotes the all-ones vector. For a matrix $\bm M\in\mathbb{R}^{n\times d}$, we use $\bm m_j \in \mathbb{R}^{n}$ for its $j$-th column, $\bm m_i^T \in \mathbb{R}^{d}$ for its $i$-th row, and $m_{ij}$ for the $(i,j)$-th entry of $\bm M$. The Frobenius norm is $ \|\bm M\|_{F} := \sqrt{\sum_{i,j} m_{ij}^{2}}$ and the operator norm is $\norm{\bm M} $. The indicator function of a set $\mathcal{X}$ is $\bm 1_{\mathcal{X}}$. \RSnew{The cardinality of a countable set $\mathcal{X}$ is $ \#\mathcal{X}$.}

We denote by $\mathcal{B}_r(\bm x)$ the open ball centered at $\bm x$ of radius $r$ and and $\mathcal{B}_r := \mathcal{B}_r(\bm{0})$ when centered at the origin. The unit sphere in $\mathbb{R}^d$ is $S^{d-1}$. For a set $A\subset\mathbb{R}^{d}$, we denote its diameter by $\text{diam}(A) := \sup_{x, y \in A} \|x - y\|_2$.  We denote the Lebesgue measure by $\mu$. For $A\subset\mathbb{R}^d$,  its Lebesgue measure, or volume, in $\mathbb{R}^d$ is $\text{vol}_{\mathbb{R}^d}(A)$, so that $\text{vol}_{\mathbb{R}^d}(\mathcal{B}_r)$ represents the volume of a ball centered at the origin with radius $r$. We write $p(\bm x)$ for a probability density function and $\mathbb{P}_{\mathcal{D}}(X = \bm x)$ for the probability of a random variable $X$ taking value $\bm x$ under distribution $\mathcal{D}$. The abbreviation ``a.e.'' stands for ``almost everywhere'' on a measurable space. 

The symbol $\bigoplus$ represents the direct or orthogonal sum of vector spaces, $\bigotimes$ is used for product of measures, $\otimes$ is used to for the Kronecker product and $\odot$ is the Hadamard product. For two sets $A,B$ the set $A+B$ is their Minkowski sum. $ker(\cdot)$ represents the kernel, $dim(\cdot)$ represents the dimension, and for any vector spaces $A,B$ with $A \subset B$, $A ^{\perp}$ represents the orthogonal complement of $A$ in $B$, where $B$ is understood from the context. The symbol $\mathcal{O}$ represents the Big-O notation, the symbol $o$ represents the little-o notation.   For any two sets $A,B$ in some topological space $X$, $A\Subset B$ means $A$ is compactly embedded in $B$ with respect to the topology on $X$. For a set $A$, $\partial A$ denotes its boundary when defined and for a continuous function $f$, $\partial f(x)$ represents the generalized sub-differential set of $f$ at $x$. Throughout the paper $\nabla f(\w;\z) $ denotes the gradient of $f$ with respect to the first argument $\w$. $\mathcal{C}^r$ represents the class of $r-$continuously differentiable functions.

\section{Motivation}\label{MOT}
We now illustrate two concrete scenarios in machine unlearning where forging introduces strong, realistic, and arguably perverse incentives: (1) forging to preserve the original model, and (2) forging to prevent significant model drift when the model is highly sensitive to  minor data modifications.

\paragraph{Forging can allow the model to remain unchanged.} A particularly compelling incentive arises when replacing a data point with a carefully chosen alternative induces negligible change in the model without incurring the cost of retraining from scratch. As a concrete illustration, \Cref{thm:forging_incentive} demonstrates that when a well chosen replacement point approximately preserves the gradient of a locally smooth, strongly convex loss function, the resulting model parameters remain approximately unchanged. Before stating the theorem, we introduce some notation. Let $(\bm x_0, \bm x_1, \bm x_2, ..., \bm x_{N-1})$ denote the sequence of data points used for  $N$ updates, initialized at parameter $\bm w_0 \in  \mathbb{R}^n$. The iterates evolve according to the standard SGD-type rule:
\begin{equation}\label{update}
    \bm w_k = \bm w_{k-1} - h_{k-1}\nabla f_{k-1}(\bm w_{k-1})
\end{equation}
where $h_{k-1}$ is the step size, and  $f_{k-1}(\bm w) \coloneqq f(\bm w; \bm x_{{k-1}})$ at step $k-1$ for $1 \leq k \leq N$. As is typical in SGD optimization, data points may be reused. 

Without loss of generality, we assume  forging occurs at the beginning of the trajectory, at $\bm x_0$, for a total of $m+1$ times---one for each appearance of $\bm x_0$. Then the original and forged sequences are
\begin{equation}\label{orig-path-1}
    (\bm x_0, ..., \bm x_{n_1-1}, \bm x_0, \bm x_{n_1+1}, ..., \bm x_{n_m-1}, \bm x_0, \bm x_{n_m+1}, ..., \bm x_{N-1})
\end{equation}
and 
\begin{equation}\label{altern-path-1}
    (\widetilde{\bm x}_0, ..., \bm x_{n_1-1},\widetilde{\bm x}_0, \bm x_{n_1+1}, ..., \bm x_{n_m-1}, \widetilde{\bm x}_0, \bm x_{n_m+1}, ..., \bm x_{N-1}).
\end{equation}
Applying the update rule \eqref{update}, the data trajectory \eqref{orig-path-1} induces the parameter sequences
\begin{equation}\label{orig-wpath-1}
    (\bm w_0, \bm w_1,..., \bm w_{n_1}, \bm w_{n_1+1}, ..., \bm w_{n_m}, \bm w_{n_m+1}, ..., \bm w_{N}).
\end{equation}
Define $\widetilde{f}_0(\bm w) = f(\bm w; \widetilde{\bm x}_0)$. Then the alternative model trajectory resulting from replacing $\bm x_0$ by $\widetilde{\bm x}_0$ as in \eqref{altern-path-1}, and correspondingly replacing $f_{0}$ by $\widetilde{f}_0$ in \eqref{update} is
\begin{equation}\label{altern-wpath-1}
    (\bm w_0, \widetilde{\bm w}_1,..., \widetilde{\bm w}_{n_1}, \widetilde{\bm w}_{n_1+1}, ..., \widetilde{\bm w}_{n_m}, \widetilde{\bm w}_{n_m+1}, ..., \widetilde{\bm w}_{N}).
\end{equation}

Before quantifying the difference of forged and original parameter trajectories, we \RSnew{recall the definitions for strong convexity and Lipschitz smoothness.}

\RSnew{\begin{definition}[$\mu$-strong convexity]
Let $O\subseteq \mathbb{R}^n$ be an open convex set and let $f:O\to\mathbb{R}$ be differentiable.
We say $f$ is $\mu$-strongly convex on $O$ (with $\mu>0$) if for all $\bm x,\bm y\in O$,
\(
f(\bm y)\ge f(\bm x)+\langle \nabla f(\bm x),\,\bm y-\bm x\rangle+\frac{\mu}{2}\|\bm y-\bm x\|^2.
\)
\end{definition}}
\RSnew{\begin{definition}[$L$-smoothness]
Let $O\subseteq \mathbb{R}^n$ be an open convex set and let $f:O\to\mathbb{R}$ be differentiable.
We say $f$ is $L$-smooth on $O$ (with $L>0$) if for all $\bm x,\bm y\in O$,
\(
\|\nabla f(\bm x)-\nabla f(\bm y)\|\le L\|\bm x-\bm y\|.
\)
Equivalently, for all $\bm x,\bm y\in O$,
\(
f(\bm y)\le f(\bm x)+\langle \nabla f(\bm x),\,\bm y-\bm x\rangle+\frac{L}{2}\|\bm y-\bm x\|^2.
\)
\end{definition}}
\RSnew{Next, we }introduce the following definition \cite{pach2008state}, \cite{driemel2010approximating}.
\vspace{1em}
\begin{definition}\label{def: epstube}
The \emph{discrete $\epsilon$-tube} around  trajectory \eqref{orig-wpath-1} is the union of open $\epsilon$-balls centered at each point:
  \[
T_\epsilon^{\mathrm{disc}}:=T_\epsilon^{\mathrm{disc}}(\bm w_0, ..., \bm w_N) =  \bigcup_{i=0}^N \mathcal{B}_\epsilon(\bm{w}_i), \quad \text{where } \mathcal{B}_\epsilon(\bm{w}_i) := \left\{ \bm{x} \in \mathbb{R}^n : \|\bm{x} - \bm{w}_i\| < \epsilon \right\}.
  \]
The \emph{interpolated (or continuous) $\epsilon$-tube} is  the union of $\epsilon$-balls centered along the line segments between successive points:
  \begin{equation}
  T_\epsilon^{\mathrm{cont}}:=T_\epsilon^{\mathrm{cont}}(\bm w_0, ..., \bm w_N) = \bigcup_{i=0}^{N-1} \bigcup_{t \in [0,1]} \mathcal{B}_\epsilon\left( (1-t)\bm{w}_i + t\bm{w}_{i+1} \right).
  \end{equation}
$T_\epsilon^{\mathrm{disc}} \subseteq T_\epsilon^{\mathrm{cont}}\subset \mathbb{R}^n$, and the inclusion is strict when adjacent points are separated by more than $2\epsilon$.
\end{definition}

We now state our first result showing that the resulting model can remain nearly unchanged even when a data point in the training trajectory is replaced by a far-away point.
\vspace{1em}
\begin{theorem}\label{thm:forging_incentive}
Let the functions $f_k$, \( 1 \leq k \leq N \), be \( \mu_k \)-strongly convex and \( L_k \)-smooth on $S\subset \mathbb{R}^n$
and let \( \{\bm w_i\}_{i=0}^N \subset  T^{\mathrm{cont}}_\epsilon \subset S \) be the SGD trajectory as given in \eqref{orig-wpath-1}. 
Denote the gradient deviation caused by replacing \( \bm x_0 \) with \( \widetilde{\bm x}_0 \) at $k=1$ by $\delta_0 \coloneqq \|\nabla f_0(\bm w_0) - \nabla \widetilde{f}_0(\bm w_0)\| \leq \epsilon$
where $\widetilde{f}_0( \cdot ) = f(\cdot; \widetilde{\bm x}_0 )$. Assume that \( f_0 \in \mathcal{C}^2 \), and that for each subsequent replacement step \( k > 1 \)
\begin{equation}\label{lip-in-data}
    \|\nabla f_0(\widetilde{\bm w}_k) - \nabla \widetilde{f}_0(\widetilde{\bm w}_k)\| \leq \mu_0 \|\widetilde{\bm w}_k - \bm w_k\|.
\end{equation}
Then, if the step sizes satisfy \( h_k \leq \frac{1}{L_k} \) for all \( k \), the final model parameters satisfy $\|\widetilde{\bm w}_N - \bm w_N\| < \delta_0$.

\end{theorem}
\proof{Please see Appendix \ref{pf-motiv} for the full proof. }
\vspace{1em}
\begin{remark}
    The alternative data point $\widetilde{\bm x}_0$ used to replace $\bm x_0$ only needs to yield a small norm difference between the original and new gradients. Notably, this does not require the two data points to be close in input space. For example, consider the function $f : \mathbb{R}^d \times \mathbb{R}^d \to \mathbb{R}$ defined by 
    \[f(\bm w; \bm x) = \frac{1}{4}\|\bm w\|^2 + e^{-\|\bm x\|^2}\mathbf{1}^T\bm w.\]
    This function is $\mu$-strongly convex and $L$-smooth in $\bm w$ with $\mu = L = \frac{1}{2}$. Fix $\bm w$ and let $\epsilon > 0$. For a sufficiently large real number $M$, define $\bm x = (M, 0, \ldots, 0)$, so that $e^{-\|\bm x\|^2} < \frac{\epsilon}{2\sqrt{d}}$. Let $\bm y = (0, M, 0, \ldots, 0)$, yielding
    \(\|\bm x - \bm y\| = \sqrt{2}M,\)
    and
    \[\|\nabla f(\bm w; \bm x) - \nabla f(\bm w; \bm y)\| = |e^{-\|\bm x\|^2} - e^{-\|\bm y\|^2}|\|\mathbf{1}\| \leq \left(e^{-\|\bm x\|^2} + e^{-\|\bm y\|^2}\right)\sqrt{d} < \epsilon.\]
\end{remark}

Next, we present another aspect of how forging can benefit an adversary.

\paragraph{Not forging may cause the model to deviate.} The second incentive for forging arises when replacing a single data point may lead to significantly different model parameters. Non-convex models are often highly sensitive to small perturbations, which can cause them to shift toward entirely different local minima and produce qualitatively distinct outcomes. To illustrate this effect, let \(\bm{w} \RSnew{ =(w_1,...,w_n)}\in \mathbb{R}^n\) and \(\bm{x} \in \mathbb{R}^d\), and define $\bm{a}(\bm{x}) := A \bm{x} \in \mathbb{R}^n$,
where \(A \in \mathbb{R}^{n \times d}\).  Let \(\bm{\mu} := c \cdot \bm{e}_1 \in \mathbb{R}^n\) for some constant \(c > 0\), which defines the centers of two attraction basins in parameter space. We define 
\begin{align*}
    g_1(\bm w; \bm x) &= \|\bm{w} - \bm{\mu}\|^2 + \log\left(1 + \exp\left(- \bm{a}(\bm{x})^\top \bm{w}\right)\right), \\
    g_2(\bm w; \bm x) &= \|\bm{w} + \bm{\mu}\|^2 + \log\left(1 + \exp\left(- \bm{a}(\bm{x})^\top \bm{w}\right)\right).
\end{align*}

Let the overall loss be a smooth interpolation between \(g_1\) and \(g_2\), defined by
\[
f(\bm w; \bm x) = \alpha(\bm{w}) \cdot g_1(\bm{w}; \bm{x}) + \left(1 - \alpha(\bm{w})\right) \cdot g_2(\bm{w}; \bm{x}),
\]
where the interpolation weight is given by the logistic function
\[
\alpha(\bm{w}) := \frac{1}{1 + \exp\left(-5 \cdot \bm{w}^\top \bm{\mu} / \|\bm{\mu}\|\right)} = \frac{1}{1 + \exp(-5 w_1)}.
\]
\RSnew{where the second equality follows from $\frac{\bm w^\top\bm\mu}{\|\bm\mu\|} = \frac{c\,\bm w^\top\bm e_1}{c\|\bm\e_1\|}=\bm w^\top\bm e_1=w_1$}.
This construction produces a nonconvex loss landscape with two basins of attraction approximately centered at \(\bm{w} = \pm \bm{\mu}\). In each basin, the loss behaves locally like a convex function. However, when training is initialized near the saddle point (e.g., \(\bm{w}_0 = \bm 0\)), small perturbations to the input \(\bm{x}\), such as replacing \(\bm{x}\) with a nearby \(\widetilde{\bm{x}}\), can cause the gradient to point in different directions, leading to divergent parameter trajectories.

To provide a visualization, consider  \(n = d = 1\), and use the training data \(X = (x_0, x_1, \ldots, x_{19})\) for 20 updates according to \eqref{update}, with a fixed learning rate of \(0.3\). When initialized at \(w_0 = 10^{-4}\), the model converges toward the local minimum at \(w^* = -2\). In contrast, replacing the first data point \(x_0 = -0.5\) with \(\widetilde{x}_0 = 0.2\) results in an alternative trajectory that drives the model toward the opposite basin at \(w^* = 2\), as illustrated in \Cref{fig:div}. In such cases, a forger may be strongly tempted to carefully choose a replacement point that preserves the model output.

\begin{figure}[H]
    \centering
    \includegraphics[width=0.7\linewidth]{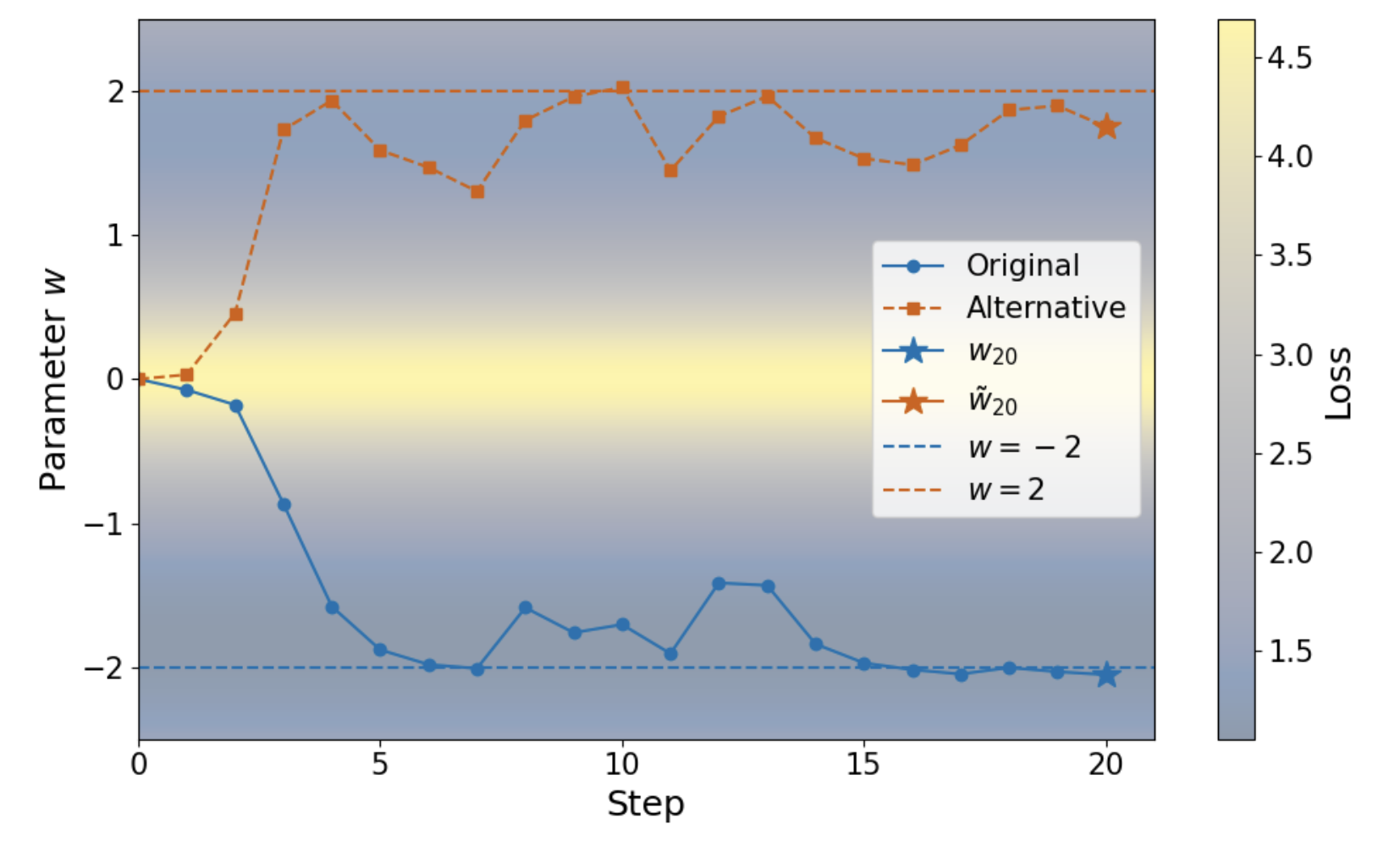}
    \caption{Model trajectories with the original dataset (\RSnew{orange squares}) and the forged dataset (\RSnew{blue dots}). Forging is applied at the first step, immediately after initialization.}
    \label{fig:div}
\end{figure}

\section{Case Study: Linear Regression and Shallow Neural Networks}\label{case}
We now examine forging in the context of simple models: linear regression and one-layer neural networks. By explicitly analyzing the gradient-matching condition defined in \Cref{grad-match}, we  bound the Lebesgue measure of the forging set. 

\subsection{Linear Regression}\label{LR}
To better understand the forging phenomenon, one of the simplest  loss functions from which we can gain intuition is linear regression. 
Linear regression uses the loss function $f$ evaluated at the parameter $\bm w$, associated with a data point \((\bm x, y)\) given by 
\begin{align}f\left(\bm w; (\bm x,y)\right) = \frac{1}{2}(\bm x^T\bm w-y)^2.\label{eq:lr_loss}\end{align} For any  $(\bm x,y)$ and $\epsilon>0$, the corresponding $\epsilon$-forging set $S_\epsilon$ is defined as
\begin{equation}\label{def: lr-eps-forging-set}
S_\epsilon(\bm w, \bm x, y) \coloneqq \{(\bm z,t): \|\nabla_{\bm w}f\left(\bm w; (\bm x,y)\right) - \nabla_{\bm w}f\left(\bm w; (\bm z,t)\right)\| \leq \epsilon \}.
\end{equation}
When $\epsilon=0$, this corresponds to exact-forging, where one seeks a data point whose gradient exactly matches that of the target point under a one-step gradient descent update. Explicitly,
\begin{equation}\label{def:lr-ex-forging}
S_{0}(\bm w, \bm x, y)  := \{(\bm z,t): \|\nabla_{\bm w}f\left(\bm w; (\bm x,y)\right) - \nabla_{\bm w}f\left(\bm w; (\bm z,t)\right)\| = 0 \}
\end{equation}

For notational simplicity, we omit the dependence on $(\bm w, \bm x, y)$ and refer to the set as $S_\epsilon$ or $S_0$ when the context is clear. We start by analyzing the exact-forging set.

\vspace{1em}
\begin{proposition}\label{prop:lr-ex}
    Let $f$ be as in \eqref{eq:lr_loss}. For any  $(\bm x,y)\in \mathbb{R}^d\times\mathbb{R}$ with $\nabla_{\bm w}f\left(\bm w; (\bm x,y)\right)\neq 0$, the exact-forging set defined in \eqref{def:lr-ex-forging} has Lebesgue measure zero. 
\end{proposition}

\begin{proof}
    Fix $(\bm x,y)$. 
    Taking derivatives with respect to $\bm w$, the statement $(\bm z,t) \in S_0$ is equivalent to 
    $(\bm z^T\bm w-t)\,\bm z = (\bm x^T\bm w-y)\,\bm x$.
    Since $\bm x$ and $y$ are given, then denoting $\bm x^T\bm w-y \in \mathbb{R}$ by $A$ and defining $s(\bm z, t) \coloneqq \bm z^T\bm w-t$, we see that $(\bm z,t) \in S_0$ is equivalent to 
    \begin{equation}\label{eq:exactLR1}
         s(\bm z, t)\,\bm z = A\,\bm x.
    \end{equation}
    Given that $\nabla_{\bm w}f\left(\bm w; (\bm x,y)\right)\neq 0$, we conclude that  $A\neq0$ and $\bm x\neq\bm 0$, and also that neither $s(\bm z, t)$ nor $\bm z$ can be zero. So we can further define 
    $\alpha(\bm z, t)\coloneqq\frac{A}{s(\bm z, t)}$
    so that according to \eqref{eq:exactLR1}
    \begin{equation}\label{eq:exactLR3}
        \bm z = \alpha(\bm z, t)\,\bm x,
    \end{equation}
    which essentially forces $\bm z$ to be parallel to $\bm x$. Substitute in \eqref{eq:exactLR1} to obtain 
    \begin{equation}\label{eq:exactLR2}
        A\,\bm x = \Big(s(\bm z, t)\,\alpha(\bm z, t)\Big)\,\bm x.
    \end{equation}
    Further substituting $\bm z = \alpha(\bm z, t)\,\bm x$ in $s(\bm z, t) = \bm z^T\bm w-t$, we derive
    \[s(\bm z, t)\,\alpha(\bm z, t) = \Big(\alpha(\bm z, t)\,\bm x^T\bm w-t\Big)\,\alpha(\bm z, t) = \alpha(\bm z, t)^2\,(\bm x^T\bm w)-\alpha(\bm z, t)\, t.\]
    Then from \eqref{eq:exactLR2}, we have 
    \begin{align}\label{quadr}A = \alpha(\bm z, t)^2\,c-\alpha(\bm z, t)\, t\quad \text{with}\quad c\coloneqq\bm x^T\bm w.\end{align}
    For each fixed $t\in\mathbb{R}$, if $c\neq 0$, this is a quadratic equation in $\alpha(\bm z, t)$, and the solution is
    $\alpha(\bm z, t) = \frac{t\pm\sqrt{\,t^2+4\,c\,A}}{2\,c}$.
    By \eqref{eq:exactLR3}, $\bm z$ can thus be expressed as a function of $t$ via
    \(\bm z = \frac{t\pm\sqrt{\,t^2+4\,c\,A}}{2\,c}\,\bm x,\)
    which indicates that $S_0$ is formed by two separate continuous curves in $\mathbb{R}^d\times\mathbb{R}$.
    On the other hand, if $c=\bm x^T\bm w = 0$, the equation reduces to
    $y = \alpha(\bm z, t)\,t$ since $A = \bm x^T\bm w - y$.
    This provides a solution 
    $\bm z = \frac{y}{t}\,\bm x$
    which is a continuous curve in $\mathbb{R}^d \times \mathbb R$. Note that in this case $t\neq 0$, because otherwise $A=0$,  contradicting our assumption that the gradient is non-zero.
    Therefore, $\mu(S_0)=0$.
\end{proof}

The result above can be extended to  $\epsilon$-forging with $\epsilon > 0$. The next proposition does exactly this, providing a bound on the Lebesgue measure of the $\epsilon$-forging set, demonstrating that even with the relaxation, the set is highly constrained. 
Specifically, for any non-zero radius, we bound 
\RSnew{the measure of $S_\epsilon$ in a cylinder $\mathcal{C}_R$ of $\mathbb{R}^{d+1}$} and outline the main proof ideas, deferring the full details to Appendix~\ref{pf-sec3}. Note that while the result is stated for the \RSnew{object} centered at the origin, it holds regardless of center. 
\vspace{1em}
\begin{proposition}
\label{prop:lr-eps-forging}
    Let $R>0$, 
    then for any  $(\bm x,y)\in\mathbb{R}^{d}\times\mathbb{R}$ with $d>1$ and $\nabla_{\bm w}f\left(\bm w; (\bm x,y)\right)\neq \mathbf{0}$ , the $\epsilon$-forging set defined in \eqref{def: lr-eps-forging-set} restricted to 
    \RSnew{$\mathcal{C}_R = \mathcal{B}_R\times[-R,R]\subset\mathbb{R}^{d+1}$
    \begin{equation}\label{lr-eps-1}\frac{\mu(S_\epsilon\cap \mathcal{C}_R)}{\mathrm{vol}_{\mathbb{R}^{d+1}}(\mathcal{C}_R)} \leq \frac{d}{d-1}\frac{\epsilon}{R^2}.\end{equation}}
    Furthermore, 
    if $\frac{\epsilon}{A}<\sin(c\epsilon)$ for some $c\in\left[\frac{1}{A}, \frac{\pi}{2A}\right]$, where $A = \|\nabla_{\bm w} f\left(\bm w; (\bm x,y)\right)\|$, then
     \begin{equation}\label{lr-eps-2}
     \RSnew{\frac{\mu(S_\epsilon\cap \mathcal{C}_R)}{\mathrm{vol}_{\mathbb{R}^{d+1}}(\mathcal{C}_R)} \le \frac{\sqrt{2}d\sqrt{d+1}c^{d-1}}{\sqrt{\pi}(d-1)^2}\frac{\epsilon^d}{R^2}}.
     \end{equation}
\end{proposition}
\begin{proof}[Proof sketch]
The main idea in estimating the Lebesgue measure of $S_\epsilon\cap \RSnew{\mathcal{C}_R}$ is to first compute the feasible range of the label $t$ for a fixed data point $\bm{z}$, and then integrate over the data space. Fix $(\bm{x}, y)$ and let $\epsilon > 0$. 
Let $\bm{a} \coloneqq (\bm{x}^T \bm{w} - y) \bm{x}$ and define $A \coloneqq \|\bm{a}\|$. We also let $s(\bm{z}, t) \coloneqq \bm{z}^T \bm{w} - t$.
The membership condition for the $\epsilon$-forging set defined in \eqref{def: lr-eps-forging-set} is 
the norm inequality
\begin{equation}\label{ineq:epslr0}\|\bm a - s(\bm z,t)\bm z\| \leq \epsilon.\end{equation}
For $\bm{z}\neq 0$, squaring both sides of \eqref{ineq:epslr0} leads to a quadratic equation in $s(\bm{z}, t)$, from which the feasible range of $t$ (by translation invariance of Lebesgue measure) can be determined. The resulting measure of this interval in $\mathbb{R}$ is
\(
L(\bm{z}) = \frac{2\sqrt{\epsilon^2 - A^2 \sin^2\theta}}{\|\bm{z}\|},
\)
where $\theta$ is the angle between $\bm{x}$ and $\bm{z}$. This  introduces a constraint on $\theta$ arising from the non-negativity of the discriminant, namely $A|\sin\theta| \leq \epsilon$
which implies:
\begin{equation} \label{anglecond0}
\theta \in [-\theta_0, \theta_0], \quad \text{where} \quad \theta_0 = \arcsin\left(\min\left\{1,\; \frac{\epsilon}{A} \right\}\right).
\end{equation}

To compute the total volume, we integrate over $\bm{z} \in \mathbb{R}^d$, in a ball of radius $R$. The volume satisfies
\[
\RSnew{\mu(S_{\epsilon}\cap\mathcal{C}_R) \leq \mu(S_{\epsilon}\cap(\mathcal{B}_R\times\mathbb{R}}))\leq\int_{\bm{z} \in \mathcal{B}_R} \mathbf{1}_{\{A |\sin \theta| \leq \epsilon\}}\, L(\bm{z}) \,d\bm{z}.
\]
This can be explicitly calculated in spherical coordinates. 
Taking $\theta_0 = \arcsin(1) = \frac{\pi}{2}$ in \eqref{anglecond0} and simplifying, we recover the bound stated in \eqref{lr-eps-1}. Enforcing $\theta_0 = \arcsin\left(\frac{\epsilon}{A}\right) \leq c \epsilon$,
for some constant \(c\in \left[\frac{1}{A}, \frac{\pi}{2A}\right] \) and evaluating the integral gives the  bound in \eqref{lr-eps-2}. The full proof is in Appendix~\ref{pf-sec3}.
\end{proof}

\begin{remark}[Vanishing \RSnew{r}elative \RSnew{v}olume]
    Inequalities \eqref{lr-eps-1} and \eqref{lr-eps-2} show that the relative volume 
    $\frac{\mu(S_\epsilon \cap \RSnew{\mathcal{C}_R})}{\RSnew{\mathrm{vol}_{\mathbb{R}^{d+1}}(\mathcal{C}_R})}$
    tends to zero as \(R \to \infty\). This shows that, in the limit of a large ambient domain, the forging set occupies a negligible fraction of the space.
\end{remark}

\RSnew{
\begin{remark}[Small-$\epsilon$ regime]
The condition $\frac{\epsilon}{A} < \sin(c\epsilon)$ is not restrictive, as it captures the regime of interest for $\epsilon$. For instance, taking $c=\tfrac{1.5}{A}$, the inequality $\epsilon/A < \sin(1.5\,\epsilon/A)$
holds for all $\epsilon\leq 0.5 A$. Since $A$ is the norm of the gradient, this range actually allows for a relatively large perturbation. \end{remark}
}
\RSnew{
\begin{remark}[Linear regression with bias]
\label{rem:lr-bias}
The linear regression model considered above omits an explicit bias term for clarity.
If instead we consider
\(
f\big((\bm w,b);(\bm x,y)\big)=\frac12(\bm x^T\bm w+b-y)^2
\)
and define forging by matching the \emph{full gradient} with respect to $(\bm w,b)$,
the resulting forging geometry is strictly more constrained.
In particular, for any $(\bm x,y)$ with nonzero gradient,
the exact-forging set collapses to the singleton $\{(\bm x,y)\}$ as matching the gradient with respect to the bias enforces equality of residuals.
Moreover, the $\epsilon$-forging set with bias admits volume bounds of the same order
as in Proposition~\ref{prop:lr-eps-forging}. 
\end{remark}
}

\Cref{prop:lr-ex} and \Cref{prop:lr-eps-forging} show that, in linear regression, the set of points achieving exact or $\epsilon$-approximate gradient matching occupies a small region of the ambient space. Although a forger can construct such points explicitly by solving the gradient-matching equations, they are unlikely to find one through resampling without deliberate selection. This supports the intuition, which we make rigorous later via probability bounds, that random sampling from a realistic data distribution is very unlikely to produce a valid forgery.

\subsection{One-Layer Neural Network }\label{NN}
Another simple and important model for gaining insight into forging is one-layer neural networks. Consider the $\texttt{ReLU}$ activation function, and let $\bm W\in\mathbb{R}^{n\times d}$, $\bm v\in\mathbb{R}^n$. For a data point  \((\bm x, y)\), define the loss function 
\[f\left(\bm W,\bm v; (\bm x,y)\right) = \frac{1}{2}(\bm v^T\rho(\bm W\bm x)-y)^2  \]
where $\rho = \texttt{ReLU}$ acts elementwise with $\texttt{ReLU}(x) = \max\{x,0\}$. Note that  \( \rho \) is non-differentiable at zero, and its subgradient \( \rho'(0) \) can take any value in \( [0,1] \). Here, we adopt the common practical choice \( \rho'(0) = 0 \)  \cite{bertoin2021numerical,berner2019towards} and define the corresponding $\epsilon$-forging set as
\begin{align}\label{def:nn-eps-forging}S_{\epsilon}(\bm W, \bm v, \bm x,y)  := & \{(\bm z,t): \|
\nabla_{\bm W,\bm v}f\left(\bm W,\bm v; (\bm x,y)\right) - \nabla_{\bm W,\bm v}f\left(\bm W,\bm v; (\bm z,t)\right)
\|_F \leq \epsilon\}.\end{align}

The joint gradient of the loss function with respect to both $\bm W$and $ \bm v$ is then 
\[\nabla_{\bm W,\bm v}f\left(\bm W,\bm v; (\bm x,y)\right) = \begin{bmatrix}
    \nabla_{\bm W}f\left(\bm W,\bm v; (\bm x,y)\right)\\
    \nabla_{\bm v}f\left(\bm W,\bm v; (\bm x,y)\right)
\end{bmatrix}.\]
As before, when $\epsilon=0$, the exact-forging set is 
\begin{align}\label{def:nn-ex-forging}S_{0}(\bm W, \bm v,\bm x,y)  := & \{(\bm z,t): \|
\nabla_{\bm W,\bm v}f\left(\bm W,\bm v; (\bm x,y)\right) - \nabla_{\bm W,\bm v}f\left(\bm W,\bm v; (\bm z,t)\right)
\|_F =0\}.\end{align}
This set captures all data points $(\bm z, t)$ whose gradient with respect to the network parameters exactly matches that of a reference point \((\bm x, y)\). We begin by analyzing the exact-forging set and show that, under mild regularity conditions, it forms a low-dimensional subset embedded in the ambient space $\mathbb{R}^d \times \mathbb{R}$. Consequently, the exact-forging set has Lebesgue measure zero.
\vspace{1em}
\begin{proposition}\label{prop:NNex}

For any \((\bm x, y)\in\mathbb{R}^{d}\times\mathbb{R}\) with $\nabla_{\bm v}f\left(\bm W,\bm v; (\bm x,y)\right) \neq 0$ and $\nabla_{\bm W}f\left(\bm W,\bm v; (\bm x,y)\right)\neq0$,  the exact-forging set defined in \eqref{def:nn-ex-forging} is of Lebesgue measure zero.
\end{proposition}
\begin{proof}
    Fix $(\bm x, y)$. The gradients of the  loss function with respect to the parameters are 
    \begin{align*}
        \nabla_{\bm v}f\left(\bm W,\bm v; (\bm x,y)\right) &= \left(\bm v^T\rho(\bm W\bm x)-y\right)\,\rho(\bm W\bm x)\\
        \text{and}\quad \nabla_{\bm W}f\left(\bm W,\bm v; (\bm x,y)\right)&=\left(\bm v^T\rho(\bm W\bm x)-y\right)\,\left(\bm v\odot\rho'\left(\bm W\bm x\right)\right)\,\bm x^T 
    \end{align*}
    Here, element-wise, we have
    \[\bm W\bm x = \begin{bmatrix}
        \bm w_1^T\bm x\\\bm w_2^T\bm x\\\vdots\\\bm w_n^T\bm x 
    \end{bmatrix}\quad \text{and}\quad \bm v\odot\rho'(\bm W\bm x) = \begin{bmatrix}
        v_1\,\rho'\left(\bm w_1^T\bm x\right)\\v_2\,\rho'\left(\bm w_2^T\bm x\right)\\\vdots\\v_n\,\rho'\left(\bm w_n^T\bm x\right)
    \end{bmatrix},\]
    so finding $(\bm z, t)\in S_0$ entails solving a system of equations for $j = 1,...,n$ such that
    \begin{align}
    \left(\bm v^T\rho(\bm W\bm x)-y\right)\,\rho(\bm w_j^T\bm x) &= \left(\bm v^T\rho(\bm W\bm z\right)-t)\,\rho(\bm w_j^T\bm z)\label{eq:exact1nn1}\\
    (\bm v^T\rho(\bm W\bm x)-y)\,v_j\,\rho'(\bm w_j^T\bm x)\,\bm x^T &= (\bm v^T\rho(\bm W\bm z)-t)\,v_j\,\rho'(\bm w_j^T\bm z)\,\bm z^T\label{eq:exact1nn2}.
    \end{align}
    If $\nabla_{\bm v}f\left(\bm W,\bm v; (\bm x,y)\right) \neq 0$ and $\nabla_{\bm W}f\left(\bm W,\bm v; (\bm x,y)\right)\neq0$, 
    then there is some index $j$ such that 
    $\rho(\bm w_j^T\bm x)\neq 0$ and $v_j\neq 0$.
    If the left hand side of the equations \eqref{eq:exact1nn1} and \eqref{eq:exact1nn2} are nonzero, then right hand side being nonzero requires $\bm v^T\rho(\bm W\bm z)\neq t$ and $\rho(\bm w_j^T\bm z)\neq 0$.
    Using the same idea as in the proof of \Cref{prop:lr-ex}, equation \eqref{eq:exact1nn2} leads to the relation $\bm z = \alpha(\bm z,t)\bm x$ with
    \[\alpha(\bm z,t) = \frac{A}{\bm v^T\rho(\bm W\bm z)-t}\in\mathbb{R}\]
    where  $A \coloneqq\bm v^T\rho(\bm W\bm x)-y $ and we use the fact that $\rho(\bm w_j^T\bm z)\neq0$ indicates $\rho'(\bm w_j^T\bm z)=1$. 
    Then substituting into \eqref{eq:exact1nn1} and \eqref{eq:exact1nn2} for $\bm z$, we have 
     \begin{align*}
        A\,(\bm w_j^T\bm x) & = (\bm v^T\rho\left(\bm W\alpha(\bm z,t\right)\,\bm x)-t)\,\alpha(\bm z,t)\,(\bm w_j^T\,\bm x)\\
        A\,\bm x &= (\alpha(\bm z,t)\,\bm v^T\rho\left(\bm W\,\bm x)-t\right)\,\alpha(\bm z,t)\,\bm x.
    \end{align*}
    Both equations lead to 
    $ A = \left(\alpha(\bm z,t)\,c-t\right)\,\alpha(\bm z,t)$ with $c\coloneqq\bm v^T\rho\left(\bm W\,\bm x\right)$
    which coincides with  \eqref{quadr} for linear regression. Proceeding as in \Cref{prop:lr-ex}, we conclude that $S_0$ is of measure zero.
\end{proof}
Exact forging in one-layer neural networks exhibits a similar structure to the linear regression case (\Cref{prop:lr-ex}), and suggests a similar phenomenon might occur for $\epsilon$ forging. The next proposition bounds the measure
of the forging set $S_\epsilon$ defined in \Cref{def:nn-eps-forging} restricted to \RSnew{a cylinder $\mathcal{C}_R$ in $\mathbb{R}^{d+1}$}. We provide a proof-sketch and defer the full proof to Appendix~\ref{pf-sec3}. 

\begin{proposition}
\label{prop: nn-eps-forging}
   Let $R>0$, and suppose that for $(\bm x,y)\in\mathbb{R}^{d}\times\mathbb{R}$ with $d>1$, $\nabla_{\bm W}f\left(\bm W,\bm v; (\bm x,y)\right)\neq\bm 0$ 
    and $\nabla_{\bm v}f\left(\bm W,\bm v; (\bm x,y)\right)\neq\bm 0$.  
    \RSnew{Then, defining $\mathcal{C}_R = \mathcal{B}_R\times[-R,R]\subset\mathbb{R}^{d+1}$, we have}
    \begin{equation}\label{1lnn-eps-1}
\RSnew{\frac{\mu(S_\epsilon\cap \mathcal{C}_R)}{\mathrm{vol}_{\mathbb{R}^{d+1}}(\mathcal{C}_R)}\le \frac{d}{d-1}\frac{1}{\min_{v_i\neq 0}\{|v_i|\}}\sum^{d}_{k=0}\begin{pmatrix}
    n\\k
 \end{pmatrix} \frac{\epsilon}{R^2}.}
    \end{equation}
    If additionally \RSnew{$c>0$ satisfies $\frac{\epsilon}{{A_i}}<\sin(c\epsilon)$,  for all $i$, where $A_i = \|\nabla_{\bm W}f\left(\bm W,\bm v; (\bm x,y)\right)_i^T\|$, then} 
\begin{equation}\label{1lnn-eps-2}
\RSnew{\frac{\mu(S_\epsilon\cap \mathcal{C}_R)}{\mathrm{vol}_{\mathbb{R}^{d+1}}(\mathcal{C}_R)}\le
\frac{\sqrt{2}d\sqrt{d+1}c^{d-1}}{\sqrt{\pi}(d-1)^2}\,\frac{1}{(\min_{v_i\neq0}|v_i|)^d}\sum^{d}_{k=0}\begin{pmatrix}
    n\\k
\end{pmatrix}\,\frac{\epsilon^d}{R^2}.}
\end{equation}
\end{proposition}

\begin{proof}[Proof sketch]
By \eqref{def:nn-eps-forging}, a necessary condition for $(\bm z,t)\in S_\epsilon$ is 
$
\|\nabla_{\bm W}f(\bm W,\bm v; (\bm x,y)) - \nabla_{\bm W}f(\bm W,\bm v; (\bm z,t))\|_F \\\le \epsilon
$.
In turn, by examining the $i$th row, it is necessary that $(\bm z,t)\in S_i$, the set of points satisfying
\begin{equation}\label{nnecess-cond}
    \|(\bm v^T\rho(\bm W\bm x)-y)\left[\bm v\odot\rho'(\bm W\bm x)\right]_i\bm x - (\bm v^T\rho(\bm W\bm z)-t)\left[\bm v\odot\rho'(\bm W\bm z)\right]_i\bm z\|\leq\epsilon.
\end{equation}
So   $\mu(S_\epsilon\cap\RSnew{\mathcal{C}_R) \le \mu(\bigcap_i \left(S_i\cap\mathcal{C}_R\right))\leq \min_i \mu(S_i\cap\mathcal{C}_R)}$. Now, note that each $\left[\bm v\odot\rho'(\bm W\bm z)\right]_i$ can either be $v_i$ or 0 and there are at most $\sum_{k=0}^d \binom{n}{k}$ such combinations of values across the rows, so we fix one and later apply a union bound.
Letting $\bm a_i=(\bm v^T\rho(\bm W\bm x)-y)\left[\bm v\odot\rho'(\bm W\bm x)\right]_i\bm x$, $s(\bm z,t) = \bm v^T\rho(\bm W\bm z)-t$, and $\widetilde{v}_i =\left[ \bm v\odot\rho'(\bm W\bm z)\right]_i$, \Cref{nnecess-cond} reduces to 
\begin{equation}\label{ineqcond0}
     \|\bm a_i  - s(\bm z,t)\widetilde{v}_i \, \bm z\|\leq\epsilon.
\end{equation} For $\widetilde{v}_i\neq 0$, dividing the equation by $\widetilde{v}_i$ yields an inequality of the form  \eqref{ineq:epslr0} from \Cref{prop:lr-eps-forging}; $\widetilde{v}_i=0$ is handled by the same worst-case bound. Thus, proceeding in the same way as in \Cref{prop:lr-eps-forging}, 
\[
\mu(S_\epsilon\cap\RSnew{\mathcal{C}_R) \le \frac{d}{d-1}\frac{\mathrm{vol}_{\mathbb{R}^{d+1}}(\mathcal{C}_R)}{R^2}\frac{\epsilon}{\max|v_i|}},
\]
and applying a union bound gives \eqref{1lnn-eps-1}.
A sharper bound follows under $\epsilon/A_i < \sin(c_i\epsilon)$ for suitable $c_i$ and $A_i=\|\nabla_{\bm W}f(\bm W,\bm v; (\bm x,y))_i^T\|$, yielding
\[
\mu(S_\epsilon\cap\RSnew{\mathcal{C}_R)) \le \frac{\sqrt{2}d\sqrt{d+1}}{\sqrt{\pi}(d-1)^2}\frac{\mathrm{vol}_{\mathbb{R}^{d+1}}(\mathcal{C}_R)}{R^2}\frac{c^{d-1}}{(\max|v_i|)^d}\epsilon^d},
\]
and a union bound yields \eqref{1lnn-eps-2}. Full details are in Appendix~\ref{pf-sec3}.
\end{proof}

\begin{remark}[Vanishing \RSnew{r}elative \RSnew{v}olume]
As with linear regression, the relative volume of the forging set 
$\RSnew{\frac{\mu(S_\epsilon\cap \mathcal{C}_R)}{\mathrm{vol}_{\mathbb{R}^{d+1}}(\mathcal{C}_R)}}$
decays as \(R \to \infty\). Thus, in the limit of a large ambient domain, the forging set of one-layer neural network also occupies a negligible fraction of the space.
\end{remark}
\begin{remark}[Dimension-\RSnew{w}idth \RSnew{t}radeoff]
 The combinatorial term $\sum_{k=0}^d \binom{n}{k}$, which appears in the Lebesgue measure bounds for the $\epsilon$-forging set $S_\epsilon$, can be simplified depending on the relationship between the data dimension $d$ and the hidden layer width $n$. When $d\geq n$, the sum simplifies to $\sum_{k=0}^d \binom{n}{k}= 2^n$. On the other hand, when $d\leq n$ (see Chapter 1.2 of \cite{matouvsek2001probabilistic})
\(\sum_{k=0}^d \binom{n}{k}\leq (d+1)\Big(\frac{en}{d}\Big)^d.\) 
Substituting  into \eqref{1lnn-eps-1} and \eqref{1lnn-eps-2}, we see that when $d\leq n$
\begin{align}\label{bound1*}
\RSnew{\frac{\mu(S_\epsilon\cap \mathcal{C}_R)}{\mathrm{vol}_{\mathbb{R}^{d+1}}(\mathcal{C}_R)}\leq\,\frac{d(d+1)}{d-1}\frac{1}{R^2}\frac{1}{\min_{v_i\neq 0}\{|v_i|\}}\Big(\frac{en}{d}\Big)^d\epsilon}.
\end{align}
and that for sufficiently small $\epsilon$
\begin{align}\label{bound2*}
\RSnew{\frac{\mu(S_\epsilon\cap \mathcal{C}_R)}{\mathrm{vol}_{\mathbb{R}^{d+1}}(\mathcal{C}_R)}\leq \frac{\sqrt{2}d(d+1)^{3/2}}{\sqrt{\pi}(d-1)^2}\, \frac{1}{R^2}\frac{c^{d-1}}{(\min_{v_i\neq0}|v_i|)^d}\Big(\frac{en}{d}\Big)^d\,\epsilon^d}.
\end{align}
\end{remark}

\subsection{Anti-concentration bounds}
The fact that forging sets have small Lebesgue measure suggests that under reasonable probability distributions it should be unlikely to randomly sample a data point from a  forging set. We now provide results demonstrating that is indeed the case. We derive probability bounds for linear regression and one-layer neural networks under the following assumptions.
\paragraph{Assumptions.}\label{assum-prob-case} Let $\mathcal{D}$ be a probability distribution on $\mathbb{R}^d\times\mathbb{R}$, and let 
$\RSnew{K}= C_1 \times C_2\subset\mathbb{R}^d\times\mathbb{R}$\RSnew{,} where $C_1$ and $C_2$ are compact sets with radius $R_1$ and $R_2$, respectively. 
Assume that the joint density $p(\bm x, y)$ associated with $\mathcal{D}$ exists and satisfies the following conditions.
\begin{itemize}
\item[\emph{(i)}] $p(\bm x, y)$ is proportional to $ e^{-g(\bm x, y)}$, where $g:\mathbb{R}^d\times\mathbb{R}\to\mathbb{R}$ satisfies the Lipschitz condition that there exists a constant $L_g>0$ such that for all \((\bm x_1, y_1), (\bm x_2, y_2)\in \RSnew{K}\),\[|g(\bm x_1, y_1) - g(\bm x_2, y_2)|\leq L_g\|(\bm x_1, y_1) - (\bm x_2, y_2)||,\]
\item[\emph{(ii)}]
There exists $(\bm x_c, y_c) \in \RSnew{K}$ and constants $C>0$ and $\omega>0$ such that for all $t\geq t_0$,
\[\mathbb{P}\Bigl(\|(\bm x,y)-(\bm x_c,y_c)\|>t\Bigr) \leq C\,e^{-t^\omega} \]
where $t_0 =  \sup\{r>0:\overline{B_r(\bm x_c,y_c)}\subseteq \RSnew{K}\} $. 
\end{itemize}

Under these assumptions, we prove a bound on the probability of drawing a point from a set with a given Lebesgue measure in \Cref{probass} (Appendix \ref{prob-sec3}). Combining this with the results from the previous subsections, we obtain probability bounds for drawing a forging data point for linear regression and a one-layer neural network. We start with linear regression, and a consequence of \Cref{prop:lr-eps-forging}.
\vspace{1em}
\begin{corollary}\label{corlr-p}
    Under the assumption of \Cref{assum-prob-case}, for $\epsilon>0$ and any $(\bm x,y)$,  the $\epsilon$-forging set $S_\epsilon$ in linear regression \eqref{def: lr-eps-forging-set} satisfies     %
\begin{equation}\label{cor:prob-lr}
\mathbb{P}_{\mathcal{D}}\Big((\bm z,t)\in S_\epsilon\Big)\,\leq\, C_{L_g, \RSnew{K}}\frac{d}{(d-1)R_1R_2}\,\epsilon \, + \,Ce^{-(\frac{\text{diam}(\RSnew{K})}{2})^\omega}
    \end{equation}
    where $C_{L_g, \RSnew{K}} = e^{L_g\,\text{diam}(\RSnew{K})}$. 
    Furthermore, if $\frac{\epsilon}{A}<\sin(c\epsilon)$ for some $c\in[\frac{1}{A},\frac{\pi}{2A}]$, where $A = \|\nabla_{\bm w} f\left(\bm w; (\bm x,y)\right)\|$, then
    \begin{equation}\label{cor:prop-lr2}
    \mathbb{P}_{\mathcal{D}}\Big((\bm z,t)\in S_\epsilon\Big)\,\leq\,\RSnew{C_{L_g, K}\,\frac{\sqrt{2}d\sqrt{d+1}c^{d-1}}{\sqrt{\pi}(d-1)^2}\,\frac{1}{R_1R_2}\;\epsilon^d + \,Ce^{-(\frac{\text{diam}(\RSnew{K})}{2})^\omega}}.
    \end{equation}
\end{corollary}
\begin{proof}
\RSnew{Applying \Cref{probass} with \eqref{lr-eps-1} and \eqref{lr-eps-2}, yields \eqref{cor:prob-lr} and \eqref{cor:prop-lr2}} respectively.
\end{proof}
We can apply the same technique to one-layer neural networks using the results in \Cref{prop: nn-eps-forging}.
\begin{corollary}\label{cornn-p}
     Under the assumption of \Cref{assum-prob-case}, for any $\epsilon>0$ and any $(\bm x,y)$,
      the $\epsilon$-forging set $S_\epsilon$ in one-layer neural networks \eqref{def:nn-eps-forging} satisfies
     \[
\mathbb{P}_{\mathcal{D}}\Big((\bm z,t)\in S_\epsilon\Big)\,\leq\, C_{L_g, \RSnew{K}}\frac{d}{(d-1)R_1R_2}\frac{1}{\min_{v_i\neq 0}\{|v_i|\}}\sum^{d}_{k=0}\begin{pmatrix}
    n\\k
\end{pmatrix} \epsilon + Ce^{-(\text{diam}(\RSnew{K})/2)^\omega}
    \]
    where $C_{L_g, \RSnew{K}} = e^{L_g\,\text{diam}(\RSnew{K})}$. 
    If 
    \RSnew{there exists $c>0$ with $\frac{\epsilon}{{A_i}}<\sin(c\epsilon)$,  for all $i$, where $A_i = \|\nabla_{\bm W}f\left(\bm W,\bm v; (\bm x,y)\right)_i^T\|$, then}
    \[
    \mathbb{P}_{\mathcal{D}}\Big((\bm z,t)\in S_\epsilon\Big)\,\leq\, \RSnew{C_{L_g, K}\frac{\sqrt{2}d\sqrt{d+1}c^{d-1}}{\sqrt{\pi}(d-1)^2}\, \frac{1}{R_1R_2}\frac{1}{(\min_{v_i\neq0}|v_i|)^d}\sum^{d}_{k=0}\begin{pmatrix}
    n\\k
\end{pmatrix}\,\epsilon^d} + Ce^{-(\text{diam}(\RSnew{K})/2)^\omega}.
    \]
\end{corollary}
\begin{proof}
As before, directly apply \Cref{probass} with \eqref{1lnn-eps-1} and \eqref{1lnn-eps-2}.
\end{proof}

\section{Forging for smooth loss functions}\label{generalforgeanalysismainsec}

We now turn to the analysis of general smooth loss functions, aiming to characterize the volume of forging sets under minimal assumptions. This broader perspective provides a unified framework that applies to a wide range of problems, including linear regression and neural networks with smooth activation functions, without the need for case-by-case treatment. However, the sharper bounds obtained in the previous section for the specific problems of linear regression and one-layer neural networks rely on stronger, problem-specific structure, and are therefore not fully encompassed by the forthcoming results. Because our analysis here prioritizes generality over specialization, the resulting bounds may not always be sharp, but this is an expected trade-off. As before, we have 
\begin{align}
\x_{k+1} = \x_k - h \nabla f(\x_k; \z_k), \label{gd1}  
\end{align}
where now $f : \mathbb{R}^n \times \mathcal{Z} \to \mathbb{R}$ is $\mathcal{C}^1$-smooth in its first argument (the parameter), and \RSnew{$\mathcal{Z} \cong \mathbb{R}^d$ is the standard Euclidean space in $d$ dimensions that corresponds to the data manifold. We derive all the results in this work for the Euclidean space $\mathcal{Z} \cong \mathbb{R}^d$ and conjecture that these results can be extended to smooth non-Euclidean data manifolds using appropriate charts with local diffeomorphisms.} We leave this for future work.  Recall also that the iteration \eqref{gd1} may originate from a stochastic algorithm 
or it may be deterministic when ${\z_k}$ is any fixed sequence from $\mathcal{Z}$. The distinction is immaterial for our purposes, as we assume that the full trajectory ${\x_k}$ is fixed in advance.

 Let $\mathcal{Z} \cong \mathbb{R}^d$, $\mu_1, \mu_2$ be the Lebesgue measures on $ \mathbb{R}^n, \mathbb{R}^d$ respectively and the product measure $ \mu_1 \bigotimes \mu_2$ be the Lebesgue measure on $ \mathbb{R}^n \times \mathbb{R}^d$. Further, let $\pi_1, \pi_2$ be the projection maps defined as $ \pi_1 : \mathbb{R}^n \times \mathcal{Z}\to \mathbb{R}^n $, $ \pi_2 : \mathbb{R}^n \times \mathcal{Z} \to \mathcal{Z}$. Then we make the following assumptions on the function $f$.
\paragraph{Assumptions}
\begin{enumerate}
\item[\textbf{A1.}]\textbf{(Smoothness)} The function $f$ is jointly $\mathcal{C}^2$ smooth $\mu_1 \bigotimes \mu_2$ a.e. on $ \mathbb{R}^n \times \mathcal{Z} \cong \mathbb{R}^n \times \mathbb{R}^d$ and 
$$ f \in \mathcal{C}^2((\mathbb{R}^n \times \mathcal{Z}) \backslash V) \hspace{0.1cm}$$
where the set $ V \subset \mathbb{R}^n \times \mathcal{Z}$ is closed and $ \mu_1 \bigotimes \mu_2(V) = 0$.  
 \item[\textbf{A2. }]\textbf{(Lipschitz regularity of second variations)} The second variation matrix function $\nabla_{\z}\nabla_{\x}f(\cdot \hspace{0.1cm}; \hspace{0.1cm}\cdot)$ defined on $ (\mathbb{R}^n \times \mathcal{Z}) \backslash V$ is locally Lipschitz continuous with respect to the operator norm on every compact set of $(\mathbb{R}^n \times \mathcal{Z}) \backslash V$.
\RSnew{\item[\textbf{A3.}] \textbf{(Quantitative nondegeneracy in the data variable)}
There exists some compact, convex set $ D_1 \times D_2 \Subset \mathbb{R}^n \times \mathcal{Z} \cong \mathbb{R}^n \times \mathbb{R}^d$ with non-empty interior and constants $\gamma>0$, $r<d$ that depend only on $  D_1 \times D_2$ such that, for every $(\x,\z)\in D_1\times D_2$, at most $r$ singular values of the $n \times d$ matrix
\(
    \nabla_{\z}\nabla_{\x}f(\x;\z)
\)
are strictly smaller than $\gamma$.
Equivalently, if
\[
    r_\gamma(\x,\z)
    :=
    \#\Big\{j\in\{1,\dots,d\}:\sigma_j\big(\nabla_{\z}\nabla_{\x}f(\x;\z)\big)<\gamma\Big\},
\]
where $\sigma_1\geq \cdots \geq \sigma_d\geq 0$ are the singular values, then
\(
    r_\gamma(\x,\z)\leq r \text{ for every }
      (\x,\z)\in D_1\times D_2.
\)}
\end{enumerate}

\RSnew{Assumption \textbf{A3} ensures  the model gradient has a local variation in at least one direction in $\mathbb{R}^d$ around any  data point. This is because $\gamma >0$ and $r_{\gamma} \leq r <d $ so for every $(\w,\z) \in D_1 \times D_2$, $\nabla_{\z}\nabla_{\x}f(\x;\z) $ has at least one singular value greater than or equal to $\gamma$. Without \textbf{A3} the model gradient can be locally constant in some open neighborhood of the data point and one could forge the model gradient using any data point from such neighborhood. To prevent such pathological cases that give rise to trivial forging sets, we require  \textbf{A3}.}

These assumptions cover a broad class of learning/unlearning models and several standard setups satisfy \textbf{A1-A3} outright. These include quadratic loss with analytic activations in neural networks, as well as classical linear regression (see Appendix~\ref{degenerateassumptionsec}). In fact, consider any  $\mathcal{C}^2$ loss function whose joint second derivative is locally Lipschitz continuous. Such functions when combined with neural networks using smooth activation functions (e.g., sigmoid, tanh) satisfy \textbf{A1–A2}.
Even with quadratic loss and non-smooth activations such as leaky \texttt{ReLU}, \textbf{A1--A2} continue to hold (see Appendix~\ref{reluassumptionsec}). Finally, the non-degeneracy condition \textbf{A3}, which holds in settings like linear regression is discussed more generally in Appendix~\ref{degenerateassumptionsec}. With these conditions in hand, we now derive volume bounds for forging sets.

We first assume, without loss of generality, that  
\[
    f \in \mathcal{C}^2(\mathbb{R}^n \times \mathcal{Z}),
\]  
or equivalently $V = \emptyset$, so that non-differentiability issues do not arise. Since $f$ is jointly $\mathcal{C}^2$ $\mu_1 \otimes \mu_2$-a.e. on $\mathbb{R}^n \times \mathcal{Z}$, results established under global differentiability will naturally extend to the almost-everywhere setting.  
We restrict our forging analysis to a compact, convex set  
\[
    D_1 \times D_2 \Subset \mathbb{R}^n \times \mathcal{Z} \cong \mathbb{R}^n \times \mathbb{R}^d,
\]  
\RSnew{satisfying \textbf{A3}} where both $D_1$ and $D_2$ have non-empty interior. By \textbf{Assumption~A2}, $\nabla_{\z}\nabla_{\x}f(\cdot\,;\,\cdot)$ is $L$-Lipschitz continuous on $D_1 \times D_2$, with the constant $L$ depending only on this compact set.  Formally, $L$-Lipschitz continuity means that for any $(\x_1,\z_1), (\x_2,\z_2) \in D_1 \times D_2$,  
\begin{align}
    \big\|\nabla_{\z}\nabla_{\x}f(\x_1; \z_1) - \nabla_{\z}\nabla_{\x}f(\x_2; \z_2) \big\| 
    \leq L \left\| 
    \begin{bmatrix} \x_1 - \x_2 \\[2pt] \z_1 - \z_2 \end{bmatrix} 
    \right\|. \nonumber
\end{align}

We recall the definition of $\epsilon$ forging set for any data point $\z^* \in D_2$ 
\begin{align}
   S_{\epsilon}(\x,\z^*) = \{\z \in D_2 : \norm{\nabla f(\x; \z) - \nabla f(\x; \z^*)} \leq \epsilon\}. \nonumber
\end{align}

and now establish a key result on the second variation matrix $\nabla_{\z} \nabla_{\x} f(\x; \z^*)$.  
\begin{lemma}\label{seconvarlem1}
   \RSnew{Suppose \textbf{A1--A3} hold and $V=\emptyset$. For any $\x\in D_1$ and $\z^*\in D_2$, we have
    \begin{align}
        \|\nabla_{\z}\nabla_{\x}f(\x;\z^*)(\z^*-\z)\|
        \leq
        \|\nabla_{\x}f(\x;\z^*)-\nabla_{\x}f(\x;\z)\|
        +
        \frac{L}{2}\|\z^*-\z\|^2 .
        \label{eq:lem1_main}
    \end{align}
    In particular, if $\|\z^*-\z\|\leq \sqrt{\frac{2\epsilon}{L}}$ and $\z\in S_\epsilon(\x,\z^*)$, then
    \begin{align}
        \|\nabla_{\z}\nabla_{\x}f(\x;\z^*)(\z^*-\z)\|
        \leq 2\epsilon .
    \end{align}}
\end{lemma}
The proof of Lemma \ref{seconvarlem1} is in Appendix \ref{sec4proof1}.
Using Lemma~\ref{seconvarlem1}, we can estimate the local volume of points near $\z^*$, that $\epsilon$-forge $\z^*$. 

\subsection{Volume bounds}

Before deriving general volume bounds for $\epsilon$-forging sets, we present a lemma that provides a bound for the volume of local forging regions.

\begin{lemma}\label{seconvarlem2}
\RSnew{ Suppose \textbf{A1--A3} hold and $V=\emptyset$. Fix $\x\in D_1$ and $\z^*\in D_2$, and let
    \(
        \M_0(\z^*) := \nabla_{\z}\nabla_{\x}f(\x;\z^*)\in \mathbb{R}^{n\times d}.
    \)
    Let $E_{<\gamma}(\z^*)\subset \mathbb{R}^d$ denote the span of the right singular vectors of $\M_0(\z^*)$ corresponding to singular values strictly smaller than $\gamma$, and let
    \(
        E_{\geq \gamma}(\z^*) := E_{<\gamma}(\z^*)^\perp.
    \)
    For any $\eta>0$, if $\z$ $\eta$-forges $\z^*$ and
    \(
        \z\in \mathcal{B}_{\sqrt{\frac{2\eta}{L}}}(\z^*),
    \)
    then
    \begin{align}
        \z-\z^*
        \in
        \Big(\mathcal{B}_{\sqrt{\frac{2\eta}{L}}}(\mathbf{0})\cap E_{<\gamma}(\z^*)\Big)
        \oplus
        \Big(\mathcal{B}_{\frac{2\eta}{\gamma}}(\mathbf{0})\cap E_{\geq \gamma}(\z^*)\Big).
        \label{eq:local_spectral_inclusion}
    \end{align}
    Consequently, for all $\eta > 0$, with
    \(        k
        :=
        \mathrm{dim}  (E_{<\gamma}(\z^*))
    \), we have
    \begin{align}
        \mathrm{vol}_{\mathbb{R}^d}\Big(
        S_\eta(\x,\z^*)\cap \mathcal{B}_{\sqrt{\frac{2\eta}{L}}}(\z^*)
        \Big)
        \leq
        2^d
         \left(\sqrt{\frac{2}{L}}\right)^k
        \left(\frac{2}{\gamma}\right)^{d-k}
        \eta^{\,d-\frac{k}{2}}.
        \label{eq:local_spectral_volume}
    \end{align}
    }
\end{lemma}
The proof of Lemma \ref{seconvarlem2} is in Appendix \ref{sec4proof2}. 
\RSnew{In Lemma~\ref{seconvarlem2},
\(
    \Big(\mathcal{B}_{\sqrt{\frac{2\eta}{L}}}(\mathbf{0})\cap E_{<\gamma}(\z^*)\Big)
        \oplus
        \Big(\mathcal{B}_{\frac{2\eta}{\gamma}}(\mathbf{0})\cap E_{\geq \gamma}(\z^*)\Big)
\)
is a 
thickening of the vector space spanned by the right singular vectors of $\M_0(\z^*)$ with singular values strictly smaller than $\gamma$.  
The upper bound 
estimates the volume of this 
thickening inside the ball $\mathcal{B}_{\sqrt{\frac{2\eta}{L}}}(\mathbf{0})$.  
Adding $\z^*$ to the set  translates it and does not affect its volume. }
The next theorem extends this local volume bound to the entire compact, convex set $D_2$ via a covering argument.

\begin{theorem}\label{seconvarthm1}
\RSnew{ Suppose \textbf{A1--A3} hold and $V=\emptyset$. Fix $\x\in D_1$ and $\z^*\in D_2$, and recall that
    \(
        S_\epsilon(\x,\z^*)
        :=
        \Big\{
        \z\in D_2:
        \|\nabla_{\x}f(\x;\z)-\nabla_{\x}f(\x;\z^*)\|
        \leq \epsilon
        \Big\}.
    \)
    Assume  that $D_2\subset \mathbb{R}^d$ is compact and convex, and that there exist $\c\in \mathbb{R}^d$ and $\tau>0$ such that
    \(
        \mathcal{B}_\tau(\c)\subset D_2.
    \)
    Then for every
    \(
        0<\epsilon\leq \min\left\{\frac{\gamma^2}{4L },\frac{L\tau^2}{2}\right\},
    \)
    we have
    \begin{align}
        \mu_2\big(S_\epsilon(\x,\z^*)\big)
        \leq
        C_{D_2,\gamma,d,L}\,
        \epsilon^{\frac{d-r}{2}},
        \label{eq:global_measure_bound}
    \end{align} 
   where one may take
    \(  C_{D_2,\gamma,d,L}
        :=
        24^d \,
        \frac{\Gamma(\frac{d}{2}+1)}{\pi^{d/2}}
        \,\mathrm{vol}_{\mathbb{R}^d}(D_2)\,
        \left(\frac{L}{2 \gamma^2}\right)^{d/2}
       \left( \frac{\gamma^2}{4L } \right)^{r/2}.
    \)
    }
\end{theorem}
The proof of Theorem \ref{seconvarthm1} is in Appendix \ref{sec4proof3}.

\begin{remark}
\RSnew{ It may, in general, be difficult to control $r$ (from \textbf{A3}) over $D_2$. 
 However, in certain cases it  can be done  using  properties of the loss function $f$. In Appendix~\ref{degenerateassumptionsec}, we compute  a bound for linear regression.}
\end{remark}

\begin{remark}[Limiting behavior as $d$ grows]
  From Theorem~\ref{seconvarthm1},
\begin{align}
    \mu_2\big( S_{\epsilon}(\x, \z^*) \big) 
    & \le 
    \RSnew{  24^d \,
        \frac{\Gamma(\frac{d}{2}+1)}{\pi^{d/2}}
        \,\mathrm{vol}_{\mathbb{R}^d}(D_2)\,
        \left(\frac{L}{2 \gamma^2}\right)^{d/2}
       \left( \frac{\gamma^2}{4L } \right)^{r/2} \epsilon^{\frac{d- r }{2}}} \\
    & \lesssim_{C_1(d)}
    \RSnew{\sqrt{\pi d } \left( \frac{144 L\, d \, (\mathrm{diam}(D_2))^2}{\pi e  \gamma^2 } \right)^{d/2}
    \epsilon^{\frac{d- r}{2}}}, \label{thm2simple1}
\end{align}
where in the last step we used  $\mathrm{vol}_{\mathbb{R}^d}(D_2) \leq  (\mathrm{diam}(D_2))^d$,
\(
\RSnew{  \left( \frac{\gamma^2}{4L } \right)^{r/2} \leq  1 }
\)
\RSnew{for $L > \frac{\gamma^2}{4 } $}, 
together with 
$\Gamma(\frac{d}{2}+1) = (\frac{d}{2})!$ 
and Stirling's approximation $(\frac{d}{2})! \sim \sqrt{\pi d} (\frac{d}{2e})^{d/2}$ for large $d$.
Here $C_1(d) = 1+\mathcal{O}(1/d)$.

For the special case \RSnew{$ r = o(d)$
(e.g., $r = 2$ for  linear regression, see Appendix~\ref{degenerateassumptionsec})},
we have
\[
\begin{aligned}
    \mu_2\big( S_{\epsilon}(\x, \z^*) \big) 
     \lesssim_{C_1(d)}
    \RSnew{{\sqrt{\pi d}}\left( \frac{144 L\, d \, (\mathrm{diam}(D_2))^2 \, \epsilon^{1 - o(d)/d}}{\pi e \gamma^2} \right)^{d/2}},
\end{aligned}
\]
where we assume that the local Lipschitz parameter $L: = L(d)$ is a function of $d$ and $L \to \infty$ as $d \to \infty$. Then for fixed $\epsilon$, the right-hand side grows without bound as $d \to \infty$.
Thus, if $\epsilon = \epsilon(d)$ depends on $d$, 
a sufficient condition for $\mu_2(S_{\epsilon}(\x,\z^*)) \to 0$ as $d \to \infty$ is
\begin{equation}\label{eq:eps_bigO}
\epsilon = \mathcal{O}\!\left( (L \RSnew{(\mathrm{diam}(D_2))^2} )^{-(1+a)} d^{-\frac{(1+a)(d+1)}{d}}\right) 
\quad \forall \ a>0,
\end{equation}
where $a$ is independent of $d$.  
In particular, if \RSnew{$ r = o(d)$}, 
Theorem~\ref{seconvarthm1} and \eqref{eq:eps_bigO} imply
\[
\lim_{\substack{d \to \infty \\ \epsilon = \mathcal{O}\bigg(  (L \RSnew{(\mathrm{diam}(D_2))^2} )^{-(1+a)} d^{-\frac{(1+a)(d+1)}{d}}\bigg),\ \forall a>0}}
\mu_2\big( S_{\epsilon}(\x, \z^*) \big) = 0.
\]
\RSnew{For linear regression in $\mathbb{R}^d$ and a certain choice of $D_1, D_2$ (see Appendix \ref{degenerateassumptionsec_aux}), it can be shown that for any $a>0$ if $$ \epsilon = \epsilon(d) : =  (d+1)^{-2.5(1+a)} (d+1)^{-\frac{(1+a)}{(d+1)}} $$ then the volume of $\epsilon$ forging set vanishes as $d \to \infty$ . }
\end{remark}

\subsection{Anti-concentration of probability measure for $\epsilon$-forging}\label{antieps}

Building on the local volume bounds (Lemmas~\ref{seconvarlem1}--\ref{seconvarlem2}) and the global volume bound (Theorem~\ref{seconvarthm1}), we now convert these geometric controls into probability bounds. Inside $D_2$, probability compares to volume via a locally log-Lipschitz density. Outside $D_2$, a tail concentration controls the remainder.
Assume $\probP \ll \mu_2$ on $\mathbb{R}^d$ with density $p(\z)$ and:
\begin{enumerate}
\item[\textbf{P1.}] $p(\z) \propto e^{-g(\z)}$ for a continuous $g:\mathbb{R}^d\to\mathbb{R}$ that is locally Lipschitz on every compact set of $\mathbb{R}^d$.

\item[\textbf{P2.}] Let $\z_c := \frac{1}{\mathrm{vol}_{\mathbb{R}^d}(D_2)} \int_{D_2} \z \, d\mu_2$ be the center of the compact, convex, non-empty set $D_2$. Then
\[
\probP(\{\z:\|\z-\z_c\|\ge t\}) \le C e^{-t^{\omega}}
\quad\text{for some $\omega>0$ and all } t \ge t_0 := \sup\{r>0:\overline{\mathcal{B}_{r}(\z_c)} \subseteq D_2\}.
\]
\end{enumerate}

\medskip

\begin{theorem}\label{seconvarthm1p}
\RSnew{ Assume \textbf{A1--A3}  and $V=\emptyset$. Fix $\x\in D_1$ and $\z^*\in D_2$, and assume that $D_2\subset \mathbb{R}^d$ is compact and convex, and that there exist $\c\in \mathbb{R}^d$ and $\tau>0$ such that
    \(
        \mathcal{B}_\tau(\c)\subset D_2.
    \)
    }
Suppose $\probP$ satisfies \textbf{P1--P2}, and let $L_g$ be the local Lipschitz constant of $g$ on $D_2$. Then \RSnew{for every
    \(
        0<\epsilon\leq \min\left\{\frac{\gamma^2}{4L },\frac{L\tau^2}{2}\right\},
    \)
    we have}
  \begin{align}
      \probP\bigg( \{\z \in \mathbb{R}^d : \norm{\nabla f(\x; \z) - \nabla f(\x; \z^*)} \leq \epsilon\}\bigg) &    \leq \RSnew{C'_{D_2,\gamma,d,L}\,
        \epsilon^{\frac{d-r}{2}}}  + C  e^{-t_0^{\omega}} \nonumber
    \end{align} 
    \RSnew{where one may take
    \(  C'_{D_2,\gamma,d,L}
        :=
        24^d \,
        \frac{\Gamma(\frac{d}{2}+1)}{\pi^{d/2}}
        \, e^{L_g \mathrm{diam}(D_2)} \,
        \left(\frac{L}{2 \gamma^2}\right)^{d/2}
       \left( \frac{\gamma^2}{4L } \right)^{r/2}.
    \)}
\end{theorem}

The proof of Theorem \ref{seconvarthm1p} is in Appendix \ref{prob-sec4app}. Note that if $D_2$ is a closed ball in $\mathbb{R}^d$, then $t_0=\mathrm{diam}(D_2)/2$. 
The bound may be optimized by scaling $t_0$ (equivalently $\mathrm{diam}(D_2)$), bearing in mind that the constants $L$ and $L_g$ depend on $D_2$ and can scale with $\mathrm{diam}(D_2)$.

\subsection{Volume estimates of forging sets under different data regimes}\label{subsec:n-vs-d}

We now examine how the volume bounds scale with the relative sizes of the model dimension $n$ and the data dimension $d$, which may be relevant in various machine learning contexts. Indeed, recall that $d$ denotes the intrinsic data/input dimension (e.g., pixels, patch or token embeddings, feature vectors), and $n$ denotes the number of trainable parameters (globally or for the layer/block in focus) that influence $\nabla_{\x} f(\x;\z)$. In modern deep networks, both regimes can arise naturally. Early convolutional layers can be effectively underparameterized ($d \ge n$) due to high-resolution inputs and weight sharing associated with convolutions. Meanwhile, wide fully connected layers or attention layers, and later dense layers are often overparameterized ($d < n$). Our bounds  predict larger forging sets in \RSnew{underparameterized} settings, precisely where $\nabla_{\z}\nabla_{\x} f$ tends to have a larger null space relative to $d$.\footnote{Here “over/underparameterized” refers to the parameter–input relation ($n$ vs.\ $d$), not to the sample-size relation used elsewhere in learning theory.} \RSnew{We consider representative achievable scenarios where the lower bound on $r$ is attained. These constructions highlight how the interplay between $n$ and $d$ can induce qualitatively different scaling of the forging-set volume.}
The key driver is the \RSnew{parameter $r$ from \textbf{A3} 
}
which enters Theorem~\ref{seconvarthm1} through the factors \RSnew{$\left( \frac{\gamma^2}{4L } \right)^{r/2}$ and $\epsilon^{\frac{d-r}{2}}$}.  Rank--nullity yields
\begin{align}
\ker(\M_0(\z^*)) \oplus \ker(\M_0(\z^*))^{\perp} \cong \mathbb{R}^d, \label{genforge006}
\end{align}
and
\begin{align}
\dim(\ker(\M_0(\z^*))) + \dim (\mathrm{range}(\M_0(\z^*))) = d \label{genforge007}
\end{align}
\RSnew{where $r \ge r_{\gamma} \ge   \dim(\ker(\M_0(\z^*)))$ from \textbf{A3}.}
Intuitively, larger $r$ enlarges directions in $\mathcal{Z}$ where gradients change little, and thus tends to increase forging-set volume.

\subsubsection*{Case 1: Data dimension is dominant, i.e., $ d \ge n $ } \label{underparacase1}
Since $\M_0(\z^*)$ has rank at most $n$, we have
\begin{align}
0 \le d-n \le \dim(\ker(\M_0(\z^*))) \RSnew{\le r} \le d-1 \, . \label{genforge02}
\end{align}

\RSnew{The lower bound $r=d-n$ is achievable by construction. For example, consider  $f(\w,\z)  :=  \langle \w, \A\z\rangle + \lambda \norm{\w}^2 $ with regularizer $\lambda>0$ and $\A \in \mathbb{R}^{n \times d}$ satisfying $\mathrm{dim}\, \mathrm{ker}(\A) = d-n$ and $\min_{1 \le i \le n} \sigma_i(\A) > \gamma$.  Then $\nabla_{\z}\nabla_{\x} f(\w,\z) = \A $. By construction, $\A$ has exactly $d-n$ zero singular values, and all nonzero singular values are bounded below by $\gamma$. Hence, the number of singular values strictly smaller than $\gamma$ is exactly $r = d-n$. Applying Theorem~\ref{seconvarthm1} ($\frac{\gamma^2}{4L\epsilon }>1$) with $r = d-n $ yields}
\begin{align}
\mu_2\big( S_{\epsilon}(\x, \z^*) \big) 
\le \RSnew{24^d \,
        \frac{\Gamma(\frac{d}{2}+1)}{\pi^{d/2}}
        \,\mathrm{vol}_{\mathbb{R}^d}(D_2)\,
        \left(\frac{L}{2 \gamma^2}\right)^{d/2}
       \left( \frac{\gamma^2}{4L \epsilon} \right)^{\frac{d-n}{2}}  \epsilon^{\frac{d}{2}}} . \label{genforge99}
\end{align}

\subsubsection*{Case 2: Model dimension is dominant, i.e., $ n > d $ } \label{overparacase2}
Here
\begin{align}
0 \le \dim(\ker(\M_0(\z^*))) \RSnew{\le r}\le d-1 \, . \label{genforge029}
\end{align}
\RSnew{In this regime, the minimal possible value $r = 0$ is also achievable. For example, consider the same function $f(\w,\z)  :=  \langle \w, \A\z\rangle + \lambda \norm{\w}^2 $ where $\A \in \mathbb{R}^{n \times d}$ has full column rank and
$\min_{1 \le i \le d} \sigma_i(\A) > \gamma$. Then $\nabla_{\z}\nabla_{\x} f(\w,\z) = \A$
and no singular values fall below $\gamma$, so $r = 0$.
Using 
$r=0$ and applying Theorem~\ref{seconvarthm1} ($\frac{\gamma^2}{4L\epsilon } >1$)} gives
\begin{align}
\mu_2\big( S_{\epsilon}(\x, \z^*) \big) 
\le \RSnew{24^d \,
        \frac{\Gamma(\frac{d}{2}+1)}{\pi^{d/2}}
        \,\mathrm{vol}_{\mathbb{R}^d}(D_2)\,
        \left(\frac{L}{2 \gamma^2}\right)^{d/2}
        \epsilon^{\frac{d }{2}}} . \label{genforge23c}
\end{align}
\RSnew{Clearly, the bound in \eqref{genforge23c} from Case 2  is smaller compared to the bound in \eqref{genforge99} from Case 1. This is because $r$ can be much smaller (close to $0$) when $ n>d$ thereby creating fewer directions in $\mathbb{R}^d$ where the second variation norm is at most $\gamma$, hence varies little. When $d >n$ we have $r \ge d-n$ so at least $d-n$ directions in $\mathbb{R}^d$ exist where the second variation norm is at most $\gamma$. }
Probability bounds for the two regimes follow by combining Theorem~\ref{seconvarthm1p} with \eqref{genforge99} and \eqref{genforge23c}. We omit the routine substitution.

\section{Forging analysis for batch SGD}\label{forgeinbatch}

We now consider forging when the parameters evolve via batch SGD. A key point that we recall is that the sampling distribution does not matter for the forger. At each step $k$ the mini-batch $\{\z_{k_j}\}_{j=1}^B$ is given, and our bounds are deterministic functions of the mini-batch. We therefore work conditionally on the realized batch sequence and treat $\{\z_{k_j}\}_{k,j}$ as fixed. We also recall that as in Section~\ref{antieps}, probabilistic assumptions are only needed if we wish to convert volume bounds into probability bounds.
We consider the batch-SGD update
\begin{align}
    \x_{k+1} = \x_k -  \frac{h}{B} \sum_{j=1}^B \nabla f(\x_k; \z_{k_j}), \qquad \z_{k_j}\in \mathcal{Z},
    \label{batchsgd}
\end{align}
and assume throughout that $f\in \mathcal{C}^2(\mathbb{R}^n \times \mathcal{Z})$ satisfies Assumptions A1--A3 with $V=\emptyset$. As before, we restrict attention to a compact, convex set $D_1 \times D_2 \Subset \mathbb{R}^n \times \mathbb{R}^d$ with non-empty interiors such that $\{\x_k\}_k \subset D_1$ and $\{\z_{k_j}\}_{k,j} \subset D_2$. By A2, $\nabla_{\z}\nabla_{\x} f$ is $L$-Lipschitz on $D_1 \times D_2$ with $L$ depending  on this compact set.

\begin{remark}
[On smoothness] Batch subgradient methods for merely Lipschitz $f$ are delicate, involving finite sums of subgradients, Clarke calculus, and step-size schedules \cite{rockafellar1998variational, van2024weak}. General convergence guarantees typically require additional structure (e.g., weak convexity or Clarke regularity). To keep the forging analysis tractable and avoid these technicalities, we assume $\mathcal{C}^2$ smoothness on the domain in this section. See also \RSnew{\Cref{aesmoothrem2}}  in \Cref{sectionaeforge} on the technical challenges associated with the analysis of $\epsilon-$forging sets in the context of non-smooth functions.
\end{remark}
Fix a step $k$ and let $\z^*$ be a data point appearing in the batch $\{\z_{k_j}\}_{j=1}^B$ with multiplicity $m>0$. Since the forger knows $f$ and the realized batch, they can replicate the averaged gradient $\frac{h}{B} \sum_{j=1}^B \nabla f(\x_k; \z_{k_j})$ either by replacing only the $m$ occurrences of $\z^*$ or by replacing the entire batch. We first analyze the single-point replacement (replacing the copies of $\z^*$ only) and then obtain the full-batch replacement as a direct consequence in Remark~\ref{remarkbatchsgd}, which simply relies on the insight that replacing the entire batch is equivalent to setting $m=B$.

Because only the $m$ occurrences of $\z^*$ in  $\{\z_{k_j}\}_{j=1}^B$ are replaced while all other batch elements are fixed, the forging constraint at step $k$ depends solely on the  replacements. Thus the relevant event is in $\mathcal{Z}^m$ and any sampling statement is with respect to the product measure $\probP^{\otimes m}$ on $\mathcal{Z}^m$. Define
\begin{align}
    \tilde{S}_{\epsilon}(\x_k, \z^*) 
    := 
    \Bigg\{ (\tilde{\z}_1,\ldots, \tilde{\z}_m ) \in \mathcal{Z}^m \cong \mathbb{R}^{md} 
    :\;
    \Big\|\frac{1}{B}\sum_{j=1}^m \big(\nabla f(\x_k; \z^*) - \nabla f(\x_k; \tilde{\z}_j)\big) \Big\| 
    \le \epsilon 
    \Bigg\}.
    \label{batchforgeset1}
\end{align}
We will bound the volume of the above set. Note that $\probP^{\otimes m}\big(\tilde{S}_{\epsilon}(\x_k,\z^*)\big)$ can be obtained using the same Lipschitz and second-variation controls as in the single-point case. Let $F : \mathbb{R}^n \times \mathbb{R}^{md} \to \mathbb{R}$ be
\[
F(\x; \Z) \equiv \frac{1}{B} \sum_{j=1}^m f(\x;\z_j),
\qquad 
\Z := \begin{bmatrix}\z_1\\ \vdots\\ \z_m\end{bmatrix} \in \mathbb{R}^{md}.
\]
Under Assumption \textbf{A1} with $f \in \mathcal{C}^2(\mathbb{R}^n \times \mathbb{R}^d)$, we have $F \in \mathcal{C}^2(\mathbb{R}^n \times \mathbb{R}^{md})$.
For fixed $\x$, 
\begin{align}
\nabla_{\Z}\nabla_{\x}F(\x;\Z)
= \frac{1}{B}\big[\,\nabla_{\z_1}\nabla_{\x}f(\x;\z_1)\ \big|\ \nabla_{\z_2}\nabla_{\x}f(\x;\z_2)\ \big|\ \cdots\ \big|\ \nabla_{\z_m}\nabla_{\x}f(\x;\z_m)\,\big].
\label{eq:batch-block}
\end{align}

Let $D_2^m := D_2 \times \cdots \times D_2 \subset \mathbb{R}^{md}$. 
Using \textbf{A2} (local Lipschitz continuity of $\nabla_{\z}\nabla_{\x} f$ on $D_1\times D_2$ with constant $L$), we obtain for any fixed $\x\in D_1$ and any $\Z_1,\Z_2 \in D_2^m$,
\begin{align*}
\big\|\nabla_{\Z}\nabla_{\x}F(\x;\Z_1)-\nabla_{\Z}\nabla_{\x}F(\x;\Z_2)\big\|
&\le \frac{1}{B}
\left(\sum_{i=1}^m 
\big\|\nabla_{\z_i}\nabla_{\x}f(\x;\z_i)-\nabla_{\z_i}\nabla_{\x}f(\x;\z_i')\big\|^2
\right)^{1/2} \\
&\le \frac{L}{B}\left(\sum_{i=1}^m \|\z_i-\z_i'\|^2\right)^{1/2}
= \frac{L}{B}\,\|\Z_1-\Z_2\|.
\end{align*}
Here we used the block-operator inequality (Lemma \ref{suplem4}) 
$\big\|\,[\A_1|\cdots|\A_m]\,\big\| \le \big(\sum_{i=1}^m \|\A_i\|^2\big)^{1/2}$ 
for horizontal concatenation of matrices.

Hence, for each fixed $\x\in D_1$, the mixed second variation $\nabla_{\Z}\nabla_{\x}F(\x;\cdot)$ is $(L/B)$-Lipschitz on the closed, convex set $D_2^m$.  \RSnew{Throughout this section, we assume a modified version of assumption \textbf{A3}.}
\RSnew{
\begin{enumerate}
   \item[\textbf{A3'.}] \textbf{(Quantitative nondegeneracy in the batch setting)}
There exists some compact, convex set $ D_1 \times D_2^m \Subset \mathbb{R}^n \times \mathbb{R}^{md}$ with non-empty interior and constants $\gamma>0$, $r<md$ where $\gamma$ depends only on $  D_1 , \, D_2$ and $r$ depends on $  D_1 , \, D_2, \, m$ such that, for every $(\x,(\z_1,\cdots,\z_m))\in D_1\times D_2^m$, at most $r$ singular values of the $n \times md$ concatenated block matrix
\[
    \big[\,\nabla_{\z_1}\nabla_{\x}f(\x;\z_1)\ \big|\ \nabla_{\z_2}\nabla_{\x}f(\x;\z_2)\ \big|\ \cdots\ \big|\ \nabla_{\z_m}\nabla_{\x}f(\x;\z_m)\,\big]
\]
are strictly smaller than $\gamma$. 
\end{enumerate}
}

\RSnew{Observe that in \textbf{A3'}, $\gamma>0$ is independent of $m$. This is not a vacuous condition as, using \textbf{A3}, one can always find $\gamma >0$, $r'<d$ that depend only on $D_1 \times D_2$ such that for any $(\w,\z) \in D_1 \times D_2 $, at most $r'$ singular values of $\nabla_{\z}\nabla_{\x}f(\x;\z) $ are strictly less than $\gamma$. Thereafter, the $n \times md$ concatenated block matrix
\(
    \big[\,\nabla_{\z_1}\nabla_{\x}f(\x;\z_1)\ \big|\ \nabla_{\z_2}\nabla_{\x}f(\x;\z_2)\ \big|\ \cdots\ \big|\ \nabla_{\z_m}\nabla_{\x}f(\x;\z_m)\,\big]
\) will have at most $r<md$ singular values strictly less than $\gamma$ where $r$ will be a function of $ r', m$. The explicit relation $r:= r(r',m)$  can be derived from \textbf{A3} as follows.}
\RSnew{Let $\bm A = \big[\, \bm A_1\ \big|\ \bm A_2\ \big|\ \cdots\ \big|\ \bm A_m\,\big]$ be our concatenated block matrix with $\bm A_i = \nabla_{\z_i}\nabla_{\x}f(\x;\z_i) $ for $ 1\le i\le m$ where each $\bm A_i$ is a matrix with at most $r’$ singular values less than $\gamma$. Then $\bm A_i \bm A_i^T$ has at least $d-r’$ eigenvalues greater than $\gamma^2$. But $\bm A \bm A^T = \sum_{i=1}^m \bm A_i \bm A_i^T $ so $ \bm A \bm A^T \succeq \bm A_i \bm A_i^T \succeq \mathbf{0}$ (in the matrix ordering) which implies the eigenvalues satisfy $\lambda_j (\bm A \bm A^T) \geq \lambda_j (\bm A_i \bm A_i^T)$ for every $i, j$. So $\bm A$ has at least $d-r’$ singular values greater than or equal to $\gamma$ thus it has at most $ md - (d-r’) = (m-1)d +r’$ singular values smaller than $\gamma$. We can therefore take $r = (m-1)d +r’$. Then by \textbf{A3'}, for every $(\x,(\z_1,\cdots,\z_m))\in D_1\times D_2^m$ at most $r<md$ singular values of the $n \times md$ matrix
\(
  \nabla_{\Z}\nabla_{\x}F(\x;\Z)
=  \frac{1}{B}\big[\,\nabla_{\z_1}\nabla_{\x}f(\x;\z_1)\ \big|\ \nabla_{\z_2}\nabla_{\x}f(\x;\z_2)\ \big|\ \cdots\ \big|\ \nabla_{\z_m}\nabla_{\x}f(\x;\z_m)\,\big]
\)
are strictly smaller than $\tfrac{\gamma}{B}$.}

Let $\Z^*=\boldsymbol{1}_m\otimes \z^*$, where $\boldsymbol{1}_m$ is the all-ones vector in $\mathbb{R}^m$. In analogy with \eqref{batchforgeset1}, we define the batched forging set for $F$ by
\begin{align}
    \tilde{S}_{\epsilon}(\x_k,\Z^*)
    :=
    \Big\{\tilde{\Z}\in D_2^m:\ \|\nabla_{\x}F(\x_k;\Z^*)-\nabla_{\x}F(\x_k;\tilde{\Z})\|\le \epsilon\Big\}.
    \label{batchforgeset2}
\end{align}
Since $\nabla_{\x}F(\x;\Z)=\frac{1}{B}\sum_{j=1}^m \nabla_{\x} f(\x;\z_j)$, this definition is equivalent to \eqref{batchforgeset1}.

\RSnew{As in \Cref{subsec:n-vs-d}, we focus on favorable scenarios where the lower bound on $r$ is attained.}
Using the batch mixed-derivative Lipschitz constant $\frac{L}{B}$ from the previous subsection, we now bound the volume of $\tilde{S}_{\epsilon}(\x_k,\Z^*)$ in three cases. 

\subsubsection*{Case 1: Data dimension is dominant, i.e., $ d \geq n $ }
Let $ \M_0(\Z^*) = \nabla_{\Z} \nabla_{\x} F(\x; \Z^*) \in \mathbb{R}^{n \times md}$ with $\Z^* \in D_2^m \Subset \mathbb{R}^{md}$. Then 
\begin{align}
    \ker(\M_0(\Z^*))  \bigoplus  \ker(\M_0(\Z^*))^{\perp} &\cong \mathbb{R}^{md}  \label{genforge006a1} 
\end{align}
and by rank--nullity, 
\begin{align}
     \dim(\ker(\M_0(\Z^*))) + \dim(\mathrm{range}(\M_0(\Z^*))) &= md. \label{genforge007a1}
\end{align}
\RSnew{Since $\dim(\ker(\M_0(\Z^*)))  \le r_{\gamma}  \le r$ from \textbf{A3}}, we have
\begin{align}
    0 \le md-n \le \dim(\ker(\M_0(\Z^*))) \RSnew{\le r }\le md-1 
    \, , \label{genforge02a1}
\end{align}
where the upper bound uses \textbf{A3} for $F$ (the column space is nontrivial) and the lower bound follows since $\mathrm{rank}(\M_0(\Z^*))\le n$.

With Lipschitz constant $\frac{L}{B}$, applying Theorem~\ref{seconvarthm1} to $F$ (with $d\mapsto md$, $L\mapsto L/B$ \RSnew{and $ \gamma \mapsto \gamma/B$}) yields
\begin{align}
    \mu_2^{\otimes m}\!\big( \tilde{S}_{\epsilon}(\x_k, \Z^*) \big) 
    & \le  \RSnew{ 24^{md} \,
        \frac{\Gamma(\frac{{md}}{2}+1)}{\pi^{{md}/2}}
        \,\mathrm{vol}_{\mathbb{R}^{md}}(D^m_2)\,
        \left(\frac{LB}{2 \gamma^2}\right)^{{md}/2}
       \left( \frac{\gamma^2}{4L B \epsilon } \right)^{r/2} \epsilon^{\frac{{md} }{2}} },
    \label{genforge99ba1}
\end{align}
where $r \le md-1$. \RSnew{Specializing to the achievable case where $r$ attains its lower bound, and noting that $\frac{\gamma^2}{4L B\epsilon } > 1$ for $\epsilon <\frac{\gamma^2}{4LB  }$, we use the lower bound from \eqref{genforge02a1} for $r$, namely $r = md-n$, and obtain}
\begin{align}
  \mu_2^{\otimes m}\!\big( \tilde{S}_{\epsilon}(\x_k, \Z^*) \big) 
  & \le \RSnew{ 24^{md} \,
        \frac{\Gamma(\frac{{md}}{2}+1)}{\pi^{{md}/2}}
        \,\mathrm{vol}_{\mathbb{R}^{md}}(D^m_2)\,
        \left(\frac{LB}{2 \gamma^2}\right)^{{md}/2}
       \left( \frac{\gamma^2}{4L B \epsilon } \right)^{\frac{md-n}{2}} \epsilon^{\frac{{md} }{2}} }.
  \label{genforge99a1}
\end{align}

\subsubsection*{Case 2: Model dimension is sub-dominant, i.e., $ md \geq n > d $ }
As in Case~1, $\mathrm{rank}(\M_0(\Z^*)) \le n$, hence
\begin{align}
    0 \le md-n \le \dim(\ker(\M_0(\Z^*))) \RSnew{\le r }\le md-1 
    \, .  \nonumber
\end{align}
 Therefore substituting \RSnew{the lower bound $ r = md-n$} in \eqref{genforge99ba1} yields exactly \eqref{genforge99a1}:
\begin{align}
\mu_2^{\otimes m}\!\big(\tilde{S}_{\epsilon}(\x_k,\Z^*)\big)
\le
\RSnew{ 24^{md} \,
        \frac{\Gamma(\frac{{md}}{2}+1)}{\pi^{{md}/2}}
        \,\mathrm{vol}_{\mathbb{R}^{md}}(D^m_2)\,
        \left(\frac{L B}{2 \gamma^2}\right)^{{md}/2}
       \left( \frac{\gamma^2}{4L B \epsilon } \right)^{\frac{md-n}{2}} \epsilon^{\frac{{md} }{2}} }.  \nonumber
\end{align}

\subsubsection*{Case 3: Model dimension is super-dominant, i.e., $ n > md $ }
Here $\mathrm{rank}(\M_0(\Z_i^*)) \le md$, so
\begin{align}
    0 \le \dim(\ker(\M_0(\Z_i^*))) \RSnew{\le r } \le md-1 
    \, .   \nonumber
\end{align}
\RSnew{When the lower bound $r=0$ is achieved, noting that $\frac{\gamma^2}{4LB\epsilon}>1$ for $\epsilon<\frac{\gamma^2}{4LB}$, we obtain}
\begin{align}
\mu_2^{\otimes m}\!\big(\tilde{S}_{\epsilon}(\x_k,\Z^*)\big)
\le
\RSnew{ 24^{md} \,
        \frac{\Gamma(\frac{{md}}{2}+1)}{\pi^{{md}/2}}
        \,\mathrm{vol}_{\mathbb{R}^{md}}(D^m_2)\,
        \left(\frac{LB}{2 \gamma^2}\right)^{{md}/2}
        \epsilon^{\frac{{md} }{2}} }. \nonumber
\end{align}

\RSnew{Clearly, for $\epsilon <\frac{\gamma^2}{4L B }$, the upper bounds of the volume estimates from cases 1 and 2, where we used $r=md-n$, are greater than the upper bound estimate from case 3 where $r=0$ was used. }

\begin{remark}[Replacing the entire batch]\label{remarkbatchsgd}
{Replacing the entire batch is equivalent to setting $m=B$. So, one can obtain analogous volume bounds by simply replacing $m$ with $B$ in the analyses above. }
\end{remark}
\RSnew{\begin{remark}[Forging volume and batch size]
From \eqref{genforge99ba1}, we obtain
\[
\mu_2^{\otimes m}\!\big( \tilde{S}_{\epsilon}(\x_k, \Z^*) \big)
\le
C_{m,d,r,D_2}
\left(\frac{LB}{\gamma^2 }\right)^{\frac{md-r}{2}}
\epsilon^{\frac{md-r}{2}},
\]
where
\(
C_{m,d,r,D_2}
:=
24^{md}\,2^{-\frac{md}{2}-r}\,
\frac{\Gamma\!\left(\frac{md}{2}+1\right)}{\pi^{md/2}}\,
\mathrm{vol}_{\mathbb{R}^{md}}(D_2^m).
\)
Thus, for fixed $m$, the bound grows like $B^{(md-r)/2}$ as $B$ increases. For example, when $m=1$,
\(
\mu_2\!\big( \tilde{S}_{\epsilon}(\x_k, \Z^*) \big)
\le
C_{1,d,r,D_2}
\left(\frac{LB}{\gamma^2 }\right)^{\frac{d-r}{2}}
\epsilon^{\frac{d-r}{2}}.
\)
Hence, in this regime, the upper bound on the forging volume increases compared to the plain SGD estimate corresponding to $B=1$. 
Indeed, intuitively, from \eqref{batchforgeset1}, the $\epsilon$-forging condition for $m=1$ implies finding $ \tilde{\z}_1 \in D_2$ such that
    \(
        \frac{1}{B} \left\| \nabla f(\x_k; \z^*) - \nabla f(\x_k; \tilde{\z}_1) \right\| 
    \le \epsilon 
    \)
    which is satisfied by any  $\tilde{\z}_1 \in D_2$ as $B \to \infty$.
\end{remark}}

\section{Forging analysis under almost-everywhere smoothness}\label{sectionaeforge}

Having established volume and probability bounds under global $\mathcal{C}^2$ smoothness, we now extend the results of Section~\ref{generalforgeanalysismainsec} to the almost-everywhere smooth setting of \textbf{Assumption~A1}, where
\[
f \in \mathcal{C}^2\big((\mathbb{R}^n \times \mathcal{Z}) \setminus V\big), 
\qquad \mu_1 \!\otimes\! \mu_2(V)=0,\qquad V\ \text{closed, possibly nonempty}.
\]
\RSnew{In particular, \Cref{sectionaeforge} provides a general structural analysis for the almost-everywhere smooth setting. It  shows how forging-set bounds depend on local regularity parameters without exactly quantifying these relationships. Appendix \ref{sectiongeomk1} complements this by introducing stronger assumptions to derive more explicit and quantifiable results. 
}

We begin with some notation and preliminaries. As before, we restrict to compact and convex $D_1\times D_2 \Subset \mathbb{R}^n\times\mathcal{Z}$ with nonempty interiors. By \textbf{Assumption~A2}, $\nabla_{\z}\nabla_{\x} f$ is \emph{locally} Lipschitz on $(\mathbb{R}^n\times\mathcal{Z})\setminus V$ where the Lipschitz constant $L$ depends only on the compact set $ D_1 \times D_2$.
 By Fubini’s theorem, for $\mu_1$-almost every $\x\in\mathbb{R}^n$ the slice
\[
V_2(\x) \ :=\ \pi_2\big(V\cap(\{\x\}\times\mathcal{Z})\big)\ \subset \ \mathcal{Z}
\]
satisfies $\mu_2\!\big(V_2(\x)\big)=0$. 
Moreover, \textbf{Assumption~A3} \RSnew{may not hold everywhere on the compact set $ D_1 \times D_2$ due to non-smoothness of $f$ on $\x \times V_2(\x)$ for $\x \in D_1$. Instead, in this section, for $\mu_1$-a.e. $\w \in D_1$ we will require that \textbf{A3} holds on certain compact subsets of the form $\w \times K_1 \subset D_1 \times D_2$ where $f$ is smooth on $ \w \times K_1$ and $K_1$ will be specified shortly.  }
Since our forging analysis fixes $\x$, we henceforth suppress the $\x$-dependence and write $V_2 := V_2(\x)$. Because $V$ is closed in $\mathbb{R}^n\times\mathcal{Z}$, the set
$V\cap(\{\x\}\times\mathcal{Z})=\{\x\}\times V_2$ is closed in the subspace $\{\x\}\times\mathcal{Z}$; the natural homeomorphism $\{\x\}\times\mathcal{Z}\cong\mathcal{Z}$ then implies that $V_2$ is closed in $\mathcal{Z}$. Consequently, for compact $D_2\subset\mathcal{Z}$ the intersection $D_2\cap V_2$ is compact.

A main idea of our arguments 
is to remove the null set $V$ and $\partial D_2$, use inner regularity to build a compact $K_1\subset D_2\setminus(V_2\cup\partial D_2)$ on which $f$ is $\mathcal{C}^2$, and  apply our previous arguments on these cores.
\begin{definition}\label{defk2}
For any $\nu_1>0$, there exists a $\mu_2$-measurable compact set $K_1=K_1(\nu_1)$ such that
\[
K_1 \subset D_2 \setminus \big(V_2 \cup \partial D_2\big)
\quad\text{and}\quad
 \mu_2(K_{1})<\mu_2(D_2 \backslash (V_2 \cup \partial D_2)) < \mu_2(K_{1}) + \nu_1.
\]
\end{definition}

Such a compact set $K_{1}$ exists because the Lebesgue measure $\mu_2$ is inner regular and 
$D_2 \setminus (V_2 \cup \partial D_2)$ is $\mu_2$-measurable with positive measure 
(here $\mu_2(V_2)=\mu_2(\partial D_2)=0$, and the boundary of a compact convex set has zero measure; see Lemma~\ref{suplem1}). 
Clearly $f \in \mathcal{C}^2$ on the slice $\{\x\}\times K_{1}$ for $\mu_1$-a.e.\ $\x$.
Since $ \mu_2(K_{1})<\mu_2(D_2 \backslash (V_2 \cup \partial D_2)) < \mu_2(K_{1}) + \nu_1$ and $\mu_2(V_2)=\mu_2(\partial D_2) =0$ we have 
\begin{align}
    \mu_2(K_1) 
    &> \mu_2\!\big(D_2 \setminus (V_2 \cup \partial D_2)\big) - \nu_1 \nonumber \\
    &= \mu_2(D_2) - \mu_2\!\big(D_2 \cap (V_2 \cup \partial D_2)\big) - \nu_1 \nonumber \\
    &= \mu_2(D_2) - \nu_1. \label{genforgeinnerreg}
\end{align}

The next lemma guarantees the existence of non-intersecting open covers for  $K_1, D_2 \cap V_2, \partial D_2 $.
\begin{lemma}\label{coverlemma00}
   Let $\nu_1>0$ and $K_1=K_1(\nu_1)$ be as in Definition~\ref{defk2}. Then there exists $\xi=\xi(\nu_1)>0$  such that the open covers $O_1(\xi), O_2(\xi), O_3(\xi) $ given by
     \[
    O_1(\xi) = \bigcup_{\z \in K_1} \mathcal{B}_{\xi}(\z) \hspace{0.2cm}, \hspace{0.2cm} O_2(\xi) = \bigcup_{\z \in D_2 \cap V_2} \mathcal{B}_{\xi}(\z) , \hspace{0.2cm} O_3(\xi) = \bigcup_{\z \in \partial D_2} \mathcal{B}_{\xi}(\z)\label{eq:cover00} \tag{\textbf{cover}}
\]
      satisfy 
$$ O_1(\xi) \cap O_2(\xi) = \emptyset, \hspace{0.2cm} O_1(\xi) \cap O_3(\xi) = \emptyset, \hspace{0.2cm}  O_3(\xi) \subset D_2 + \mathcal{B}_{\xi}(\mathbf{0}), \hspace{0.2cm} O_1(\xi) \subseteq \text{int}(D_2).$$
    Moreover the measures satisfy \begin{align}
   0 \leq  \mu_2(D_2 \backslash(V_2 \cup \partial D_2)) - \mu_2(O_1(\xi)) = \mu_2(D_2) - \mu_2(O_1(\xi)) < \nu_1 \label{measureineqa1}
\end{align}
and
    $\xi \to 0$ as $\nu_1 \downarrow 0$.
\end{lemma}

The proof of Lemma \ref{coverlemma00} is in Appendix \ref{sec6proof1}.

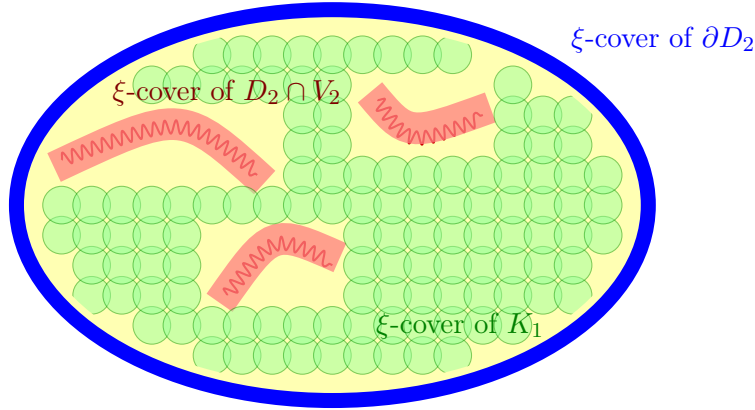
\begin{figure}[h]
\begin{center}
\begin{tikzpicture}[scale=2]

\fill[yellow!30] (0,0) ellipse (2.1 and 1.3); 
\draw[blue, line width=2mm] (0,0) ellipse (2.1 and 1.3); 

\node[blue] at (2.2,1.1) {\( \xi\text{-cover of } \partial D_2  \)};

\draw[red!80!black, thick, decorate, decoration={snake, segment length=4pt}] 
    (-1.8,0.3) .. controls (-1,0.6) .. (-0.5,0.2);
\draw[red!80!black, thick, decorate, decoration={snake, segment length=4pt}] 
    (0.3,0.7) .. controls (0.5,0.4) .. (1,0.6);
\draw[red!80!black, thick, decorate, decoration={snake, segment length=4pt}] 
    (-0.7,-0.6) .. controls (-0.4,-0.2) .. (0,-0.4);

\draw[red!50, line width=12pt, opacity=0.75]
    (-1.85,0.25) .. controls (-1,0.65) .. (-0.45,0.15);
\draw[red!50, line width=12pt, opacity=0.75]
    (0.25,0.75) .. controls (0.5,0.45) .. (1.05,0.65);
\draw[red!50, line width=12pt, opacity=0.75]
    (-0.75,-0.65) .. controls (-0.4,-0.15) .. (0.05,-0.35);

\begin{scope}
\clip (0,0) ellipse (2 and 1.2); 

\foreach \x in {-1.8,-1.6,...,1.8}
  \foreach \y in {-1.0,-0.8,...,1.0}
  {
    
    \pgfmathsetmacro{\r}{sqrt((\x)^2/4 + (\y)^2/1.44)}
    \ifdim \r pt<0.95pt

\pgfmathtruncatemacro{\inTubeA}{(\x > -1.9) && (\x < -0.3) && (\y > 0.1) && (\y < 0.7)}
\pgfmathtruncatemacro{\inTubeB}{(\x > 0.15) && (\x < 1.1) && (\y > 0.25) && (\y < 0.82)}
\pgfmathtruncatemacro{\inTubeC}{(\x > -0.85) && (\x < 0.1) && (\y > -0.7) && (\y < -0.1)}

      \ifnum\inTubeA=0
        \ifnum\inTubeB=0
          \ifnum\inTubeC=0
            
            \filldraw[fill=green!40, draw=green!60!black, opacity=0.5] (\x,\y) circle (0.125);
          \fi
        \fi
      \fi

    \fi
  }
\end{scope}

\node[green!50!black] at (0.85,-0.8) {\( \xi\text{-cover} \text{ of } K_1 \)};
\node[red!50!black] at (-0.7,0.75) {\(\xi\text{-cover of } D_2 \cap V_2 \)};

\end{tikzpicture}
\captionsetup{width=\textwidth}
    \caption{A two dimensional representation of the $\xi$ covers for the sets $ K_1,  D_2 \cap V_2, \partial D_2$. Here, $D_2$ is the closure of an ellipse in $\mathbb{R}^2$ and the set $ D_2 \cap V_2$ is represented by the three disconnected red curves. The sum of volumes in the yellow, red and blue regions is equal to $\nu_1$ and the set $K_1 \subset D_2 \backslash (V_2 \cup \partial D_2) $ is a function of $\nu_1$.}
    \label{figcoveringsets1}
    \end{center}
\end{figure}

\subsection{Lebesgue-volume bounds for $\epsilon$-forging under a.e. smoothness}

Fix $\mu_1$-a.e.\ $\x\in D_1$ and set $\rho_\epsilon:=\sqrt{2\epsilon/L}$ . 
\RSnew{For any $\z^*\in K_1$, define the forging set with respect to $K_1$ by
\begin{align*}
S_{\epsilon}(\x,\z^*,K_1)
:=\Big\{\, \z' \in \bigcup_{\z \in K_1} \mathcal{B}_{\rho_\epsilon}(\z)\ :\ 
\|\nabla_{\x} f(\x;\z')-\nabla_{\x} f(\x;\z^*)\|\le \epsilon\,\Big\}.
\end{align*}}
\begin{theorem}\label{aeforgethm1}
Fix $\nu_1>0$ and let $K_1=K_1(\nu_1)\subset D_2\setminus (V_2\cup\partial D_2)$ be as in Definition~\ref{defk2}, and let $\xi=\xi(\nu_1)>0$ be as in Lemma~\ref{coverlemma00}. 
\RSnew{Assume \textbf{A1--A2} with $V\neq\emptyset$, that $D_2\subset \mathbb{R}^d$ is compact and convex. 
    Further, for $\mu_1$-a.e.\ $\x\in D_1$, assume \textbf{A3} holds for the compact (not necessarily convex) set $\x \times K_1 $ with parameter $r \leq d-1$.
    Then for every
    \(
        0<\epsilon\leq \min\left\{\frac{\gamma^2}{4L },\frac{L\tau^2}{2} , \tfrac{L}{2}\,\xi^2\right\},
    \)
    and for $\mu_1$-a.e.\ $\x\in D_1$,
    we have
    \begin{align}
        \mu_2\big(S_\epsilon(\x,\z^*, K_1) \cap K_1 \big)
        \leq
        C_{D_2,\gamma,d,L}\,
        \epsilon^{\frac{d-r}{2}},
    \end{align} 
   where one may take
    \(  C_{D_2,\gamma,d,L}
        :=
         16^d \,
        \frac{\Gamma(\frac{d}{2}+1)}{\pi^{d/2}}
        \,\mathrm{vol}_{\mathbb{R}^d}(D_2)\,
        \left(\frac{L}{2 \gamma^2}\right)^{d/2}
       \left( \frac{\gamma^2}{4L } \right)^{r/2}.
    \)
    } 
\end{theorem}
The proof of Theorem \ref{aeforgethm1} is in Appendix \ref{sec6proof2}. Omitting the technical details, the proof is completed in \RSnew{four} steps: first, using Lemma \ref{coverlemma00}, we identify a uniform open cover for \RSnew{$K_1$  that is away from $V_2 \cap \partial D_2$. Then we specify a finite sub-cover $\tilde O_1(\epsilon)$ for $S_{\epsilon}(\x,\z^*,K_1) \cap K_1$. Next, in each ball of this sub-cover $\tilde O_1(\epsilon)$} we estimate the volume of a local $\epsilon$ forging set using Lemma~\ref{seconvarlem2}, and in the last step we use a union bound to estimate the total volume of $\epsilon$ forging in $K_1$. \RSnew{We also note that in Theorem \ref{aeforgethm1}, $r$ is a function (that is not necessarily continuous) of $K_1(\nu_1)$ so $r:= r(K_1) = r(\nu_1)$. Thus changing $\nu_1$ can change $r$. Fortunately, since $r \le d-1$ uniformly from \textbf{A3} for any $K_1$, taking $\nu_1 \to 0$  still guarantees  $ \limsup_{\nu_1 \to 0} r(\nu_1) \le d-1$.  }

\begin{remark}[On the $\nu_1$–dependence of $\epsilon$.]
Compared to Theorem~\ref{seconvarthm1}, Theorem~\ref{aeforgethm1} is more restrictive in that
$\epsilon$ cannot be chosen arbitrarily. It must satisfy
\[
\epsilon \;<\; \epsilon_{\max}(\nu_1)
\quad\text{with}\quad
\epsilon_{\max}(\nu_1)\ :=\ \min\Big\{\RSnew{\tfrac{\gamma^2}{4L}},\tfrac{L\tau^2}{2} \, ,\tfrac{L}{2}\,\xi(\nu_1)^2\Big\},
\]
where $\xi(\nu_1)>0$ is the separation radius from Lemma~\ref{coverlemma00} ensuring that all
$\rho_\epsilon$-balls remain inside $\mathrm{int}(D_2)$ and away from $V_2$.
This dependence is a direct consequence of assuming only a.e.\ joint $\mathcal{C}^2$-smoothness:
as $K_1=K_1(\nu_1)$ approaches $D_2\setminus(V_2\cup\partial D_2)$ (inner regularity), its distance
to $V_2\cup\partial D_2$ may shrink, forcing $\rho_\epsilon=\sqrt{2\epsilon/L}$ to shrink accordingly.

By Lemma~\ref{coverlemma00}, one can choose $K_1(\nu_1)$ so that $\xi(\nu_1)$ is nonincreasing and
$\xi(\nu_1)\downarrow 0$ as $\nu_1\downarrow 0$; consequently,
$\epsilon_{\max}(\nu_1)$ is nonincreasing and right-continuous at $\nu_1=0$.
The \emph{rate} at which $\epsilon_{\max}(\nu_1)\downarrow 0$ depends on the geometry of
$K_1(\nu_1)$ near $V_2\cup\partial D_2$ and cannot be specified in general. For simple models
(e.g., squared loss with two-layer networks and leaky \texttt{ReLU}), one can characterize
$K_1(\nu_1)$ more precisely and obtain concrete decay rates; see Appendix~\ref{sectiongeomk1}.
\end{remark}

\subsection{Anti-concentration for $\epsilon$-forging under a.e. smoothness}

Building on the volume bound of Theorem~\ref{aeforgethm1}, we now derive
probability (anti-concentration) bounds for the $\epsilon$-forging set
\[
A_\epsilon(\x,\z^*)\ :=\ \big\{\z\in \mathbb{R}^d\setminus V_2:\ 
\|\nabla_{\x} f(\x;\z)-\nabla_{\x} f(\x;\z^*)\|\le \epsilon\big\},
\quad \text{for $\mu_1$-a.e.\ $\x\in D_1$.}
\]
Assuming \textbf{P1--P2} (log-Lipschitz density on $D_2$ and subexponential tails),
we convert Lebesgue-volume bounds on $S_{\epsilon}(\x,\z^*,K_1)$ into bounds on
$\probP(A_\epsilon(\x,\z^*))$ by (i) controlling the density oscillation on $D_2$
via $e^{L_g\,\mathrm{diam}(D_2)}$ and (ii) bounding the mass outside $D_2$ using the
tail $Ce^{-t_0^\omega}$. As in Theorem~\ref{aeforgethm1}, $\epsilon$ must satisfy
$\epsilon \le \epsilon_{\max}(\nu_1)$ with $\epsilon_{\max}(\nu_1)\downarrow 0$ as
$\nu_1\downarrow 0$, and we pass to the limit by taking $\nu_1\to 0$.

\begin{theorem}[Anti-concentration under a.e.\ smoothness]\label{aeforgethm2}
{Under the setting of Definition~\ref{defk2} and Lemma~\ref{coverlemma00}, let $\nu_1>0$ and that $K_1=K_1(\nu_1)\subset D_2\setminus(V_2\cup\partial D_2)$.}
 \RSnew{Assume \textbf{A1--A2} with $V\neq\emptyset$. 
    Further, for $\mu_1$-a.e.\ $\x\in D_1$, assume \textbf{A3} holds for the compact (not necessarily convex) set $\x \times K_1 $ with parameter $r \leq d-1$. Assume \textbf{P1--P2} and let $L_g$ denote the local Lipschitz constant of $g$ on the compact, convex set $D_2$. 
    Then for every
    \(
        0<\epsilon\leq \min\left\{\frac{\gamma^2}{4L },\frac{L\tau^2}{2} , \tfrac{L}{2}\,\xi^2\right\},
    \)
    and for $\mu_1$-a.e.\ $\x\in D_1$,
    we have
    \begin{align}
            \probP\bigg( \{\z \in \mathbb{R}^d \backslash V_2 : \norm{\nabla f(\x; \z) - \nabla f(\x; \z^*)} \leq \epsilon\}\bigg)
        &\leq
        C'_{D_2,\gamma,d,L}\,
        \epsilon^{\frac{d-r}{2}} +  \RSnew{2}\frac{e^{L_g \text{diam}(D_2)}}{\text{vol}_{\mathbb{R}^d}(D_2) } \nu_1 + C  e^{-t_0^{\omega}} \nonumber \\ & \hspace{3cm}   \mu_1 \text{ a.e. on } D_1 ,
    \end{align} 
    \RSnew{where one may take
    \(  C'_{D_2,\gamma,d,L}
        :=
         16^d \,
        \frac{\Gamma(\frac{d}{2}+1)}{\pi^{d/2}}
        \, e^{L_g \mathrm{diam}(D_2)} \,
        \left(\frac{L}{2 \gamma^2}\right)^{d/2}
       \left( \frac{\gamma^2}{4L } \right)^{r/2}.
    \)}
    }
\end{theorem}
The proof of Theorem \ref{aeforgethm2} is in  \Cref{prob-sec6app}.
Unlike Theorem~\ref{aeforgethm1}, the probability bound in Theorem~\ref{aeforgethm2} carries an
explicit $\nu_1$ term. Absent additional structure on $V_2$, there is no general
rate relating $\epsilon$ and $\nu_1$.
\begin{remark}[Toward non-smooth losses]\label{aesmoothrem2}
Throughout Section~\ref{sectionaeforge} the a.e.\ analysis relies on the existence of gradients 
$\nabla_{\x} f(\cdot;\cdot)$ and mixed derivatives $\nabla_{\z}\nabla_{\x} f(\cdot;\cdot)$ on a 
large-measure compact core $K_1\subset D_2$. A more general framework for genuinely non-smooth $f$ 
would replace gradients by generalized (Clarke) subgradients and study the forging set
\[
S_{\epsilon}(\x,\z^*)
:= \Big\{\z \in D_2:\ \inf_{\substack{\v \in \partial f(\x;\z)\\ \v^* \in \partial f(\x;\z^*)}}
\|\v-\v^*\|\le \epsilon\Big\}.
\]
Pursuing this requires tools beyond Lemmas~\ref{seconvarlem1}–\ref{seconvarlem2} to obtain workable “second-variation” surrogates. 
We leave this non-smooth extension to future work.
\end{remark}

\section{Conclusions and Future Work}\label{conclusion}

We presented geometric and probabilistic bounds on the volume of $\epsilon$-forging sets. We first considered linear regression and simple neural networks, then obtained results both under global $\mathcal{C}^2$ smoothness and under almost-everywhere smoothness. We also provided batch-SGD variants and dimension-regime comparisons. We believe this work opens several avenues for interesting future work. 

For example, our analysis was aimed at the case of \emph{one-step} forging. It considered when a single replacement yields an $\epsilon$-close update. A natural extension is \emph{multi-step forging}, where a more sophisticated adversary may (benignly) perturb now and (adversarially) repair later to return to the original trajectory. Formalizing and analyzing  such multi-step forging attacks is an avenue we leave open to future work.

Another interesting direction of future work is to extend our Lebesgue measure and probability bounds to smooth embedded data manifolds. Yet another is to handle more general function classes such as weakly convex functions and Clarke regular functions (see Section~\ref{sectionaeforge}). 

Additionally, there appears to be a connection to differential privacy (DP) \cite{dwork2014algorithmic} that is under-explored. Our bounds characterize typical single-point sensitivity (``what is the measure of points that would have produced nearly the same update?'') and the fact that forging sets are of low measure arguably shows that this sensitivity is generally high. This, in turn, suggests a tension with DP's mandate to suppress individual influence \cite{dwork2014algorithmic,chien2024langevin,sekhari2021remember}. It would be interesting to rigorously explore whether this tension is due to an inherent tradeoff between privacy and robustness to forging.


\appendix
\renewcommand{\thesection}{\Alph{section}}
\renewcommand{\theHsection}{appendix.\Alph{section}}

\section{Proof of \Cref{thm:forging_incentive}}\label{pf-motiv}

Before we prove the theorem, we first restate Theorem 2.1.12 from \cite{nesterov2013introductory}, as we will refer to it later. We also present two lemmas that study the only sources of deviation that may arise in the gradient updates associated with an alternative parameter trajectory. Either the same loss function is applied to two different initializations as would happen in the iterations following a data point being replaced, or different loss functions are used, as would happen when a data point is replaced. The induced distance between the resulting model parameters can then be bounded by combining the bounds on these deviations and applying them inductively across the full sequence of parameter updates.
\vspace{1em}
\begin{theorem}[\cite{nesterov2013introductory}]\label{ext-co-coerc}
    If $f : \mathbb{R}^d \to \mathbb{R}$ is $\mu$-strongly convex and $L$-smooth in an open set $O\subset \mathbb{R}^d$, then for all $\bm x, \bm y\in O$,
    \[\langle \nabla f(\bm x) - \nabla f(\bm y),\bm x - \bm y \rangle \geq \frac{\mu L}{L + \mu}\|\bm x - \bm y\|^2 + \frac{1}{L + \mu}\|\nabla f(\bm x) - \nabla f(\bm y)\|^2. \]
\end{theorem}
Suppose one replaces the initial parameter vector $\bm w_0$ by an alternative $ \widetilde{\bm w}_0$ that is at most $\epsilon$ away. The next lemma shows that if the original function is smooth and strongly convex within an $\epsilon$-tube of the original trajectory, then the resulting alternate trajectory remains within $\epsilon$ of the original.
\vspace{1em}
\begin{lemma}\label{lemma2}
    Suppose a $N$-step parameter trajectory $(\bm w_0, \bm w_1,...,\bm w_{N})$ initialized with $\bm w_0$ is generated by 
    \[\bm w_k = \bm w_{k-1} - h_{k-1}\nabla f_{k-1}(\bm w_{k-1})\]
    for $1\leq k\leq N$, where $h_{k-1}$ is the learning rate and $f_{k-1}$ is the loss function at each step. Let $\widetilde{\bm w}_0$ be an alternative initialization with $\|\widetilde{\bm w}_0 - \bm w_0\|\leq\epsilon$ for some $\epsilon>0$, and $T^{\mathrm{cont}}_\epsilon$ be the $\epsilon$-tube formed by $\bm w_0,...,\bm w_N$. If $f_{k}$ is $\mu_k$-strongly convex and $L_k$-smooth for all $k$ in $T^{\mathrm{cont}}_\epsilon$, then running the iteration 
        \[\widetilde{\bm{w}}_k = \bm{\widetilde{w}}_{k-1} - h_{k-1}\nabla f_{k-1}(\widetilde{\bm{ w}}_{k-1})\]
     with  $h_k<\frac{1}{L_t}$, leads to $\widetilde{\bm w}_N$ satisfying
    \[\|\widetilde{\bm w}_{N} - \bm w_{N}\|<\prod^{N-1}_{k = 0}\,|1-h_k L_k|\,\|\widetilde{\bm w}_0 - \bm w_0\|\leq\epsilon.\]
    
\end{lemma}
\begin{proof}
    According to the given rule, provided $\|\widetilde{\bm w}_{k} - \bm w_{k}\|\leq \epsilon$ we have
\begin{align*}
    \|\widetilde{\bm w}_{k+1} - \bm w_{k+1}\|^2
    & = \|\widetilde{\bm w}_{k} - h_k\nabla f_{k}(\widetilde{\bm w}_{k}) - \bm w_{k} + h_k\nabla f_{k}(\bm w_{k})\|^2\\
    & = \|\widetilde{\bm w}_{k} - \bm w_{k}\|^2 + h_k^2\|\nabla f_{k}(\widetilde{\bm w}_{k}) - \nabla f_{k}(\bm w_{k})\|^2 - 2h_k\langle \widetilde{\bm w}_{k} - \bm w_{k}, \nabla f_{k}(\widetilde{\bm w}_{k}) - \nabla f_{k}(\bm w_{k}) \rangle \\   
    & \leq \|\widetilde{\bm w}_{k} - \bm w_{k}\|^2 + h_k^2\|\nabla f_{k}(\widetilde{\bm w}_{k}) - \nabla f_{k}(\bm w_{k})\|^2 \notag\\ & \qquad \qquad \qquad \qquad - 2h_k\left(\frac{\mu_k L_k}{L_k + \mu_k}\|\widetilde{\bm w}_{k} - \bm w_{k}\|^2 + \frac{1}{L_k + \mu_k}\|\nabla f_{k}(\widetilde{\bm w}_{k}) - \nabla f_{k}(\bm w_{k})\|^2 \right)\\ 
    & = \left(1-\frac{2h_k\mu_k L_k}{L_k + \mu_k}\right)\|\widetilde{\bm w}_{k} - \bm w_{k}\|^2 + \left(h_k^2 - \frac{2h_k}{L_k + \mu_k}\right)\|\nabla f_{k}(\widetilde{\bm w}_{k}) - \nabla f_{k}(\bm w_{k})\|^2\\
    & \leq \left(1-\frac{2h_k\mu_k L_k}{L_k + \mu_k} + h_k^2L_k^2 - \frac{2h_k L_k^2}{L_k + \mu_k}\right)\|\widetilde{\bm w}_{k} - \bm w_{k}\|^2 \\
    & = \left(1-h_k L_k\right)^2\|\widetilde{\bm w}_{k} - \bm w_{k}\|^2
\end{align*}
where the first inequality is by applying Theorem \ref{ext-co-coerc} with $O=\mathcal{B}_\epsilon(\bm w_k)$, and the second inequality uses $L_k$-smoothness of $f_k$. Hence, the recursive relation for any two consecutive steps is 
\begin{equation}\label{eq:recursive}
    \|\widetilde{\bm w}_{k+1} - \bm w_{k+1}\|\leq\left|1-h_k L_k\right|\|\widetilde{\bm w}_{k} - \bm w_{k}\|.
\end{equation}

Therefore, choosing $h_k<\frac{1}{L_k}$ allows us to apply \eqref{eq:recursive} recursively for $0\leq k\leq N-1$ to obtain
\[\|\widetilde{\bm w}_{N} - \bm w_{N}\|\leq\prod^{N-1}_{k = 0}\,|1-h_k L_k|\,\|\widetilde{\bm w}_{0} - \bm w_{0}\|.\]
Consequently 
$\|\widetilde{\bm w}_{N} - \bm w_{N}\|<\|\widetilde{\bm w}_{0} - \bm w_{0}\|\leq \epsilon$.
\end{proof}

On the other hand, if $\bm{w}_0$ and $\widetilde{\bm{w}}_0$ are updated separately using two different loss functions, their resulting parameters can still remain within an $\epsilon$-neighborhood of each other, provided that the gradient deviation is properly controlled. The precise statement is given below.
\vspace{1em}
\begin{lemma}\label{lemma3}
    Let $\bm w_0$ be an initial point and $\tilde{\bm w}_0$ satisfy $\|\tilde{\bm w}_0 - \bm w_0\|\leq \epsilon$. Let $f_0:\mathbb{R}^d\to\mathbb{R}$ be a loss function that is $L$-smooth and $\mu$-strongly convex in $\mathcal{B}_\epsilon(\bm w_0)$. Let $\tilde{f}_0$ be another loss function. Consider one step of gradient descent which is defined by
    \[\bm w_1 = \bm w_0 - h\nabla f_0(\bm w_0) \quad \text{and}\quad \widetilde{\bm w}_1 = \widetilde{\bm w}_0 - h\nabla \tilde{f}_0(\widetilde{\bm w}_0)\] with the learning rate $h$. 
    If $\nabla^2f_0$ exists and $\|\nabla f_0(\widetilde{\bm w}_0) - \nabla \widetilde{f}_0(\tilde{\bm w}_0)\|\leq \mu\epsilon$, then taking $h\leq \frac{1}{L}$ leads to $\|\widetilde{\bm w}_1 - \bm w_1\|<\|\widetilde{\bm w}_0 - \bm w_0\|\leq \epsilon.$
\end{lemma}
\begin{proof}
    According to gradient descent,
    \begin{align*}
        \|\tilde{\bm w}_1 - \bm w_1\| &= \|\tilde{\bm w}_0 - h\nabla \tilde{f}_0(\tilde{\bm w}_0) - \big(\bm w_0 - h\nabla f_0(\bm w_0)\big)\|\\
        & = \|\tilde{\bm w}_0 - \bm w_0 - h\big( \nabla f_0(\tilde{\bm w}_0) - \nabla f_0(\bm w_0)\big) + h\big(\nabla f_0(\tilde{\bm w}_0) - \nabla \tilde{f}_0(\tilde{\bm w}_0)\big)\|\\
        & \leq \|\tilde{\bm w}_0 - \bm w_0 - h\big( \nabla f_0(\tilde{\bm w}_0) - \nabla f_0(\bm w_0)\big)\| + h\|\nabla f_0(\tilde{\bm w}_0) - \nabla \tilde{f}_0(\tilde{\bm w}_0)\|\\
        & \leq \|I-hA\|\|\tilde{\bm w}_0 - \bm w_0\| + h\,\mu\,\epsilon
    \end{align*}
    where $A = \int^1_0\nabla^2f_0\big(\bm w_0+t(\widetilde{\bm w}_0 - \bm w_0)\big)dt$ by the integral form of the Mean Value Theorem.  
    Strong convexity yields
    $\|I-h\nabla^2f_0(\xi)\|\leq
    1-h\mu$.
    Therefore, 
    $\|\tilde{\bm w}_1 - \bm w_1\|<(1-h\mu)\epsilon+h\mu\epsilon = \epsilon$.
\end{proof}

With these lemmas in hand we can now control the induced distance between the resulting model parameters by applying \RSnew{\Cref{lemma2}} and \RSnew{\Cref{lemma3}}   inductively across the full sequence of parameter updates. We now present the proof of \Cref{thm:forging_incentive}.

\begin{proof}
    In order to analyze the evolution of the alternative trajectory, we partition the updates into $m+1$ slices with boundaries $n_1, n_2,...,n_m$ where each slice starts at $\widetilde{\bm x}_0$ and ends with $\bm x_{n_1-1}$,..., $\bm x_{n_m-1}$ or $\bm x_{N-1}=\bm x_{n_{m+1}-1}$. Then the alternative data trajectory is 
    \[ (\widetilde{\bm x}_0, \bm x_1,...,\bm x_{n_1-1}\,|\,\widetilde{\bm x}_0, \bm x_{n_1+1}, ..., \bm x_{n_m-1}\,|\,\widetilde{\bm x}_0, \bm x_{n_m+1},...,\bm x_{N-1})\]
    with  $0<n_1<n_2<\dots<n_m<N$. The corresponding parameter updates form the trajectory
    \[ (\bm w_0, \widetilde{\bm w}_1,...,\widetilde{\bm w}_{n_1-1},\widetilde{\bm w}_{n_1}, \widetilde{\bm w}_{n_1+1}, ..., \widetilde{\bm w}_{n_m-1},\widetilde{\bm w}_{n_m}, \widetilde{\bm w}_{n_m+1},...,\widetilde{\bm w}_{N-1},\widetilde{\bm w}_N).\]

    We  analyze  $\|\widetilde{\bm w}_{N} - \bm w_{N}\|$ by aggregating the effects of each modified slice.
    For the first slice, we have 
    \[\|\widetilde{\bm w}_{1} - \bm w_{1}\| = \|\bm w_{0} - h_{0}\nabla \widetilde{f}_{0}(\bm w_{0}) - \bm w_{0} + h_{0}\nabla f_{0}(\bm w_{0})\|= h_{0}\,\|\nabla f_{0}(\bm w_{0})- \nabla \widetilde{f}_{0}(\bm w_{0})\| \leq h_{_0}\,\delta_{0}\]
    
    If $h_0\leq 1$, then according to Lemma \ref{lemma2} and by choosing $h_k\leq\frac{1}{L_k}$ for $1\leq k\leq n_1-1$, we get 
    \[
    \|\widetilde{\bm w}_{n_1} - \bm w_{n_1}\|<\prod^{n_1-1}_{k = 1}\,|1-h_k L_k|\,\|\widetilde{\bm w}_1 - \bm w_1\|<\|\widetilde{\bm w}_1 - \bm w_1\|\leq h_0\delta_0\leq\delta_0.
    \]
    
    We proceed by induction. Assume $\|\widetilde{\bm w}_{n_{j-1}} - \bm w_{n_{j-1}}\|<\delta_{0}$ for $j\geq2$.
    The assumption \eqref{lip-in-data} implies 
    \[\|\nabla f_0(\widetilde{\bm w}_{n_{j-1}}) - \nabla \widetilde{f}_0(\widetilde{\bm w}_{n_{j-1}})\|\leq \mu_0\,\|\widetilde{\bm w}_{n_{j-1}} - \bm w_{n_{j-1}}\|.\] Using Lemma \ref{lemma3}, by requiring $h_{n_{j-1}}\leq\frac{1}{L_0}$, we have
    $\|\widetilde{\bm w}_{n_{j-1}+1} - \bm w_{n_{j-1}+1}\|<\delta_{0}$.
    Applying Lemma $\ref{lemma2}$ again, we conclude that for $n_j\in\{n_2,...,n_m,n_{m+1}\}$
    \[\|\widetilde{\bm w}_{n_j} - \bm w_{n_j}\|<\delta_{0}\]
    if $h_k\leq\frac{1}{L_k}$ for $n_{j-1}+1\leq k\leq n_j-1$,
     where we recall that $N=n_{m+1}$.
\end{proof}

\section{Proofs for \Cref{case}}\label{pf-sec3}
In this section, we present detailed proofs for the Lebesgue measure estimates of $\epsilon$-forging set as discussed in \Cref{case}. We start with linear regression (\Cref{prop:lr-eps-forging}).

\begin{proof}
Fix $(\bm x,y)$ and $\epsilon>0$. 
     The forging set can be explicitly written as $S_{\epsilon} = \{(\bm z,t): \|(\bm x^T\bm w-y)\bm x - (\bm z^T\bm w-t)\bm z\| \leq \epsilon \}$. Denote 
    $\bm a \coloneqq (\bm x^T\bm w-y)\bm x$ with $A =\|\bm a\|$,
    and define $s(\bm z,t) \coloneqq \bm z^T \bm w - t$. The condition in the forging set becomes a norm inequality
    \begin{equation}\label{ineq:epslr}\|\bm a - s(\bm z,t)\bm z\| \leq \epsilon.\end{equation}
    We then evaluate the measure of the set of solutions to \Cref{ineq:epslr} restricted to \RSnew{$\mathcal{B}_R\times[-R,R]$}.
    We do this by first fixing $\bm z$ and finding the measure associated to $t$. Then we integrate the measure with respect to $\bm z$ in $\mathbb{R}^d$. Since any solution with $\bm z = \bm 0$ is a low dimensional embedding in $\mathbb{R}^d\times\mathbb{R}$ which is of measure zero, it suffices to consider the case for nonzero $\bm z$. For any nonzero $\bm z$,  \eqref{ineq:epslr} implies
    \begin{equation}\label{lrqeq}\|\bm z\|^2\,s(\bm z;t)^2  - 2\,(\bm a^T\bm z)\, s(\bm z;t) + (A^2-\epsilon^2) \leq 0,\end{equation}
    which is a quadratic equation with respect to $s(\bm z;t) = \bm z^T\bm w - t$. We next calculate the measure for the set of feasible $s(\bm z;t)$ as it is the same as that for $t$ by the invariance of the Lebesgue measure to shifting. Requiring the discriminant to be nonnegative imposes the condition   
    \begin{equation}\label{cond:lr}A\,|\sin\theta|\leq \epsilon .\end{equation}
    where $\theta$ is the angle between $\bm a$ and $\bm z$. Explicitly, it implies that $\theta$ is restricted to
\begin{equation}\label{anglecond}
\theta \in \left[-\theta_0, \theta_0\right], \quad \text{with} \quad \theta_0 = \arcsin\left(\min\{1,\;\frac{\epsilon}{A}\}\right).
\end{equation}
    Under the condition \eqref{cond:lr}, we solve \eqref{lrqeq} and obtain the Lebesgue measure of the set of feasible $s(\bm z;t)$, hence the corresponding forging labels $t$, as
    \[
    L(\bm z) = \frac{2\sqrt{\epsilon^2-A^2\sin^2\theta}}{\|\bm z\|}.
    \]
    Next, we integrate with respect to $\bm z$ in $\mathbb{R}^d$ under the condition \eqref{cond:lr}. Without loss of generality, assume that the data are normalized and restrict $\bm z$ to the unit ball $\mathcal{B}_1\subset\mathbb{R}^d$. Using spherical coordinates for $\bm z$, write
$\bm z = r\, \bm u$, \ $r = \|\bm z\| \in [0,1]$, and $ \bm u \in S^{d-1}$,
with the volume element
$
d\bm z = r^{d-1}\, dr\, d\Omega(\bm u)
$
where
\begin{equation}\label{sphereele}
d\Omega(\bm u) = \frac{2\pi^{\frac{d-1}{2}}}{\Gamma\left(\frac{d-1}{2}\right)} (\sin \theta)^{d-2}\, d\theta
\end{equation}
is the surface element on the unit sphere $S^{d-1}$ \cite{blumenson1960derivation}.
The volume can then be evaluated as 
\begin{align*}
\RSnew{\mu(S_{\epsilon}\cap(\mathcal{B}_1\times[-R,R])}) & \leq \RSnew{\int_{\bm z \in \mathcal{B}_1} \mathbf{1}_{\{A |\sin \theta| \leq \epsilon\}}\, \min\{L(\bm z),2R\} \,d\bm z}\\
&\leq \int_{\bm z \in \mathcal{B}_1} \mathbf{1}_{\{A |\sin \theta| \leq \epsilon\}}\, L(\bm z) \,d\bm z\\
& = \int_{r=0}^{1}\int_{\bm u \in S^{d-1}} \mathbf{1}_{\{A|\sin \theta|\leq\epsilon\}}\, \frac{2\sqrt{\epsilon^2-A^2\sin^2\theta}}{r}\, r^{d-1}\, d\Omega(\bm u)\, dr \\
& \leq 2 \int_{r=0}^{1} r^{d-2}\, dr \int_{\{\bm u \in S^{d-1} : A|\sin \theta| \leq \epsilon\}} \epsilon\, d\Omega(\bm u),\quad\text{by $\sqrt{\epsilon^2-A^2\sin^2\theta} \leq \epsilon$}\\
& = \frac{2}{d-1} \Big(\int_{\{\bm u \in S^{d-1} : A|\sin \theta|\leq \epsilon\}}  d\Omega(\bm u)\Big)\, \epsilon.
\end{align*}
Using \eqref{anglecond}, \eqref{sphereele} and the symmetry of the angular domain,
\begin{align}\label{lrvol}
\RSnew{\mu(S_{\epsilon}\cap(\mathcal{B}_1\times[-R,R])}) \leq \frac{4}{d-1}\, \frac{2\pi^{\frac{d-1}{2}}}{\Gamma\left(\frac{d-1}{2}\right)}\left( \int_0^{\theta_0} (\sin \theta)^{d-2}\, d\theta \right) \epsilon
\end{align}

By \eqref{anglecond},
a bound could be obtained by taking 
$\theta_0 = \arcsin(1) = \frac{\pi}{2}$,
and substituting
\[
\int^{\pi/2}_0\;(\sin\theta)^{d-2}\;d\theta\; = \; \frac{\sqrt{\pi}\,\Gamma\Bigl(\frac{d-1}{2}\Bigr)}{2\,\Gamma\Bigl(\frac{d}{2}\Bigr)}
\]
in \eqref{lrvol}. This yields
\begin{equation}\label{lrresult1}
    \RSnew{\mu(S_{\epsilon}\cap(\mathcal{B}_1\times[-R,R])})\,\leq\,\frac{4\,\pi^{\frac{d}{2}}}{(d-1)\Gamma\Bigl(\frac{d}{2}\Bigr)}\epsilon.
\end{equation}

Now, consider the case where the angle allowed is restricted to\begin{equation}\label{tighter-cond}\theta_0 = \arcsin(\frac{\epsilon}{A})\leq c \epsilon\end{equation}
for some $c$ such that $ \frac{\pi}{2A}>c > \frac{1}{A}$.
Then, 
\begin{align*}
\int_{0}^{\theta_0} (\sin\theta)^{d-2} d\theta\,& =\,\int^{\arcsin{(\epsilon/A)}}_0\,(\sin\theta)^{d-2}\,d\theta\\
&\leq \,\int^{c\epsilon}_0\,\theta^{d-2}\,d\theta,\quad\text{since \eqref{tighter-cond} and}\; \sin\theta \leq \theta \;\text{ for } \;\theta \geq 0\\
& = \frac{1}{d-1}\,(c\epsilon)^{d-1}.
\end{align*}
Substituting the result to  \eqref{lrvol}, we get a tighter bound in this case 
\begin{equation}\label{lrresult2}
    \RSnew{\mu(S_{\epsilon}\cap\RSnew{(\mathcal{B}_1\times[-R,R])})\,\leq\,\frac{4c^{d-1}}{(d-1)^2}\, \frac{2\pi^{\frac{d-1}{2}}}{\Gamma\left(\frac{d-1}{2}\right)}\;\epsilon^d}.
\end{equation}
To generalize the volume result for the dataset $D$ that is bounded by an open ball with radius $R$, rescale the variables so that $\tilde{ \bm {z}} = \frac{\bm z}{R}$. This leads to
\[
r = R\,\tilde{r}\,,\; dr = R\,d\tilde{r}\quad\text{so that}\quad d\bm z = (R\,\tilde{r})^{d-1}\,R\,\,d\tilde{r}\,\,\RSnew{d}\Omega(\bm u) = R^{d}\,\tilde{r}^{d-1}d\tilde{r}\,\,\RSnew{d}\Omega(\bm u).
\]
The bound becomes
\begin{align*}
\mu(S_{\epsilon}\cap\RSnew{(\mathcal{B}_R\times[-R,R])})
& \leq \int_{\tilde{r}=0}^{1}\int_{\bm u \in S^{d-1}} \mathbf{1}_{\{A|\sin \theta|<\epsilon\}}\, \frac{2\sqrt{\epsilon^2-A^2\sin^2\theta}}{R\tilde{r}}\, R^{d}\,\tilde{r}^{d-1}d\tilde{r}\,\,\Omega(\bm u) \\
& \leq R^{d-1}\,2\int_{\tilde{r}=0}^{1} \tilde{r}^{d-2}\, d\tilde{r} \int_{\{\bm u \in S^{d-1} : A|\sin \theta| < \epsilon\}} \epsilon\,\, d\Omega(\bm u)
\end{align*}
\RSnew{which is exactly $R^{d-1}$ times the angular–radial integral as in the unit-radius case $\bm z\in\mathcal B_1$ (see the derivation leading to \eqref{lrvol})}.
Combining the result from \eqref{lrresult1}, the coarse volume estimate is
\begin{equation*}
    \mu(S_{\epsilon}\cap\RSnew{\mathcal{C}_R})\,\leq\, \frac{4\,\pi^{\frac{d}{2}}R^{d-1}}{(d-1)\Gamma\Bigl(\frac{d}{2}\Bigr)}\,\epsilon.
\end{equation*}
Using the standard formula \RSnew{$\mathrm{vol}_{\mathbb{R}^{d+1}}(\mathcal{C}_R) = \frac{2\pi^{d/2}R^{d+1}}{\Gamma(d/2+1)}$\cite{blumenson1960derivation}, and the fact that $\Gamma(\frac{d}{2}+1) = \frac{d}{2}\Gamma(\frac{d}{2})$, the result can be rewritten as \eqref{lr-eps-1}.}
If additionally $\frac{\epsilon}{A}<\sin(c\epsilon)$ where $A = \|\nabla_{\bm w} f\left(\bm w; (\bm x,y)\right)\|$, for some $c\in[\frac{1}{A},\frac{\pi}{2A}]$, according to \eqref{lrresult2}, we can similarly derive 
\begin{equation}\label{eq:lr-eps2-temp}
\RSnew{\frac{\mu(S_{\epsilon}\cap\mathcal{C}_R)}{\mathrm{vol}_{\mathbb{R}^{d+1}}(\mathcal{C}_R)}}\,\leq\,\frac{2c^{d-1}\pi^{-1/2}d}{(d-1)^2}\, \frac{\Gamma(\frac{d}{2})}{\Gamma\left(\frac{d-1}{2}\right)}\;\frac{\epsilon^d}{R^2}.
\end{equation}
\RSnew{Using $x=\frac{d-1}{2}$ and $s = 1/2$ in Gautschi's inequality 
$\frac{\Gamma(x+1)}{\Gamma(x+s)}<(x+1)^{1-s}$,we get 
\[\frac{\Gamma(d/2)}{\Gamma\!\left(\tfrac{d-1}{2}\right)} 
= \frac{\Gamma\!\left(\tfrac{d-1}{2} + \tfrac{1}{2}\right)}{\Gamma\!\left(\tfrac{d-1}{2}\right)} 
< \left(\frac{d-1}{2} + 1\right)^{1/2} 
= \sqrt{\frac{d+1}{2}}.\]
Substituting into \eqref{eq:lr-eps2-temp} yields \eqref{lr-eps-2}}.
\end{proof}

\begin{remark} For completeness, we also provide a calculation when $d=1$. \Cref{ineq:epslr} now becomes $|a-wz^2+tz|\leq\epsilon$. For a fixed $z\neq 0$, this is equivalent to $t\in\left[\frac{wz^2-a-\epsilon}{z}, \frac{wz^2-a+\epsilon}{z}\right]$. So the feasible interval length $L(z) \leq\min\{ \frac{2\epsilon}{|z|},2\sqrt{R^2-z^2}\}$, since the forging set is restricted to $\mathcal{B}_R\RSnew{\subset\mathbb{R}^2}$ and $|t|\leq\sqrt{R^2-z^2}$. As the cut $z=0$ contributes zero measure,
\[\mu(S_\epsilon\cap\mathcal{B}_R)\leq\int^R_{-R}\min\{2\sqrt{R^2-z^2}, \frac{2\epsilon}{|z|}\}dz.\]
Note that near $z=0$, $\frac{2\epsilon}{|z|}$ blows up and $2\sqrt{R^2-z^2}=\frac{2\epsilon}{|z|}$ when $z$ satisfies $\epsilon^2 = z^2(R^2-z^2)$. If $\epsilon$ is small, then taking $c = \min\{R,\frac{\epsilon}{R}\}$ and by the symmetry, we evaluate 
\begin{align*}
    \mu(S_\epsilon\cap\mathcal{B}_R) &\leq 4\left(\int^c_0\sqrt{R^2-z^2}dz + \epsilon\int^R_c\frac{1}{z}dz\right)\\
    & = 2c\sqrt{R^2 - c^2} + 2R^2\arcsin{(\frac{c}{R})} + 4\epsilon\ln{(\frac{R}{c})}.
\end{align*}
\end{remark}
Next, we prove \Cref{prop: nn-eps-forging}, which follows a similar strategy as in the linear regression case.
\begin{proof}
We begin with the observation that 
  \(      S_{\epsilon} \subset S_\epsilon^{\bm{W}}\cap S_\epsilon^{\bm v}\)
    where
    \begin{align*}
    S_\epsilon^{\bm{W}} & = \{(\bm z,t):\|\nabla_{\bm W}f\left(\bm W,\bm v; (\bm x,y)\right) - \nabla_{\bm W}f\left(\bm W,\bm v; (\bm z,t)\right)\|_F\leq\epsilon\}\\
    S_\epsilon^{\bm v} & = \{(\bm z,t): \|\nabla_{\bm v}f\left(\bm W,\bm v; (\bm x,y)\right) - \nabla_{\bm v}f\left(\bm W,\bm v; (\bm z,t)\right)\|\leq\epsilon \}.
    \end{align*}
Thus,
$\mu(S_\epsilon\cap\RSnew{\mathcal{C}_R})\,\leq\,\min\{\mu(S_\epsilon^{\bm{W}}\cap\RSnew{\mathcal{C}_R}),\;\mu( S_\epsilon^{\bm v}\cap\RSnew{\mathcal{C}_R})\}\,\leq\mu(S_\epsilon^{\bm{W}}\cap\RSnew{\mathcal{C}_R})$.
So it suffices to evaluate $\mu(S_\epsilon^{\bm{W}})$.  To that end, fix $\epsilon >0$ and $(\bm x, y)\in D$. For $(\bm{z},t)\in S_\epsilon^{\bm{W}}$, 
\begin{align}
    \|(\bm v^T\rho(\bm W\bm x)-y)[\bm v\odot\rho'(\bm W\bm x)]\bm x^T - (\bm v^T\rho(\bm W\bm z)-t)[\bm v\odot\rho'(\bm W\bm z)]\bm z^T\|_F\leq\epsilon.\label{cond:1}
\end{align}
Note that  \( \rho \) is non-differentiable at zero, and its subgradient \( \rho'(0) \) can take any value in \( [0,1] \). In this proof, as is standard in practice---especially with gradient descent algorithms---we adopt the choice \( \rho'(0) = 0 \). So that
\begin{align*}\rho(\bm Wx)_i & = \rho(\bm w_i^T\bm x) = \begin{cases}
    \bm w_i^T\bm x \quad &\text{if }\;\; \bm w_i^T\bm x>0\\
    0\quad &\text{if }\;\; \bm w_i^T\bm x \leq 0
\end{cases}\\
\text{and}\quad\rho'(\bm W\bm x)_i & = \rho'(\bm w_i^T\bm x) = \begin{cases}
    1 \quad \quad&\text{ if }\;\; \bm w_i^T\bm x>0\\
    0\quad \quad&\text{ if }\;\; \bm w_i^T\bm x \leq 0.
\end{cases}
\end{align*}
Thus, we can define a diagonal matrix $\bm D_{\bm x}$ with diagonal entries
\[(\bm D_{\bm x})_{ii} = \begin{cases}
    1\quad &\text{if }\;\; \bm w_i^T\bm x>0\\
    0\quad &\text{if }\;\; \bm w_i^T\bm x\leq 0
\end{cases}\]
and rewrite 
$\rho(\bm W\bm x) = \bm D_{\bm x}\bm W\bm x$ and $\bm v\odot\rho'(\bm W\bm x) = \bm D_{\bm x}\bm v$.
Intuitively, the diagonal matrix $\bm D$ acts as a selection of activated neurons.
Since $\bm W$ and $\bm v$ are fixed, $\bm D_{\bm x}$ is dependent on \(\bm x\), and with slight abuse of notation we indicate this dependence in the subscript. Extending the same notation to $\bm D_{\bm z}$, we can rewrite the necessary condition \eqref{cond:1} as
\begin{align}
\|(\bm v^T\bm D_{\bm x}\bm W\bm x-y)(\bm D_{\bm x}\bm v \, \bm x^T) - (\bm v^T\bm D_{\bm z}\bm W\bm z-t)(\bm D_{\bm z}\bm v \, \bm z^T)\|_F\leq\epsilon\label{rcond:1}
\end{align}
In turn, a necessary condition for \eqref{cond:1} to hold is that all rows $i\in[n]$ must satisfy 
\begin{align}\label{eq:necc_cond_nn} \|(\bm v^T\bm D_{\bm x}\bm W\bm x-y)(\bm D_{\bm x}\bm v)_i \, \bm x - (\bm v^T\bm D_{\bm z}\bm W\bm z-t)(\bm D_{\bm z}\bm v)_i \, \bm z\|\leq\epsilon.\end{align}
Denoting the set of all $(\bm{z},t)$ satisfying \eqref{eq:necc_cond_nn} for a given index $i$ by $S_i$, it follows that 
$S_\epsilon^{\bm{W}}\cap\RSnew{\mathcal{C}_R}\subset\left(\bigcap^n_{i = 1}S_i\right)\cap\RSnew{\mathcal{C}_R}\subset S_i\cap\RSnew{\mathcal{C}_R}$
for all $i$, which implies 
\begin{equation}\label{Si}
    \mu(S_\epsilon^{\bm{W}}\cap\RSnew{\mathcal{C}_R})\leq \min_i\{\mu(S_i\cap\RSnew{\mathcal{C}_R})\}.
\end{equation} Next, we focus on 
 estimating $\mu(S_i\cap\RSnew{\mathcal{C}_R})$.
Note that each $\bm D_{\bm z}$ represents a result of sign pattern of $\{\bm w_i^T\bm z\}_{i=1}^n$, and there are at most $\sum^d_{k=0}\binom{n}{k}$ different possibilities. These correspond to the maximal number of orthants in $\mathbb{R}^n$ intersected by a $d$ dimensional hyperlane \cite{matousek2013lectures}. We will first bound the measure of $S_i$ associated with a fixed $\bm D_{\bm z}$, then take a union bound over all possibilities.

\textbf{Step 1.} To derive $\mu(S_i\cap\RSnew{\mathcal{C}_R})$ under a fixed sign pattern, we begin by defining 
\[\bm a_i\coloneqq(\bm v^T\bm D_{\bm x}\bm W\bm x-y)(\bm D_{\bm x}\bm v)_i \, \bm x, \quad \widetilde{\bm W}\coloneqq\bm D_{\bm z}\bm W, \quad \text{and }\quad \widetilde{\bm v}\coloneqq\bm D_{\bm z}\bm v.\]
Thus, \Cref{eq:necc_cond_nn} becomes
\begin{equation}\label{ineqcond}
     \|\bm a_i  - (\bm v^T\widetilde{\bm W}\bm z-t)\widetilde{\bm v}_i \, \bm z\|\leq\epsilon.
\end{equation}

Define $K=\{i\in[n] \ | \ \widetilde{\bm v}_i \neq 0\}$. For $i\in K$, dividing both sides by $v_i$, the inequality \eqref{ineqcond} becomes
\[ \left\|\frac{\bm a_i}{|v_i|}  - (\bm v^T\widetilde{\bm W}\bm z-t) \, \bm z\right\|\leq\frac{\epsilon}{|v_i|}.\]
This is essentially in the same format of the constraint derived in \eqref{ineq:epslr} of \Cref{prop:lr-eps-forging} for linear regression with $s(\bm z,t) = \bm v^T\widetilde{\bm W}\bm z-t$. Thus, we proceed with the same calculations as in \Cref{prop:lr-eps-forging} and conclude that for a chosen $\epsilon>0$, a necessary condition on $\bm{z}$ is $\|\bm a_i\| |\sin\theta|\leq \epsilon$, where $\theta$ as the angle between $\bm x$ and $\bm z$. Thus, we have (as before)
\[\mu(S_i\cap\RSnew{\mathcal{C}_R})\leq \frac{4\,\pi^{\frac{d}{2}}R^{d-1}}{(d-1)\Gamma\Bigl(\frac{d}{2}\Bigr)}\,\frac{\epsilon}{|v_i|}.\]

Combining these bounds with \eqref{Si} yields
\begin{align*}
\mu(S_\epsilon^{\bm{W}}\cap\RSnew{\mathcal{C}_R})\leq \min_i\{\mu(S_i\cap\RSnew{\mathcal{C}_R})\} = \frac{4\,\pi^{\frac{d}{2}}R^{d-1}}{(d-1)\Gamma\Bigl(\frac{d}{2}\Bigr)}\frac{1}{\max|v_i|}\,\epsilon.
\end{align*} 

Meanwhile, if for a fixed $i$, $\frac{\epsilon}{A_i}<\sin(c_i\epsilon)$ where $A_i = \|\nabla_{\bm W}f\left(\bm W,\bm v; (\bm x,y)\right)_i^T\|$, for some $c_i\in[\frac{1}{A_i},\frac{\pi}{2A_i}]$, \[\mu(S_i\cap\RSnew{\mathcal{C}_R})\leq\frac{4\RSnew{c_i^{d-1}}}{(d-1)^2}\, \frac{2\pi^{\frac{d-1}{2}}R^{d-1}}{\Gamma\left(\frac{d-1}{2}\right)}\;\left(\frac{\epsilon}{|v_i|}\right)^d.\]
Consequently, for some $c>0$ satisfies $\frac{\epsilon}{{A_i}}<\sin(c\epsilon)$,  for all $i$, where $A_i = \|\nabla_{\bm W}f\left(\bm W,\bm v; (\bm x,y)\right)_i^T\|$,
\begin{align*}
\mu(S_\epsilon^{\bm{W}}\cap\RSnew{\mathcal{C}_R})\leq \min_i\{\mu(S_i\cap\RSnew{\mathcal{C}_R})\} = \frac{8}{(d-1)^2}\, \frac{\pi^{\frac{d-1}{2}}R^{d-1}}{\Gamma\left(\frac{d-1}{2}\right)}\frac{c^{d-1}}{(\max|v_i|)^d}\,\epsilon^d.
\end{align*}

\textbf{Step 2.} We now take the union bound under all possible sign patterns.
Considering all the possible activation $\sum^{d}_{k=0}\binom{n}{k}$ sign patterns \cite{matousek2013lectures}, we obtain the volume bound as
\begin{align}\label{bound1}
\mu(S_\epsilon\cap\RSnew{\mathcal{C}_R})\leq\mu(S_\epsilon^{\bm{W}}\cap\RSnew{\mathcal{C}_R})\leq\frac{4\,\pi^{\frac{d}{2}}R^{d-1}}{(d-1)\Gamma\Bigl(\frac{d}{2}\Bigr)}\frac{1}{\min_{v_i\neq 0}\{|v_i|\}}\sum^{d}_{k=0}\begin{pmatrix}
    n\\k
\end{pmatrix} \epsilon.
\end{align}
If $c>0$ satisfies $\frac{\epsilon}{{A_i}}<\sin(c\epsilon)$,  for all $i$, where $A_i = \|\nabla_{\bm W}f\left(\bm W,\bm v; (\bm x,y)\right)_i^T\|$,
\begin{align}\label{bound2}
\mu(S_\epsilon\cap\RSnew{\mathcal{C}_R})\leq \mu(S_\epsilon^{\bm{W}}\cap\RSnew{\mathcal{C}_R})\leq\frac{8}{(d-1)^2}\, \frac{\pi^{\frac{d-1}{2}}R^{d-1}}{\Gamma\left(\frac{d-1}{2}\right)}\frac{c^{d-1}}{(\min_{v_i\neq0}|v_i|)^d}\sum^{d}_{k=0}\begin{pmatrix}
    n\\k
\end{pmatrix}\,\epsilon^d.
\end{align}
Using the standard formula \RSnew{$\mathrm{vol}_{\mathbb{R}^{d+1}}(\mathcal{C}_R) = \frac{2\pi^{d/2}R^{d+1}}{\Gamma(d/2+1)}$\cite{blumenson1960derivation}, the fact that $\Gamma(\frac{d}{2}+1) = \frac{d}{2}\Gamma(\frac{d}{2})$, and Gautschi's inequality} completes the proof.
\end{proof}

\section{Technical results on probability}
In order to control the probability of sampling a forging data point, under a mild non-degeneracy assumption on the data distribution, in this section we provide some useful technical results.

\subsection{Results for \Cref{case}}\label{prob-sec3}
For linear regression and one-layer neural networks, we assume the data distribution is essentially supported on a compact set and decays swiftly outside. 
\begin{lemma}\label{probass}
    Let 
    $\mathcal{D}$ be a probability distribution \RSnew{mostly supported} on the compact set $\RSnew{K} \subset\mathbb{R}^d\times\mathbb{R}$ \RSnew{with exponentially decaying tails outside}. 
    Assume that the joint density $p(\bm x, y)$ of $\mathcal{D}$ satisfies the following conditions:
    \begin{itemize}
    \item[\emph{(i)}] $p(\bm x, y)$ is proportional to $ e^{-g(\bm x, y)}$, where $g:\mathbb{R}^d\times\mathbb{R}\to\mathbb{R}$ satisfies the Lipschitz condition that there exists a constant $L_g>0$ such that for all \((\bm x_1, y_1), (\bm x_2, y_2)\in \RSnew{K}\),\[|g(\bm x_1, y_1) - g(\bm x_2, y_2)|\leq L_g\|(\bm x_1, y_1) - (\bm x_2, y_2)||,\]
    \item[\emph{(ii)}]
    There exists $(\bm x_c, y_c) \in \RSnew{K}$ and constants $C>0$ and $\omega>0$ such that for all $t\geq t_0$,
    \[\mathbb{P}\Bigl(\|(\bm x,y)-(\bm x_c,y_c)\|>t\Bigr) \leq C\,e^{-t^\omega} \]
    where $t_0 =  \sup\{r>0:\overline{B_r(\bm x_c,y_c)}\subseteq \RSnew{K}\} $. 
    \end{itemize}
    Let $S$ be a measurable set, and $\mu(S)$ denote its Lebesgue measure. Then
    \[\mathbb{P}_{\mathcal{D}}\Big((\bm x,y)\in S\Big)\leq \frac{e^{L_g \,\text{diam}(\RSnew{K})}}{\mu(\RSnew{K})}\,\mu(S)+ Ce^{-(\text{diam}(\RSnew{K})/2)^\omega}.\]
\end{lemma}
\begin{proof}
     We begin with the estimate
\begin{equation}
    \begin{aligned}[b]
    \mathbb{P}_{\mathcal{D}}\Big((\bm x,y)\in S\Big)\, 
    & = \, \mathbb{P}_{\mathcal{D}}\Big((\bm x,y)\in S\cap \RSnew{K}\Big) + \mathbb{P}_{\mathcal{D}}\Big((\bm x,y)\in S\backslash \RSnew{K}\Big)\\
    & = \, \int_{S\cap \RSnew{K}}p(\bm x,y)\,d\bm z\,dt + \int_{S\backslash \RSnew{K}}p(\bm x,y)\,d\bm z\,dt\\
    & \leq\, p_{M}\cdot\mu(S\cap \RSnew{K}) \, +\, \mathbb{P}\Bigl(\|(\bm x,y)-(\bm x_c,y_c)\|>t_0\Bigr)\\
    & \leq\, p_M\cdot\mu(S\RSnew{\cap K}) + Ce^{-t_0^w}
    \end{aligned}
    \label{probbound}
    \end{equation}
    where 
    \(p_{M} = \sup\{\,p(\bm x,y):(\bm x,y)\in \RSnew{K}\}\quad \text{and}\quad t_0 = \sup\{r>0:\overline{B_r(\bm x_c,y_c)}\subseteq \RSnew{K} \}.\)
    
    Let \(\,(\tilde{\bm x},\tilde{y}\,) \in \arg\min_{(\bm x,y) \in \RSnew{K}} p(\bm x,y) \quad\text{where}\quad  p(\,\tilde{\bm x},\tilde{y}\,) > 0.\) By local Lipschitz continuity of the density function on the compact set $\RSnew{K}$, for any $(\bm x,y)\in \RSnew{K}$,
\[\log \bigg( \frac{ p(\bm x,y)}{p (\,\tilde{\bm x},\tilde{y}\,)}\bigg) = | g(\bm x,y) - g(\tilde{\bm x},\tilde{y})| \leq L_g\|(\bm x,y) - (\tilde{\bm x},\tilde{y})\| \leq L_g \,\text{diam}(\RSnew{K}). \]

So that 
\begin{equation}\label{probineq}p(\bm x,y)\,\leq\,p (\,\tilde{\bm x},\tilde{y}\,)\,e^{L_g \,\text{diam}(\RSnew{K})}.\end{equation}
The normalization factor of the density function is 
\begin{align*}Z &= \int_{\mathbb{R}^{d+1}}e^{-g(\bm x,y)}\,d\bm xdy\,
\geq\,  \int_{\RSnew{K}}e^{-g(\bm x,y)}\,d\bm xdy\\
&\geq \, \int_{\RSnew{K}}e^{-g(\tilde{\bm x},\tilde{y})}\,d\bm xdy
= \,e^{-g(\tilde{\bm x},\tilde{y})}\,\int_{\RSnew{K}}\,d\bm xdy\\
& = \,e^{-g(\tilde{\bm x},\tilde{y})}\mu(\RSnew{K}).
\end{align*}
Then
\begin{align*}
    p (\,\tilde{\bm x},\tilde{y}\,) & = \frac{e^{-g(\tilde{\bm x},\tilde{y})}}{Z}\leq \, \frac{e^{-g(\tilde{\bm x},\tilde{y})}}{e^{-g(\tilde{\bm x},\tilde{y})}\mu(\RSnew{K})} = \, \frac{1}{\mu(\RSnew{K})}.
\end{align*}
Finally, combining with \eqref{probineq}, we obtain that for all $(\bm x,y)\in \RSnew{K}$,
\[
p(\bm x,y) \leq p(\tilde{\bm x}, \tilde{y}) \, e^{L_g \cdot \text{diam}(\RSnew{K})} \leq \frac{e^{L_g \cdot \text{diam}(\RSnew{K})}}{\mu(\RSnew{K})}.
\]
In particular, this shows that the quantity $p_M = \sup \{ p(\bm x, y) : (\bm x, y) \in \RSnew{K} \}$ is upper bounded as
\(
p_M \leq \frac{e^{L_g \cdot \text{diam}(\RSnew{K})}}{\mu(\RSnew{K})}.
\)
Substituting this bound into \eqref{probbound} yields
\(
\mathbb{P}_{\mathcal{D}} \big( (\bm x,y) \in S \big) \leq \RSnew{e^{L_g \cdot \text{diam}(K)}\frac{\mu(S\cap K)}{\mu(K)} + C e^{-t_0^\omega}}.
\)
\end{proof}

\subsection{Proof of Theorem \ref{seconvarthm1p}}\label{prob-sec4app}
\begin{proof}
    Under \textbf{Assumption P1} let $L_g$ be the local Lipschitz constant for $g(\z)$ on the compact, convex set $D_2$. Let $ \tilde{\z} \in \arg\inf_{\z \in D_2} p(\z)$. Then there exists a $\delta$ such that $ p(\tilde{\z}) > \delta > 0$ by compactness of $D_2$ and positivity of the density function. By the local log-Lipschitz continuity of the density function\footnote{Lipschitz continuity of $g(\z)$ implies that the density $p(\z)$ is log-Lipschitz continuous.} on the compact set $D_2$, for any $\z \in D_2$ we have 
\begin{align}
    \log \bigg( \frac{ p({\z})}{ p(\tilde{\z})}\bigg) & \leq \lvert g(\z) - g(\tilde{\z})\rvert \nonumber 
    \leq L_g \norm{\z - \tilde{\z}} \leq L_g \text{diam}(D_2) \nonumber
    \\
    \implies p({\z})
    &
    \leq  p(\tilde{\z}) e^{L_g \text{diam}(D_2)}. \label{genforge21}
\end{align}
Since the scaling factor of the density  $p(\z)$ is $\bigg( \int_{\z \in \mathbb{R}^d} e^{-g(\z)} d\z\bigg)^{-1}$ 
we also have that $ \tilde{\z} \in \arg\inf_{\z \in D_2} p(\z)$ implies $ \tilde{\z} \in \arg\inf_{\z \in D_2} e^{-g(\z)}$. Then we have
\begin{align}
    p(\tilde{\z}) & = \frac{e^{-g(\tilde{\z})}}{\int_{\z \in \mathbb{R}^d} e^{-g(\z)} d\z} \nonumber
    \leq  \frac{e^{-g(\tilde{\z})}}{\int_{\z \in D_2} e^{-g(\z)} d\z} \nonumber  
    \leq  \frac{e^{-g(\tilde{\z})}}{\int_{\z \in D_2} \bigg(\inf_{\z \in D_2}e^{-g(\z)}\bigg) d\z} \nonumber
    \\& 
    = \frac{e^{-g(\tilde{\z})}}{\int_{\z \in D_2} e^{-g(\tilde{\z})} d\z} = \frac{1}{\int_{D_2} d\mu_2} = \frac{1}{\text{vol}_{\mathbb{R}^d}(D_2)}. \label{genforge22}
\end{align}
Substituting \eqref{genforge22} in \eqref{genforge21} implies that for any $\z \in D_2$
\begin{align}
    p({\z})   & \leq  \frac{e^{L_g \text{diam}(D_2)}}{\text{vol}_{\mathbb{R}^d}(D_2)} . \label{genforge23}
\end{align}
Then the anti-concentration  bound on the $\epsilon$-forging set from $\mathbb{R}^d $ for any $\z^* \in D_2$ and any $\x \in D_1$ is
\begin{align}
  \probP\bigg( \{\z \in \mathbb{R}^d : \norm{\nabla f(\x; \z) - \nabla f(\x; \z^*)} \leq \epsilon\}\bigg) & = \probP\bigg(\{\z \in D_2 : \norm{\nabla f(\x; \z) - \nabla f(\x; \z^*)} \leq \epsilon\}\bigg) \nonumber \\ & \quad+ \probP\bigg(\{\z \in \mathbb{R}^d \backslash D_2 : \norm{\nabla f(\x; \z) - \nabla f(\x; \z^*)} \leq \epsilon\}\bigg) \nonumber\\
    & \hspace{-3cm}= \int_{\z \in S_{\epsilon}(\x,\z^*)} p(\z) d\z  + \int_{\{\z \notin D_2 : \norm{\nabla f(\x; \z) - \nabla f(\x; \z^*)} \leq \epsilon \} }  p(\z) d\z \nonumber\\
     & \hspace{-3cm}\underbrace{\leq}_{\text{from } \eqref{genforge23}} \int_{\z \in S_{\epsilon}(\x,\z^*)} \frac{e^{L_g \text{diam}(D_2)}}{\text{vol}_{\mathbb{R}^d}(D_2)} d\z +\int_{\{\z \notin D_2 : \norm{\nabla f(\x; \z) - \nabla f(\x; \z^*)} \leq \epsilon \} }  p(\z) d\z \nonumber\\
      & \hspace{-3cm}\underbrace{\leq}_{\textbf{Assumption P2}} \frac{e^{L_g \text{diam}(D_2)}}{\text{vol}_{\mathbb{R}^d}(D_2)}\int_{ S_{\epsilon}(\x,\z^*)}  d\mu_2 +\int_{\norm{\z -\z_c} \geq t_0 }  p(\z) d\z  . \label{genforge100p}
\end{align}
Further simplification of \eqref{genforge100p} yields
\begin{align}
    \probP\bigg( \{\z \in \mathbb{R}^d : \norm{\nabla f(\x; \z) - \nabla f(\x; \z^*)} \leq \epsilon\}\bigg) & \nonumber \\
    & \hspace{-5cm}\underbrace{\leq}_{\text{from Theorem }\ref{seconvarthm1}} \frac{e^{L_g \text{diam}(D_2)}}{\text{vol}_{\mathbb{R}^d}(D_2)} \,\RSnew{24^d \,
        \frac{\Gamma(\frac{d}{2}+1)}{\pi^{d/2}}
        \, \text{vol}_{\mathbb{R}^d}(D_2) \,
        \left(\frac{L}{2 \gamma^2}\right)^{d/2}
       \left( \frac{\gamma^2}{4L } \right)^{r/2} \epsilon^{\frac{d- r }{2}}  }  \nonumber \\ & + \probP(\norm{\z -\z_c} \geq t_0) \nonumber\\
    & \hspace{-3cm} \leq \RSnew{24^d \,
        \frac{\Gamma(\frac{d}{2}+1)}{\pi^{d/2}}
        \, e^{L_g \mathrm{diam}(D_2)} \,
        \left(\frac{L}{2 \gamma^2}\right)^{d/2}
       \left( \frac{\gamma^2}{4L } \right)^{r/2}  \epsilon^{\frac{d- r }{2}} }  + C  e^{-t_0^{\omega}} , \label{genforge100}
\end{align}
where $r \leq d-1 $ from \textbf{Assumption A3}.
\end{proof}

\subsection{Proof of Theorem \ref{aeforgethm2}}
\label{prob-sec6app}
\begin{proof}
   Let $\probP $ be a  probability measure that satisfies assumptions \textbf{P1-P2}. \RSnew{ Recall that
\begin{align*}
S_{\epsilon}(\x,\z^*,K_1)
:=\Big\{\, \z' \in \bigcup_{\z \in K_1} \mathcal{B}_{\rho_\epsilon}(\z)\ :\ 
\|\nabla_{\x} f(\x;\z')-\nabla_{\x} f(\x;\z^*)\|\le \epsilon\,\Big\} \subset D_2 \backslash V_2 .
\end{align*} 
From the proof of Theorem \ref{aeforgethm1}, $\tilde{O}_1(\epsilon) = \bigcup_{j=1}^{N_1} \mathcal{B}_{\rho_{\epsilon}}(\z^*_j) $ be the finite $\rho_{\epsilon}$ maximal sub-cover for $ S_{\epsilon}(\x,\z^*, K_1) \cap K_1$ in $\mathbb{R}^d$  with centers $\{\z_j^*\}_{j=1}^{N_1}\subset K_1 \cap S_{\epsilon}(\x,\z^*,K_1)$, $\|\z_i^*-\z_j^*\|\ge\rho_\epsilon$ for $i\neq j$ and $N_1$ be the covering number of $ \tilde{O}_1(\epsilon)$. } Under \textbf{Assumption P1} denote by $L_g$ the local Lipschitz constant for $g(\z)$ on the compact, convex set $D_2 \supset {\tilde{O}_1(\epsilon)} \supset \RSnew{S_{\epsilon}(\x,\z^*,K_1) \cap K_1}$. 
  Then for any $\z \in D_2$ the bound \eqref{genforge23} holds, i.e., 
\begin{align}
    p({\z})   & \leq  \frac{e^{L_g \text{diam}(D_2)}}{\text{vol}_{\mathbb{R}^d}(D_2) } \hspace{0.2cm} \forall \z \in D_2. \label{genforge23a}
\end{align}
\RSnew{Let $O_1(\rho_{\epsilon}) = \bigcup_{\z \in K_1} \mathcal{B}_{\rho_{\epsilon}}(\z) $ be as in Lemma \ref{coverlemma00} where $\rho_\epsilon=\sqrt{2\epsilon/L}$ .  For $\epsilon < \frac{L}{2} \xi^2$ or $\rho_{\epsilon} < \xi$, where $\xi = \xi(\nu)$ as in Lemma \ref{coverlemma00}, the cover $O_1(\rho_{\epsilon}) $ is separated from $ \partial D_2 $, $ D_2 \cap V_2$. Also, recall that $\tilde{O}_1(\epsilon)  \subset O_1(\rho_{\epsilon}) $.
Hence, by the inclusion-exclusion principle,  
\begin{align}
    \{\z \in  (D_2 \backslash V_2) \backslash \tilde{O}_1(\epsilon) : \norm{\nabla f(\x; \z) - \nabla f(\x; \z^*)} \leq \epsilon\}  & \nonumber \\ & \hspace{-6cm} \subset  \{\z \in  (D_2 \backslash V_2) \backslash (S_{\epsilon}(\x,\z^*,K_1) \cap K_1) : \norm{\nabla f(\x; \z) - \nabla f(\x; \z^*)} \leq \epsilon\} \\
    & \hspace{-6cm} =  \{\z \in  (D_2 \backslash V_2) \cap (S^c_{\epsilon}(\x,\z^*,K_1) \cup K^c_1) : \norm{\nabla f(\x; \z) - \nabla f(\x; \z^*)} \leq \epsilon\} \\
    & \hspace{-6cm} \subseteq \underbrace{\{\z \in  (D_2 \backslash V_2) \cap S^c_{\epsilon}(\x,\z^*,K_1)   : \norm{\nabla f(\x; \z) - \nabla f(\x; \z^*)} \leq \epsilon\}}_{\subset \, (D_2 \backslash V_2) \backslash  K_1} \nonumber \\
    & \hspace{-5cm} \cup   \{\z \in  (D_2 \backslash V_2) \cap  K^c_1 : \norm{\nabla f(\x; \z) - \nabla f(\x; \z^*)} \leq \epsilon\} \subset (D_2 \backslash V_2) \backslash  K_1 .\label{incexctemp1}
\end{align}
Further,
\begin{align}
    \{\z \in \tilde{O}_1(\epsilon)  : \norm{\nabla f(\x; \z) - \nabla f(\x; \z^*)} \leq \epsilon\} &\subseteq S_{\epsilon}(\x,\z^*,K_1) \nonumber\\&=  \bigg(S_{\epsilon}(\x,\z^*,K_1) \RSnew{\cap K_1} \bigg) \bigcup \bigg( S_{\epsilon}(\x,\z^*,K_1) \RSnew{\cap K^c_1}\bigg).  \label{incexctemp2}
\end{align}}
Then, the anti-concentration probability bound on the $\epsilon$-forging set from $\mathbb{R}^d \backslash V_2$, for any $\z^* \in D_2 \backslash V_2$, and for $\mu_1 \text{ a.e.}$ in $D_1$, is given by:
\begin{align}
   \probP\bigg( \{\z \in \mathbb{R}^d \backslash V_2 : \norm{\nabla f(\x; \z) - \nabla f(\x; \z^*)} &\leq \epsilon\}\bigg)  = \probP\bigg(\{\z \in D_2 \backslash V_2 : \norm{\nabla f(\x; \z) - \nabla f(\x; \z^*)} \leq \epsilon\}\bigg) \nonumber \\
   & \hspace{-3cm}+ \probP\bigg(\{\z \in (\mathbb{R}^d \backslash V_2) \backslash (D_2 \backslash V_2 ): \norm{\nabla f(\x; \z) - \nabla f(\x; \z^*)} \leq \epsilon\}\bigg) \nonumber \\
   & \hspace{-4cm}= \probP\bigg(\{\z \in \tilde{O}_1(\epsilon)  : \norm{\nabla f(\x; \z) - \nabla f(\x; \z^*)} \leq \epsilon\}\bigg) \nonumber \\
   & \hspace{-2cm}+ \probP\bigg(\{\z \in  (D_2 \backslash V_2) \backslash \tilde{O}_1(\epsilon) : \norm{\nabla f(\x; \z) - \nabla f(\x; \z^*)} \leq \epsilon\}\bigg) \nonumber \\
   & \hspace{-1cm}+ \probP\bigg(\{\z \in \text{ext}(D_2) \backslash V_2 : \norm{\nabla f(\x; \z) - \nabla f(\x; \z^*)} \leq \epsilon\}\bigg) \nonumber \\
      \end{align}
\begin{align}
    \hspace{-3cm}&  \RSnew {\underbrace{\leq}_{\eqref{incexctemp2}}} \int_{\z \in S_{\epsilon}(\x,\z^*,K_1) \RSnew{\cap K_1}} p(\z) d\z + \RSnew{\int_{\z \in S_{\epsilon}(\x,\z^*,K_1) \RSnew{\cap K^c_1} } p(\z) d\z} \nonumber \\  & \qquad + \int_{\RSnew{\{\z \in  (D_2 \backslash V_2) \backslash \tilde{O}_1(\epsilon) : \norm{\nabla f(\x; \z) - \nabla f(\x; \z^*)} \leq \epsilon\} }} p(\z) d\z \nonumber \\ & \qquad\qquad + \int_{\{\z \in \text{ext}(D_2) \backslash V_2 : \norm{\nabla f(\x; \z) - \nabla f(\x; \z^*)} \leq \epsilon \} }  p(\z) d\z \nonumber \\
    \hspace{-3cm}& \underbrace{\leq}_{\eqref{genforge23a}, \RSnew{ \eqref{incexctemp1}}}  \frac{e^{L_g \text{diam}(D_2)}}{\text{vol}_{\mathbb{R}^d}(D_2) } \int_{S_{\epsilon}(\x,\z^*,K_1) \RSnew{\cap K_1}}  d\mu_2  +  \frac{e^{L_g \text{diam}(D_2)}}{\text{vol}_{\mathbb{R}^d}(D_2) } \int_{\RSnew{\z \in  (D_2 \backslash V_2) \cap  K^c_1 }}  d\mu_2 \nonumber \\ & \qquad\qquad  +  \RSnew{ \frac{e^{L_g \text{diam}(D_2)}}{\text{vol}_{\mathbb{R}^d}(D_2) }\int_{\z \in  \RSnew{  (D_2 \backslash V_2 ) \backslash K_1} }  d\mu_2} +\int_{\{\z \in \text{ext}(D_2) \backslash V_2 : \norm{\nabla f(\x; \z) - \nabla f(\x; \z^*)} \leq \epsilon \} }  p(\z) d\z  .\nonumber
    \end{align}
   Invoking Theorem \ref{aeforgethm1}, $\mu_2(D_2 \backslash (V_2 \cup \partial D_2)) < \mu_2(K_{1}) + \nu_1 $ along with $\mu_2(\partial D_2) = 0$, \RSnew{the set inclusion \eqref{incexctemp1}} in the last step leads to the following simplification for $\mu_1 \text{ a.e.}$ in $D_1$: 
    \begin{align}
    \probP\bigg( \{\z \in \mathbb{R}^d \backslash V_2 : \norm{\nabla f(\x; \z) - \nabla f(\x; \z^*)} \leq \epsilon\}\bigg) &  \nonumber \\ & \hspace{-5cm} \underbrace{\leq}_{\textbf{Theorem } \ref{aeforgethm1}, \eqref{measureineqa1}} \RSnew{ 16^d \,
        \frac{\Gamma(\frac{d}{2}+1)}{\pi^{d/2}}
        \,\mathrm{vol}_{\mathbb{R}^d}(D_2)\,
        \left(\frac{L}{2 \gamma^2}\right)^{d/2}
       \left( \frac{\gamma^2}{4L } \right)^{r/2} 
        \epsilon^{\frac{d-r}{2}} \, \, }  \frac{e^{L_g \text{diam}(D_2)}}{\text{vol}_{\mathbb{R}^d}(D_2) } \nonumber \\
    & \hspace{-3cm} +  \RSnew{2}\frac{e^{L_g \text{diam}(D_2)}}{\text{vol}_{\mathbb{R}^d}(D_2) } \nu_1 + \int_{\{\z \in \text{ext}(D_2) \backslash V_2 : \norm{\nabla f(\x; \z) - \nabla f(\x; \z^*)} \leq \epsilon \} }  p(\z) d\z  \nonumber \\
      & \hspace{-5.5cm}\underbrace{\leq}_{\textbf{Assumption P2}}  \RSnew{ 16^d \,
        \frac{\Gamma(\frac{d}{2}+1)}{\pi^{d/2}}
        \,e^{L_g \text{diam}(D_2)}\,
        \left(\frac{L}{2 \gamma^2}\right)^{d/2}
       \left( \frac{\gamma^2}{4L } \right)^{r/2} 
        \epsilon^{\frac{d-r}{2}} \, \, }  \nonumber \\
    & \hspace{-3cm} +  \RSnew{2}\frac{e^{L_g \text{diam}(D_2)}}{\text{vol}_{\mathbb{R}^d}(D_2) } \nu_1 + \int_{\norm{\z -\z_c} \geq t_0 }  p(\z) d\z  \nonumber.
    \end{align}
    Then using \textbf{Assumption P2} on the last summand of the above inequality yields
    \begin{align}
    \probP\bigg( \{\z \in \mathbb{R}^d \backslash V_2 : \norm{\nabla f(\x; \z) - \nabla f(\x; \z^*)} \leq \epsilon\}\bigg) &  \nonumber \\
    & \hspace{-3.5cm} {\leq} \, \, \RSnew{16^d \,
        \frac{\Gamma(\frac{d}{2}+1)}{\pi^{d/2}}
        \,e^{L_g \text{diam}(D_2)}\,
        \left(\frac{L}{2 \gamma^2}\right)^{d/2}
       \left( \frac{\gamma^2}{4L } \right)^{r/2} 
        \epsilon^{\frac{d-r}{2}} \, \, }  \nonumber \\
    & \hspace{-3cm} +  \RSnew{2}\frac{e^{L_g \text{diam}(D_2)}}{\text{vol}_{\mathbb{R}^d}(D_2) } \nu_1 + C  e^{-t_0^{\omega}}  \quad  , \quad \mu_1 \text{ a.e. in } D_1
    \label{genforge100p1}
\end{align}
where $\epsilon$ is a function of $\nu_1$and $ \epsilon \to 0$ as $\nu_1 \downarrow 0$. The exact rate of decay for $\epsilon$ in terms of $\nu_1$ depends on the geometry of the set $K_1$ and therefore cannot be determined in general.
\end{proof}

\section{Proofs for \Cref{generalforgeanalysismainsec}}
\subsection{Proof of Lemma \ref{seconvarlem1}}\label{sec4proof1}
\begin{proof}
    \RSnew{    Let
    \(        \h:=\z^*-\z.    \)
    By the fundamental theorem of calculus,
    \begin{align}
        \nabla_{\x}f(\x;\z^*)-\nabla_{\x}f(\x;\z)
        =
        \int_0^1 \nabla_{\z}\nabla_{\x}f(\x;\z+t \h)\,\h\,dt.
        \label{eq:lem1_ftc}
    \end{align}
    Rearranging gives
    \begin{align}
        \nabla_{\z}\nabla_{\x}f(\x;\z^*)\h
        &=
        \nabla_{\x}f(\x;\z^*)-\nabla_{\x}f(\x;\z)
        -
        \int_0^1
        \Big(
        \nabla_{\z}\nabla_{\x}f(\x;\z+t \h)
        -
        \nabla_{\z}\nabla_{\x}f(\x;\z^*)
        \Big)\h\,dt .
    \end{align}
    Taking norms and using the operator norm yields
    \begin{align}
        \|\nabla_{\z}\nabla_{\x}f(\x;\z^*)\h\|
        &\leq
        \|\nabla_{\x}f(\x;\z^*)-\nabla_{\x}f(\x;\z)\|
        \nonumber\\
        &\qquad +
        \int_0^1
        \left\|
        \nabla_{\z}\nabla_{\x}f(\x;\z+t \h)
        -
        \nabla_{\z}\nabla_{\x}f(\x;\z^*)
        \right\|
        \|\h\|\,dt .
        \label{eq:lem1_before_lipschitz}
    \end{align}
    By \textbf{A2}, the map
    \(
        \z \mapsto \nabla_{\z}\nabla_{\x}f(\x;\z)
    \)
    is Lipschitz on any compact set containing the line segment joining $\z$ and $\z^*$. Let $L$ be a Lipschitz constant on such a compact set. Then
    \begin{align}
        \left\|
        \nabla_{\z}\nabla_{\x}f(\x;\z+t \h)
        -
        \nabla_{\z}\nabla_{\x}f(\x;\z^*)
        \right\|
        \leq
        L\|\z+t \h-\z^*\|
        =
        L(1-t)\|\h\|.
    \end{align}
    Substituting this into \eqref{eq:lem1_before_lipschitz}, we obtain
    \begin{align}
        \|\nabla_{\z}\nabla_{\x}f(\x;\z^*)\h\|
        &\leq
        \|\nabla_{\x}f(\x;\z^*)-\nabla_{\x}f(\x;\z)\|
        +
        \int_0^1 L(1-t)\|\h\|^2\,dt
        \nonumber\\
        &=
        \|\nabla_{\x}f(\x;\z^*)-\nabla_{\x}f(\x;\z)\|
        +
        \frac{L}{2}\|\h\|^2.
    \end{align}
    Since $\h=\z^*-\z$, this proves \eqref{eq:lem1_main}.     Now assume that $\|\z^*-\z\|\leq \sqrt{\frac{2\epsilon}{L}}$ and that $\z\in S_\epsilon(\x,\z^*)$, that is,
    \(
        \|\nabla_{\x}f(\x;\z^*)-\nabla_{\x}f(\x;\z)\|\leq \epsilon.
    \)
    Then \eqref{eq:lem1_main} gives
    \begin{align}
        \|\nabla_{\z}\nabla_{\x}f(\x;\z^*)(\z^*-\z)\|
        &\leq
        \epsilon
        +
        \frac{L}{2}\|\z^*-\z\|^2
\leq
                2\epsilon,
    \end{align}
    as claimed.}
\end{proof}
\subsection{Proof of Lemma \ref{seconvarlem2}}\label{sec4proof2}
\begin{proof}
    \RSnew{    Since $\z$ $\eta$-forges $\z^*$ and
    \(
        \|\z-\z^*\|\leq \sqrt{\frac{2\eta}{L}},
    \)
    Lemma \ref{seconvarlem1} gives
    \begin{align}
        \|\M_0(\z^*)(\z-\z^*)\|
        \leq 2\eta.
        \label{eq:lem2_proof_1}
    \end{align}
    Set
    \(
        \h:=\z-\z^*,
    \)
    and decompose $h$ orthogonally as
    \(
        \h=\u+\v,\) where \(
        \u\in E_{<\gamma}(\z^*), \v\in E_{\geq \gamma}(\z^*).
    \)
    Then
    \begin{align}
        \|\u\|
        \leq \|\h\|
        \leq \sqrt{\frac{2\eta}{L}}.
        \label{eq:lem2_proof_2}
    \end{align}
    Because $E_{<\gamma}(\z^*)$ and $E_{\geq \gamma}(\z^*)$ are orthogonal sums of right singular spaces of $\M_0(\z^*)$, we have
    \(
        \|\M_0(\z^*)\h\|^2
        =
        \|\M_0(\z^*)\u\|^2 + \|\M_0(\z^*)\v\|^2.
    \)
    In particular,
    \(
        \|\M_0(\z^*)\v\|
        \leq
        \|\M_0(\z^*)\h\|.
    \)
    Moreover, every singular value of $\M_0(\z^*)$ restricted to $E_{\geq \gamma}(\z^*)$ is at least $\gamma$, so
    \begin{align}
        \|\M_0(\z^*)\v\|
        \geq
        \gamma \|\v\|.
        \label{eq:lem2_proof_3}
    \end{align}
    Combining \eqref{eq:lem2_proof_1} and \eqref{eq:lem2_proof_3} yields
    \begin{align}
        \|\v\|
        \leq
        \frac{2\eta}{\gamma}.
        \label{eq:lem2_proof_4}
    \end{align}
    Equations \eqref{eq:lem2_proof_2} and \eqref{eq:lem2_proof_4} imply
    \eqref{eq:local_spectral_inclusion}.
   Now, by \textbf{A3},  $E_{<\gamma}(\z^*)$ has dimension at most $r$. Hence
    \begin{align}
        \mathrm{vol}_{\mathbb{R}^d}\Big(
        S_\eta(\x,\z^*)\cap \mathcal{B}_{\sqrt{\frac{2\eta}{L}}}(\z^*)
        \Big) &\leq
        \mathrm{vol}_{\mathbb{R}^k}\bigg(\bigg\{\u : \|\u\|
        \leq \sqrt{\frac{2\eta}{L}}\bigg\}\bigg) \mathrm{vol}_{\mathbb{R}^{d-k}}\bigg(\bigg\{\v : \|\v\|
        \leq
        \frac{2\eta}{\gamma}\bigg\}\bigg) \\
        & =
        \omega_k\,\omega_{d-k}\,
        \left(\sqrt{\frac{2\eta}{L}}\right)^k
        \left(\frac{2\eta}{\gamma}\right)^{d-k}, 
        \label{eq:lem2_proof_5}
    \end{align}
    where
    \(
        k:=\dim(E_{<\gamma}(\z^*))\leq r
    \)
    and
    \(
        \omega_m:=\mathrm{vol}_{\mathbb{R}^m}(\mathcal{B}_1(\mathbf{0})).
    \)
    Since $\omega_m\leq 2^m$ for every $m\geq 0$, it follows that
    \(
        \omega_k\,\omega_{d-k}\leq 2^d.
    \)
    Therefore \eqref{eq:lem2_proof_5} gives
    \begin{align}
        \mathrm{vol}_{\mathbb{R}^d}\Big(
        S_\eta(\x,\z^*)\cap \mathcal{B}_{\sqrt{\frac{2\eta}{L}}}(\z^*)
        \Big)
        &\leq
        2^d
         \left(\sqrt{\frac{2}{L}}\right)^k
        \left(\frac{2}{\gamma}\right)^{d-k}
        \eta^{\,d-\frac{k}{2}}
,
    \end{align}
    which proves
    \eqref{eq:local_spectral_volume}.}
\end{proof}

\subsection{Proof of Theorem \ref{seconvarthm1}}\label{sec4proof3}
\begin{proof}
\RSnew{We start with a standard covering argument.    Set
    \(
        \rho_\epsilon:=\sqrt{\frac{2\epsilon}{L}}.
    \)
    and let $\{\z_i\}_{i=1}^N\subset S_\epsilon(\x,\z^*)$ be a maximal $\rho_\epsilon$-separated subset of $S_\epsilon(\x,\z^*)$. Since $S_\epsilon(\x,\z^*)\subset D_2$ and $D_2$ is compact, the set $\{\z_i\}_{i=1}^N$ is finite. By maximality,
    \(
        S_\epsilon(\x,\z^*)
        \subseteq
        \bigcup_{i=1}^N \mathcal{B}_{\rho_\epsilon}(\z_i),
    \)
    and the balls
    \(
        \mathcal{B}_{\rho_\epsilon/2}(\z_i),
        \quad i=1,\dots,N,
    \)
    are pairwise disjoint.    Since $S_\epsilon(\x,\z^*)\subseteq D_2$, we have
    \(
        \bigcup_{i=1}^N \mathcal{B}_{\rho_\epsilon/2}(\z_i)
        \subseteq
        D_2+\mathcal{B}_{\rho_\epsilon/2}(\mathbf{0}),
    \)
    and hence
    \begin{align}
        N\,\omega_d\left(\frac{\rho_\epsilon}{2}\right)^d
        \leq
        \mathrm{vol}_{\mathbb{R}^d}\Big(
        D_2+\mathcal{B}_{\rho_\epsilon/2}(\mathbf{0})
        \Big),
        \label{eq:thm_proof_1}
    \end{align}
    where
    \(
        \omega_d=\mathrm{vol}_{\mathbb{R}^d}(\mathcal{B}_1(0)) =\frac{\pi^{d/2}}{\Gamma(\frac{d}{2}+1)}.
    \)
    We now estimate the right-hand side. Let
    \(
        \widetilde D_2:=D_2-\c.
    \)
    Then $\widetilde D_2$ is convex and contains $\mathcal{B}_\tau(\mathbf{0})$. Set
    \(
        t_\epsilon:=\frac{\rho_\epsilon}{2\tau}.
    \)
    If $\z \in \widetilde D_2$ and $\y\in \mathcal{B}_{\rho_\epsilon/2}(\mathbf{0})$, then
    \(
        \|\y\|\leq \frac{\rho_\epsilon}{2}=t_\epsilon\tau,
    \)
    so there exists $\b\in \mathcal{B}_\tau(\mathbf{0})\subset \widetilde D_2$ such that
    \(
        \y=t_\epsilon \b.
    \)
    Therefore
    \[
        \z+\y
        =
        (1+t_\epsilon)
        \left(
        \frac{1}{1+t_\epsilon}\z
        +
        \frac{t_\epsilon}{1+t_\epsilon}\b
        \right).
    \]
    Since $\widetilde D_2$ is convex, the vector in parentheses is in $\widetilde D_2$, and thus
    \(
        \widetilde D_2+\mathcal{B}_{\rho_\epsilon/2}(\mathbf{0})
        \subseteq
        (1+t_\epsilon)\widetilde D_2.
    \)
    Translating back, we obtain
    \(
        D_2+\mathcal{B}_{\rho_\epsilon/2}(\mathbf{0})
        \subseteq
        \c+(1+t_\epsilon)(D_2-\c).
    \)
    Taking Lebesgue measure and using translation invariance and homogeneity under dilation gives
    \begin{align}
        \mathrm{vol}_{\mathbb{R}^d}\Big(
        D_2+\mathcal{B}_{\rho_\epsilon/2}(\mathbf{0})
        \Big)
        \leq
        (1+t_\epsilon)^d\,\mathrm{vol}_{\mathbb{R}^d}(D_2).
        \label{eq:thm_proof_2}
    \end{align}
    Since
    \(
        \epsilon\leq \frac{L\tau^2}{2},
    \)
    we have $\rho_\epsilon\leq \tau$, hence
    \(
        t_\epsilon=\frac{\rho_\epsilon}{2\tau}\leq \frac12,
    \)
    so
    \(
        1+t_\epsilon\leq \frac32.
    \)
    Combining \eqref{eq:thm_proof_1} and \eqref{eq:thm_proof_2} yields
    \begin{align}
        N
        &\leq
        \frac{(3/2)^d\,\mathrm{vol}_{\mathbb{R}^d}(D_2)}
        {\omega_d\,(\rho_\epsilon/2)^d}
        \nonumber\\        &
        =
        3^d
        \frac{\Gamma(\frac{d}{2}+1)}{\pi^{d/2}}
        \,\mathrm{vol}_{\mathbb{R}^d}(D_2)\,
        \left(\frac{L}{2\epsilon}\right)^{d/2}.
        \label{eq:thm_proof_3}
    \end{align}
    Next, let
    \(
        A_i:=S_\epsilon(\x,\z^*)\cap \mathcal{B}_{\rho_\epsilon}(\z_i).
    \)
    If $\z\in A_i$, then both $\z$ and $\z_i$ belong to $S_\epsilon(\x,\z^*)$, so by the triangle inequality,
    \[
        \|\nabla_{\x}f(\x;\z)-\nabla_{\x}f(\x;\z_i)\|
        \leq
        \|\nabla_{\x}f(\x;\z)-\nabla_{\x}f(\x;\z^*)\|
        +
        \|\nabla_{\x}f(\x;\z_i)-\nabla_{\x}f(\x;\z^*)\|
        \leq 2\epsilon.
    \]
    Hence
    \(
        \z\in S_{2\epsilon}(\x,\z_i).
    \)
    Also,
    \(
        \|\z-\z_i\|
        \leq \rho_\epsilon
        =
        \sqrt{\frac{2\epsilon}{L}}
        \leq
        \sqrt{\frac{4\epsilon}{L}}.
    \)
    Therefore
    \(
        A_i
        \subseteq
        S_{2\epsilon}(\x,\z_i)\cap \mathcal{B}_{\sqrt{\frac{4\epsilon}{L}}}(\z_i).
    \)
    Applying the inclusion \eqref{eq:local_spectral_inclusion} from Lemma \ref{seconvarlem2} with $\eta=2\epsilon$, we obtain
    \begin{align}
        \mathrm{vol}_{\mathbb{R}^d}(A_i)
        &\leq
        \omega_{k_i}\,\omega_{d-{k_i}}
        \left(\sqrt{\frac{4\epsilon}{L}}\right)^{k_i}
        \left(\frac{4\epsilon}{\gamma}\right)^{d-{k_i}}
        \nonumber\\
        &\leq
        2^d
         \left(\sqrt{\frac{4}{L}}\right)^{k_i}
        \left(\frac{4}{\gamma}\right)^{d-{k_i}}
        \epsilon^{\,d-\frac{{k_i}}{2}} \nonumber \\
        & = \left(\frac{8 \epsilon}{\gamma}\right)^d
         \left( \frac{\gamma^2}{16 \epsilon} {\frac{4}{L}}\right)^{k_i/2} \le \left(\frac{8 \epsilon}{\gamma}\right)^d
         \left( \frac{\gamma^2}{16 \epsilon} {\frac{4}{L}}\right)^{r/2}  
        \label{eq:thm_proof_4}
    \end{align}
    where
    \(
        {k_i}:=\dim(E_{<\gamma}(\z_i))\le r
    \) and in the last step we used the bound $\epsilon \le \frac{\gamma^2}{4L}$.
    Finally, summing over the covering balls and using \eqref{eq:thm_proof_3} and \eqref{eq:thm_proof_4}, 
    \begin{align}
        \mu_2\big(S_\epsilon(\x,\z^*)\big)
        &\leq
        \sum_{i=1}^N \mathrm{vol}_{\mathbb{R}^d}(A_i)
   \leq
        N\cdot
       \left(\frac{8 \epsilon}{\gamma}\right)^d
         \left( \frac{\gamma^2}{4L \epsilon} \right)^{r/2}  
        \nonumber\\
        &\leq
        3^d
        \frac{\Gamma(\frac{d}{2}+1)}{\pi^{d/2}}
        \,\mathrm{vol}_{\mathbb{R}^d}(D_2)\,
        \left(\frac{L}{2\epsilon}\right)^{d/2}
        \cdot
        \left(\frac{8 \epsilon}{\gamma}\right)^d
         \left( \frac{\gamma^2}{4L \epsilon} \right)^{r/2}  
        \nonumber\\
        &=
        24^d \,
        \frac{\Gamma(\frac{d}{2}+1)}{\pi^{d/2}}
        \,\mathrm{vol}_{\mathbb{R}^d}(D_2)\,
        \left(\frac{L}{2 \gamma^2}\right)^{d/2}
       \left( \frac{\gamma^2}{4L } \right)^{r/2} 
        \epsilon^{\frac{d-r}{2}}. \label{genforge008}
    \end{align}}
\end{proof}

\section{Proofs for Section \ref{sectionaeforge}}
\subsection{Proof of Lemma \ref{coverlemma00}}\label{sec6proof1}
\begin{proof}
    We note that $K_1$ and $D_2 \cap V_2$ are compact, and $K_1\cap V_2 = \emptyset$. Since $\mathbb{R}^d$ is a Hausdorff space, there exists some $\xi_1 :=\xi_1(\nu_1) >0$ such that  
\[
    O_1(\xi_1) = \bigcup_{\z \in K_1} \mathcal{B}_{\xi_1}(\z) \hspace{0.2cm}, \hspace{0.2cm} O_2(\xi_1) = \bigcup_{\z \in D_2 \cap V_2} \mathcal{B}_{\xi_1}(\z)  \label{eq:cover1} \tag{\textbf{cover1}}
\] 
are non-intersecting uniform open covers of $K_1, D_2 \cap V_2$ respectively (Lemma \ref{suplem0}). Further, the set $D_2$ is convex and compact hence has a compact boundary. Since  $\partial D_2$ is compact, $K_1 \cap \partial D_2= \emptyset $, and $\mathbb{R}^d$ is a Hausdorff space, there exists some $\xi_2 := \xi_2(\nu_1) >0$ such that 
\[
  O_1(\xi_2) = \bigcup_{\z \in K_1} \mathcal{B}_{\xi_2}(\z) \hspace{0.2cm}, \hspace{0.2cm} O_3(\xi_2) = \bigcup_{\z \in \partial D_2} \mathcal{B}_{\xi_2}(\z)  \label{eq:cover2} \tag{\textbf{cover2}}
\]
are non-intersecting uniform open covers of $K_1, \partial D_2 $ respectively from Lemma \ref{suplem0}. Let $\xi := \xi(\nu_1) = \min\{\xi_1, \xi_2\}$. Then the covers $O_1(\xi) \supset K_1$, $ O_2(\xi) \supset D_2 \cap V_2$, $ O_3(\xi) \supset \partial D_2$ satisfy 
$$ O_1(\xi) \cap O_2(\xi) = \emptyset, \hspace{0.2cm} O_1(\xi) \cap O_3(\xi) = \emptyset, \hspace{0.2cm}  O_3(\xi) \subset D_2 + \mathcal{B}_{\xi}(\mathbf{0}), \hspace{0.2cm} O_1(\xi) \subseteq \text{int}(D_2).$$
It is straightforward to show that the second last inclusion holds. We now show that the last inclusion holds. Recall that $ O_1(\xi) \cap O_3(\xi) = \emptyset $ and thus $$  O_1(\xi)  = (O_1(\xi) \cap \text{int}(D_2) )\cup (O_1(\xi) \cap \text{ext}(D_2) ).$$ 
where  $(O_1(\xi) \cap \text{int}(D_2) ) $ and $(O_1(\xi) \cap \text{ext}(D_2) ) $ are disjoint. If not,  there exists a ball $\mathcal{B}_{\xi}(\z) \subset O_1(\xi)$ such that $ \mathcal{B}_{\xi}(\z)  \cap \text{int}(D_2) \neq \emptyset$ and  $\mathcal{B}_{\xi}(\z)  \cap \text{ext}(D_2) \neq \emptyset $. Let $\y_1 \in \mathcal{B}_{\xi}(\z)  \cap \text{int}(D_2) $, $\y_2 \in \mathcal{B}_{\xi}(\z)  \cap \text{ext}(D_2) $ and $\y_t = (1-t)\y_1+ t\y_2$ for any $t \in [0,1]$. Since the line joining $\y_1,\y_2$ intersects $ \partial D_2$ and $\mathcal{B}_{\xi}(\z)  $ is convex, then $ \y_s \in \mathcal{B}_{\xi}(\z)  \cap \partial D_2 $ for a unique $s \in (0,1)$ and so $ \mathcal{B}_{\xi}(\z)  \cap \partial D_2 \neq \emptyset$, a contradiction since $  O_1(\xi) \cap \partial D_2 = \emptyset$. Since $(O_1(\xi) \cap \text{int}(D_2) ) $, $(O_1(\xi) \cap \text{ext}(D_2) ) $ are disjoint, it must be that $O_1(\xi) \cap \text{ext}(D_2)$ is a union of balls with centers in $\text{ext}(D_2)$ and since the balls in $O_1(\xi)$ have centers in $K_1 \subset D_2 \backslash (V_2 \cup \partial D_2)  \subseteq \text{int}(D_2)$ then $O_1(\xi) \cap \text{ext}(D_2) = \emptyset$. Because $K_1 \subset O_1(\xi) \subseteq \text{int}(D_2)$ we have
\begin{align*}
   0 \leq  \mu_2(D_2 \backslash(V_2 \cup \partial D_2)) - \mu_2(O_1(\xi)) = \mu_2(D_2) - \mu_2(O_1(\xi)) < \nu_1 .
\end{align*}
In fact, for any $\nu_1>0$ where $K_1 \subset O_1(\xi) \subseteq \text{int}(D_2)$ and $\xi$ is a function of $\nu_1$, the above bound holds.
Last, it remains to show that $\xi \to 0$ as $\nu_1 \downarrow 0$. Consider an arbitrary decreasing sequence $ \{\nu_{1,j}\}_j$ with $\nu_{1,j} \downarrow 0 $. Then for every $ \nu_{1,j}$ there exists a compact $K_{1,j} \subset D_2 \backslash (V_2 \cup \partial D_2)$, that depends on $\nu_{1,j}$ with
\begin{align}
    0< \mu_2(D_2 ) - \mu_2(K_{1,j}) =\mu_2(D_2 \backslash (V_2 \cup \partial D_2)) - \mu_2(K_{1,j}) < \nu_{1,j}, \label{uscxi1}
\end{align}
from \RSnew{\Cref{defk2}} and $ \lim_{j \to \infty}\mu_2(K_{1,j}) = \mu_2(D_2 \backslash (V_2 \cup \partial D_2)) $ by inner regularity of $\mu_2$. For each $K_{1,j} $ there exist open covers $ O_1(\xi_{j}), O_2(\xi_{j}), O_3(\xi_{j})$ with the following properties: $ O_1(\xi_{j}), O_2(\xi_{j})$ are non-intersecting uniform open covers of the disjoint compact sets $K_{1,j}, D_2 \cap V_2 $ and $ O_1(\xi_{j}), O_3(\xi_{j})$ are non-intersecting uniform open covers of the disjoint compact sets $K_{1,j}, \partial D_2 $ from \eqref{eq:cover1}, \eqref{eq:cover2} respectively and Lemma \ref{suplem0} where $\xi_{j}>0$. Let $V_2 \cap \text{int}(D_2) \neq \emptyset$ without loss of generality. Otherwise, the set $K_{1,j}$ can be easily obtained by uniformly shrinking $D_2$ and then showing $\xi_{j} \to 0$ as $\nu_{1,j} \downarrow 0$ is trivial. Since $ K_{1,j} \subset O_1(\xi_{j}) \subseteq \text{int}(D_2) $ and $O_1(\xi_{j}) \cap O_2(\xi_{j})  = \emptyset  $, we have for any $j$ that the disjoint union $ (O_2(\xi_{j}) \cap \text{int}(D_2)) \cup K_{1,j}  \subseteq \text{int}(D_2) $ and therefore we have 
$$ \mu_2(O_2(\xi_{j}) \cap \text{int}(D_2)) + \mu_2( K_{1,j}) = \mu_2((O_2(\xi_{j}) \cap \text{int}(D_2)) \cup K_{1,j} ) \leq \mu_2(\text{int}(D_2)).$$
Hence
\begin{align}
    \mu_2(D_2 \backslash (V_2 \cup \partial D_2)) - \mu_2(K_{1,j}) - \mu_2(O_2(\xi_{j}) \cap \text{int}(D_2)) & = \mu_2(D_2) - \mu_2(K_{1,j}) - \mu_2(O_2(\xi_{j}) \cap \text{int}(D_2)) \nonumber \\ &\geq 0. \label{uscxi2}
\end{align}
Using \eqref{uscxi1}, \eqref{uscxi2} and taking $\liminf_{j \to \infty}$ we get:
\begin{align}
   0 \leq \mu_2(D_2 \backslash (V_2 \cup \partial D_2)) - \mu_2(K_{1,j}) - \mu_2(O_2(\xi_{j}) \cap \text{int}(D_2)) & < \nu_{1,j}  \nonumber\\
   \implies  0 \leq \mu_2(D_2 \backslash (V_2 \cup \partial D_2)) - \limsup_{j \to \infty}\mu_2(K_{1,j}) - \limsup_{j \to \infty}\mu_2(O_2(\xi_{j}) \cap \text{int}(D_2)) & \leq \liminf_{j \to \infty}\nu_{1,j} = 0 \nonumber \\
   \implies \lim_{j \to \infty}\mu_2(O_2(\xi_{j}) \cap \text{int}(D_2)) & = 0. \label{uscxi3}
\end{align}
Since $V_2 \cap \text{int}(D_2) \neq \emptyset$ there exists $\z \in V_2 \cap \text{int}(D_2) $ such that $ \mathcal{B}_{\xi_{j}} (\z) \cap \text{int}(D_2) \subset O_2(\xi_{j}) \cap \text{int}(D_2)$ and since $ \mathcal{B}_{\xi_{j}} (\z) \cap \text{int}(D_2)$, $ O_2(\xi_{j}) \cap \text{int}(D_2) $ are open sets, we have for any $j$ that
\begin{align}
    0 < \mu_2(\mathcal{B}_{\xi_{j}} (\z) \cap \text{int}(D_2)) &\leq \mu_2( O_2(\xi_{j}) \cap \text{int}(D_2)) \nonumber \\
    \implies 0 \leq \limsup_{j \to \infty} \mu_2(\mathcal{B}_{\xi_{j}} (\z) \cap \text{int}(D_2)) &\leq  \limsup_{j \to \infty}\mu_2( O_2(\xi_{j}) \cap \text{int}(D_2)) \underbrace{=}_{\eqref{uscxi3}} \lim_{j \to \infty}\mu_2( O_2(\xi_{j}) \cap \text{int}(D_2)) \nonumber \\ & = 0 \nonumber \\
    \implies \lim_{j \to \infty} \mu_2(\mathcal{B}_{\xi_{j}} (\z) \cap \text{int}(D_2)) &= 0 \nonumber
    \implies \lim_{j \to \infty} \xi_{j} = 0 \nonumber
\end{align}
where we used $\z \in \text{int}(D_2) $ in the last step. Since we started with an arbitrary decreasing sequence $ \{\nu_{1,j}\}_j$ with $ \nu_{1,j}\downarrow 0  $, the proof is complete.
\end{proof}
\subsection{Proof of Theorem \ref{aeforgethm1}}\label{sec6proof2}
\begin{proof}
    Let $O_1, O_2, O_3$ be as in Lemma \ref{coverlemma00}. Recall that for any $(\x, \z) \in D_1  \times (D_2 \backslash V_2)$, $f(\cdot,\cdot)$ is jointly $\mathcal{C}^2$ smooth for $\mu_1$ a.e. $\x$ and hence $f$ is jointly $\mathcal{C}^2$ smooth on $D_1  \times O_1(\xi)$ for $\mu_1$ a.e. $\x$ because $ O_1(\xi) \cap O_2(\xi) = \emptyset$. In particular, since $ (\text{int}(D_1) \times \text{int}(D_2)) \backslash V$ is open in $ \mathbb{R}^n \times \mathbb{R}^d$, for every $ (\x, \z) \in  (\text{int}(D_1) \times \text{int}(D_2)) \backslash V$ there exists an open neighborhood of $(\x,\z)$ where $f$ is jointly $\mathcal{C}^2$ smooth. \RSnew{For any $\epsilon>0$ where $\epsilon \leq \frac{L}{2}\xi^2$, define $\rho_{\epsilon}:=\sqrt{\frac{2 \epsilon}{L}}$, then $f$ is jointly $\mathcal{C}^2$ smooth on $D_1  \times O_1(\rho_{\epsilon})$ for $\mu_1$ a.e. $\x$ since $ O_1(\rho_{\epsilon}) \cap O_2(\rho_{\epsilon}) = \emptyset$. Next, observe that $ \bigcup_{\z \in S_{\epsilon}(\x,\z^*, K_1) \cap K_1} \mathcal{B}_{\rho_{\epsilon}}(\z)$ is a sub-cover of $ O_1(\rho_{\epsilon}) $ that covers $ S_{\epsilon}(\x,\z^*, K_1) \cap K_1$.
    Since $ S_{\epsilon}(\x,\z^*, K_1)$ is pre-compact, $ D_2 \cap V_2 $ is compact, by the Heine-Borel theorem, finite sub-covers for $S_{\epsilon}(\x,\z^*, K_1) \cap K_1$ and $ D_2 \cap V_2$ can be extracted respectively from the covering sets $ \bigcup_{\z \in S_{\epsilon}(\x,\z^*, K_1) \cap K_1} \mathcal{B}_{\rho_{\epsilon}}(\z) , O_2(\rho_{\epsilon})$ for any $\epsilon \leq \frac{L}{2}\xi^2$. \\
    \RSnew{Let $\{ \z_j^* \}_{j=1}^{N_1}\subset S_{\epsilon}(\x,\z^*,K_1) \cap K_1$ be a maximal finite $\rho_\epsilon$-separated family 
(i.e., $\|\z_i^*-\z_j^*\|\ge \rho_\epsilon$ for $i\neq j$) so that
\[
S_{\epsilon}(\x,\z^*,K_1) \cap K_1 \subset \bigcup_{j=1}^{N_1} \mathcal{B}_{\rho_\epsilon}(\z_j^*).
\]
Such cover exists, since \(S_{\epsilon}(\x,\z^*,K_1)\cap K_1\) is pre-compact.
}\\ 
\RSnew{For any $0<\epsilon \leq \frac{L}{2}\xi^2$, let $\tilde{O}_1(\epsilon) = \bigcup_{j=1}^{N_1} \mathcal{B}_{\rho_{\epsilon}}(\z^*_j) $. Let $N(D_2, \rho_{\epsilon})$ be the cardinality of the maximal finite $\rho_{\epsilon}$ separated cover of $D_2$.  
 Then for $0< \epsilon\leq \min \{ \frac{L\tau^2}{2} , \frac{L}{2}\xi^2 \}$, using the fact that $ \tilde{O}_1(\epsilon) \subseteq D_2$ we get 
 \begin{align*}
 N_1 \times \bigg(\frac{\rho_{\epsilon}}{2}\bigg)^d  \mathrm{vol}_{\mathbb{R}^d}(\mathcal{B}_{1}(\bm 0) ) \leq    \mathrm{vol}_{\mathbb{R}^d}(\tilde{O}_1(\epsilon) ) \leq    \mathrm{vol}_{\mathbb{R}^d}(D_2 ) 
 \end{align*}
 which implies that}
\begin{align}
    N_1\le\bigg( 2\sqrt{\frac{L}{ 2\epsilon}}\bigg)^{d} \frac{\text{vol}_{\mathbb{R}^d}(D_2) \Gamma(\frac{d}{2}+1)}{\pi^{\frac{d}{2}}}.
    \label{coveringnumberbd1}
\end{align}
We therefore have that 
$$ S_{\epsilon}(\x,\z^*,K_1) \cap K_1 \subset \tilde{O}_1(\epsilon) = \bigcup_{j=1}^{N_1} \mathcal{B}_{\rho_{\epsilon}}(\z^*_j) \quad , \quad \{\z_j^*\}_{j=1}^{N_1}\subset K_1 \cap S_{\epsilon}(\x,\z^*,K_1). $$
Hence the volume of  $ S_{\epsilon}(\x,\z^*, K_1) \cap K_1$ is upper bounded by the sum of volume of  sets of the form 
$$ S_{\epsilon}(\x,\z^*,\mathcal{B}_{\rho_{\epsilon}}(\z^*_j)) = \bigg\{\z \in  \mathcal{B}_{\rho_{\epsilon}}(\z^*_j)  : \norm{\nabla f(\x; \z) - \nabla f(\x; \z^*)} \leq \epsilon \bigg\} $$
where $ \mathcal{B}_{\rho_{\epsilon}}(\z^*_j) $ is the $j-$th covering ball for $S_{\epsilon}(\x,\z^*,K_1) \cap K_1$. Recall that we already have a bound on the volume of sets of the form $ S_{\epsilon}(\x,\z^*,\mathcal{B}_{\rho_{\epsilon}}(\z^*_j))$ from \eqref{eq:thm_proof_4}. In particular, let $ \M_0(\z^*_j) = \nabla_{\z} \nabla_{\x} f(\x; \z^*_j)$ for $\mu_1$ a.e. $\x \in D_1$. Then from \eqref{eq:thm_proof_4} we have that
\begin{align}
    \mu_2 \bigg(S_{\epsilon}(\x,\z^*,\mathcal{B}_{\rho_{\epsilon}}(\z^*_j))\bigg)\le \left(\frac{8 \epsilon}{\gamma}\right)^d
         \left( \frac{\gamma^2}{16 \epsilon} {\frac{4}{L}}\right)^{r/2}  
        \hspace{0.4cm}  \mu_1 \text{ a.e.} \label{genforge008ae}
\end{align}
where now $r$ is the parameter from \textbf{A3} for the compact set $\x \times K_1$. 
Using a union bound, the covering number bound  \eqref{coveringnumberbd1}, the $\epsilon-$forging volume bound \eqref{genforge008ae} over a  ball of radius $ \sqrt{\frac{2 \epsilon}{L}}$, the fact that $ S_{\epsilon}(\x, \z^*,K_1) \cap K_1 \subset D_2$, for any $0< \epsilon \le \min\{\frac{\gamma^2}{4L}, \frac{L\tau^2}{2}, \frac{L}{2}\xi^2\}$ we have 
\begin{align}
    \mu_2\bigg(  S_{\epsilon}(\x, \z^*, K_1) \cap K_1\bigg) & \leq \sum_{j=1}^{N_1} \text{vol}_{\mathbb{R}^d}\bigg(  (S_{\epsilon}(\x, \z^*,K_1) \cap K_1)  \bigcap \mathcal{B}_{\rho_{\epsilon}}(\z^*_j) \bigg) \nonumber \\
    & \leq {N_1} \times \left(\frac{8 \epsilon}{\gamma}\right)^d
         \left( \frac{\gamma^2}{16 \epsilon} {\frac{4}{L}}\right)^{r/2}   \nonumber \\
    & \leq  16^d  \,
        \frac{\Gamma(\frac{d}{2}+1)}{\pi^{d/2}}
        \,\mathrm{vol}_{\mathbb{R}^d}(D_2)\,
        \left(\frac{L}{2 \gamma^2}\right)^{d/2}
       \left( \frac{\gamma^2}{4L } \right)^{r/2} 
        \epsilon^{\frac{d-r}{2}} \hspace{0.4cm}  \mu_1 \text{ a.e.}  \label{forgegenae00}
\end{align}}
where $r \leq d-1$ from \textbf{Assumption A3}.
\end{proof}

\section{Supporting lemmas}

\begin{lemma}\label{suplem0}
    For any compact, disjoint sets $U, V$ in $\mathbb{R}^n$ there exist a $\delta >0$ such that $ U + \mathcal{B}_{\delta/3}(\mathbf{0})$,  $ V + \mathcal{B}_{\delta/3}(\mathbf{0})$ are non-intersecting uniform open covers of $U, V$ respectively.
\end{lemma}
\begin{proof}
    Since $U, V$ are compact, disjoint and $\mathbb{R}^n$ is a Hausdorff space, we get have that $$d(U,V) := \inf \{ \norm{\u-\v} : \u \in U, \v \in V\} >0 $$ Let $d(U,V) = \delta >0 $. Consider the uniform open covers $ U + \mathcal{B}_{\delta/3}(\mathbf{0})$,  $ V + \mathcal{B}_{\delta/3}(\mathbf{0})$ of $U,V$ respectively. Then for any arbitrary $\a_1 \in  U + \mathcal{B}_{\delta/3}(\mathbf{0})$, $\a_2 \in V + \mathcal{B}_{\delta/3}(\mathbf{0})$ and $\u \in U \cap \overline{\mathcal{B}_{\delta/3}(\a_1)} $, $\v \in V \cap \overline{\mathcal{B}_{\delta/3}(\a_2)} $: 
    \begin{align*}
    \norm{\u - \v} & \leq \norm{\u - \a_1} + \norm{\a_1 - \a_2} + \norm{\a_2 - \v} \leq \delta/3 + \norm{\a_1 - \a_2} +\delta/3  \\
    \implies \inf_{a_1 , a_2} \norm{\u - \v} & \leq \inf_{a_1 , a_2}2\delta/3 + {\inf_{a_1 , a_2}\norm{\a_1 - \a_2}} = 2\delta/3 +  d(U + \mathcal{B}_{\delta/3}(\mathbf{0}),V + \mathcal{B}_{\delta/3}(\mathbf{0}))   \\
    \implies \delta =  d(U,V) \leq \inf_{a_1 , a_2} \norm{\u - \v} & \leq 2\delta/3 +  d(U + \mathcal{B}_{\delta/3}(\mathbf{0}),V + \mathcal{B}_{\delta/3}(\mathbf{0})) \\
    \implies \delta / 3 & \leq  d(U + \mathcal{B}_{\delta/3}(\mathbf{0}),V + \mathcal{B}_{\delta/3}(\mathbf{0})).
    \end{align*}
\end{proof}

\begin{lemma}\label{suplem1} \cite{alberti1994structure, federer2014geometric}
    Let $A \subset \mathbb{R}^d$ be a compact convex set. Then $\partial A$ is a $(d-1)$-dimensional rectifiable set.
\end{lemma}

\begin{lemma}\label{suplem3} \cite{lee2003smooth}
    Let $A \subset \mathbb{R}^d$ be an algebraic variety. Then $A$ has zero Lebesgue measure in $\mathbb{R}^d$.
\end{lemma}

\begin{lemma}\label{suplem4}
    Consider the block matrix 
    $ \A = \bigg[  \A_1 \hspace{0.1cm} \vert \hspace{0.1cm} \A_2 \hspace{0.1cm} \vert \hspace{0.1cm} \cdots  \hspace{0.1cm} \vert \hspace{0.1cm} \A_m \bigg]$
    where $\A_i \in \mathbb{R}^{p \times q_i}$ for $i \in \{1, \cdots,m\}$. Then 
    $
        \norm{\A} \leq \sqrt{\sum_{i=1}^m \norm{\A_i}^2} .$
\end{lemma}
\begin{proof}
    Let $\v \in S^{p-1}$ be arbitrary. Since $\A$ is block matrix, $\A \A^T =  \sum_{i=1}^m \A_i \A_i^T$. Then,    
   $\norm{\A}^2 = \sup_{\v \in S^{p-1}}   \langle \v, \A \A^T\v \rangle  
   =  \sup_{\v \in S^{p-1}}\sum_{i=1}^m\langle \v, \A_i \A_i^T\v \rangle    
   \leq    \sum_{i=1}^m   \sup_{\v \in S^{p-1}} \langle \v, \A_i \A_i^T\v \rangle = \sum_{i=1}^m \norm{\A_i}^2.$ 
   Thus $\norm{\A}  \leq \sqrt{\sum_{i=1}^m \norm{\A_i}^2}$,
    which completes the proof.
\end{proof}

\section{Applicability of Assumption A1}\label{reluassumptionsec}
We now show that  \textbf{Assumption A1} is satisfied for loss functions arising in learning neural nets. Consider the empirical least squares loss function used for training an $M$ layer neural network,
\begin{align}
    f_{\texttt{ERM}}\bigg(\{\v,\W_0,\W_1,\cdots,\W_M \};\Z\bigg) &= \frac{1}{N} \sum_{j=1}^N \bigg(\v^T \rho(\W_M \hspace{0.1cm}\rho(\cdots\rho( \W_1 \rho(\W_0 \z_j))\cdots)) -y_j \bigg)^2. \label{ermforge}
\end{align}
Here, $ \{\v,\W_0,\W_1,\cdots,\W_M \}$ corresponds to the model variable $\x$ while $\Z$ is the dataset $\{\z_j\}_{j=1}^N$, 
 and $\rho$ is an activation function 
 as before. For smooth activations, \textbf{Assumption A1} holds trivially by composition of smooth functions. In addition, if $\rho \in \mathcal{C}^3(\mathbb{R})$ then \textbf{Assumption A2} holds as well. We now focus on the case when $\rho$ is leaky \texttt{ReLU} and therefore non-smooth. Formally,
\begin{align*}
    \rho(x) = \begin{cases}
        x \quad ; \quad x > 0 \\
        \alpha x \quad ; \quad x \leq 0
    \end{cases}
\end{align*}
where $\alpha \in (0,1)$ and usually $\alpha \ll 1$. We will write $ \rho(\langle \W, \y\rangle) = \langle \W, \y\rangle_{\alpha} $ and define $ \mathbb{R}_*^m \equiv \mathbb{R}^m \backslash \mathbf{0}$, $  \mathbb{R}_{**}^m \equiv \mathbb{R}^m \backslash \bigcup_{i=1}^m  \text{span}\{ \{\e_1, \cdots, \e_m \} \backslash \e_i \}$ where $\e_i$ is the $i-$th canonical basis vector of $ \mathbb{R}^m$.

\subsection{Almost everywhere smoothness of $f_{\texttt{ERM}} $ for leaky \texttt{ReLU} activation}

\subsubsection{Preliminaries}
Without loss of generality let us consider an individual summand on the right hand side of \eqref{ermforge}:
\begin{align}
    f\bigg(\{\v,\W_0,\W_1,\cdots,\W_M \};\z\bigg) = \bigg(\v^T \rho(\W_M \hspace{0.1cm}\rho(\cdots\rho( \W_1 \rho(\W_0 \z))\cdots)) -y \bigg)^2 . \label{ermforge1}
\end{align}
Then, due to the chain rule of derivatives, it suffices to show a.e. $\mathcal{C}^3$ smoothness of $\rho$ with respect to its arguments at every composition step. Suppose $ \W_i \in \mathbb{R}^{n_i \times n_{i-1}}$ for $i \in \{1,\cdots,M\}$, $n_{-1} = d$ and $\W_0 \in \mathbb{R}^{n_0 \times d}$.
We note that $f$ is $\mathcal{C}^{\infty}$ smooth in $\v \in \mathbb{R}^{n_{M}}$ since $f$ is the composition of square function and an affine function of $\v$. For any $i \geq 0$, using \eqref{ermforge1}, we define
\begin{align}
  \u_{i+1} &=  \rho(\W_{i}\cdots\rho( \W_1 \rho(\W_0 \z))\cdots)   
  \implies \u_{i+1} = \rho(\W_i \u_i) \hspace{0.3cm} \forall \hspace{0.2cm} i \geq 0, \label{reluinduct3}
\end{align}
where $\u_i \in \mathbb{R}^{n_{i-1}}$ and $\u_0 =\z \in \mathbb{R}^d$.
For any $i \geq 0$ let $U_i \subseteq  \mathbb{R}^{n_{i-1}} $ be the admissible set of $\u_i$, which we will specify later.
Then 
$$       [\rho(\W_i \u_i)]_j = \langle [\W_i]_j, \u_i \rangle_{\alpha} $$
where $ [\W_i]_j$ is the $j-th$ row vector of $\W_i$. Then $\rho$ is $\mathcal{C}^{\infty}$ smooth on the open set $\mathcal{R}_i$ given by 
\begin{align}
 \mathcal{R}_i &=    \bigg\{ (\W_i,\u_i) \in \mathbb{R}^{n_i \times n_{i-1}} \times  U_i: \langle [\W_i]_j, \u_i \rangle \neq 0 \hspace{0.2cm} \forall j \in \{1, \cdots, n_i\}   \bigg\} \nonumber \\
 &=  \bigg\{ (\W_i,\u_i) \in \mathbb{R}^{n_i \times n_{i-1}} \times  U_i  \bigg\} \backslash \overline{\mathcal{P}_i } \label{reluinduct1}
\end{align}
where
\begin{align}
     \mathcal{P}_i = \bigcup_{j=1}^{n_i} \bigg\{ (\W_i,\u_i) \in \mathbb{R}^{n_i \times n_{i-1}} \times  U_i: \langle [\W_i]_j, \u_i \rangle = 0   \bigg\} \label{pidefn}
\end{align}
and $ \overline{\mathcal{P}_i}$ is the closure of the set $ \mathcal{P}_i $. Further, if $U_i$ is open in $\mathbb{R}^{n_{i-1}}$ then $\mathcal{R}_i$ is an open set in $\mathbb{R}^{n_i \times n_{i-1}} \times U_i $. Observe that on  $\mathcal{R}_i$, the function $(\W_i, \u_i) \mapsto \rho(\W_i\u_i)$ is differentiable everywhere by the definition of $\rho$. Using the definition of the set $\mathcal{R}_i$ for any $i \geq 0$, we define $U_i$ recursively via
\begin{align}
    U_{i+1} = \rho(\mathcal{R}_{i}) \label{uidef1}
\end{align}
with $U_0 \cong \mathbb{R}^d$.
Equivalently, $U_{i+1}$ is the image of $\mathcal{R}_{i}$ under  $\rho$.

\begin{lemma}\label{algebraicvarlem1}
    The following hold for any $i>0$:
\begin{enumerate}
    \item The set $U_i \cong \mathbb{R}^{n_{i-1}}_{**}$ and hence $U_i$ is open in $ \mathbb{R}^{n_{i-1}}$, $U_i$ has full Lebesgue measure in $ \mathbb{R}^{n_{i-1}}$.
    \item The set $\mathcal{P}_i$ is a subset of the union of finitely many algebraic varieties in $ \mathbb{R}^{n_i \times n_{i-1}} \times \mathbb{R}^{n_{i-1}} $ for any $i \geq 0$ \footnote{Here, the set $ \bigg\{ (\W_i,\u_i) \in \mathbb{R}^{n_i \times n_{i-1}} \times  U_i: \langle [\W_i]_j, \u_i \rangle = 0   \bigg\} $ is a subset of an algebraic variety since both $[\W_i]_j, \u_i $ are variables in the equation $ \langle [\W_i]_j, \u_i \rangle = 0$.} and therefore has zero Lebesgue measure.
\end{enumerate}
\end{lemma}

 \begin{proof}
     We proceed with a proof by induction.
     
    \paragraph{Base Case. }For $i=0$ we have $\rho$ acting on $\W_0\z$  so $\u_1 = \rho(\W_0 \z)$ where $ \z \in U_0 \cong \mathbb{R}^d $, $\W_0 \in \mathbb{R}^{n_0 \times d}$. Hence $\rho$ is $\mathcal{C}^{\infty}$ smooth on 
    \begin{align}
        \mathcal{R}_0 &= \bigg\{ (\W_0,\z) \in \mathbb{R}^{n_0 \times d} \times  \mathbb{R}^d  \bigg\} \backslash  \overline{\mathcal{P}_0} \nonumber\\
        &= \bigg\{ (\W_0,\z) \in \mathbb{R}^{n_0 \times d} \times  \mathbb{R}^d  \bigg\} \backslash \overline{\bigcup_{j=1}^{n_0}  \bigg\{ (\W_0,\z) \in \mathbb{R}^{n_0 \times d} \times \mathbb{R}^d: \langle [\W_0]_j, \z \rangle = 0   \bigg\}} \nonumber\\
        &= \bigg\{ (\W_0,\z) \in \mathbb{R}^{n_0 \times d} \times  \mathbb{R}^d  \bigg\} \backslash \bigcup_{j=1}^{n_0}  \bigg\{ (\W_0,\z) \in \mathbb{R}^{n_0 \times d} \times \mathbb{R}^d: \langle [\W_0]_j, \z \rangle = 0   \bigg\}, \nonumber
    \end{align}
    and $\rho : \mathcal{R}_0 \mapsto \mathbb{R}^{n_0}$. Observe that 
    $ \mathcal{P}_0 = \bigcup_{j=1}^{n_0}  \bigg\{ (\W_0,\z) \in \mathbb{R}^{n_0 \times d} \times \mathbb{R}^d: \langle [\W_0]_j, \z \rangle = 0   \bigg\} $
    is a finite union of algebraic varieties in $\mathbb{R}^{n_0 \times d} \times \mathbb{R}^d $ hence of zero Lebesgue measure (Lemma \ref{suplem3}), is closed in $ \mathbb{R}^{n_0 \times d} \times \mathbb{R}^d$ so $ \mathcal{R}_0$ is open in $ \mathbb{R}^{n_0 \times d} \times \mathbb{R}^d$ and of full Lebesgue measure. Since $ U_1 = \rho(\mathcal{R}_0) $ then $ U_1 \subset \mathbb{R}_{*}^{n_0}$, in particular $U_1 \cong \mathbb{R}_{**}^{n_0}$ because the image of  $ \bigg\{ (\W_0,\z) \in \mathbb{R}^{n_0 \times d} \times  \mathbb{R}^d  \bigg\}$ under $\rho$ is $\mathbb{R}^{n_0} $ and the image of  $ \overline{\mathcal{P}_0}$ under $\rho$ is $ \bigcup_{j=1}^{n_0} \{ \u \in \mathbb{R}^{n_0} : [\u]_j =0 \} $ from the definition of $\rho$. Hence, $U_1$ has full Lebesgue measure in $ \mathbb{R}^{n_0}$. Thus, for the base case our hypothesis holds true.
    \paragraph{Induction. } Suppose 
    $U_i \cong \mathbb{R}^{n_{i-1}}_{**}$, $U_i$ has full Lebesgue measure in $ \mathbb{R}^{n_{i-1}}$ and  $ \mathcal{P}_i $ is a subset of the union of finitely many algebraic varieties in $ \mathbb{R}^{n_i \times n_{i-1}} \times \mathbb{R}^{n_{i-1}} $. Then for $i+1$,  $U_{i+1} = \rho(\mathcal{R}_i)$ where $\mathcal{R}_i$ is as in \eqref{reluinduct1}. The image of open set $\bigg\{ (\W_i,\u_i) \in \mathbb{R}^{n_i \times n_{i-1}} \times \mathbb{R}^{n_{i-1}}_{**} \bigg\} $ under $\rho$ is $\mathbb{R}^{n_i} $ since, for $\x \in \mathbb{R}^{ n_{i-1}}, \y \in \mathbb{R}^{ n_{i-1}}_{**}$, the map $g: \mathbb{R}^{ n_{i-1}} \times \mathbb{R}^{n_{i-1}}_{**} \to \mathbb{R}$, where $g(\x,\y) = \langle \x,\y\rangle_{\alpha}$, is surjective. Next, 
    \begin{align}
        \overline{\mathcal{P}_i} &= \overline{\bigcup_{j=1}^{n_i} \bigg\{ (\W_i,\u_i) \in \mathbb{R}^{n_i \times n_{i-1}} \times  U_i: \langle [\W_i]_j, \u_i \rangle = 0   \bigg\}} \nonumber \\
        &= \bigcup_{j=1}^{n_i} \overline{\bigg\{ (\W_i,\u_i) \in \mathbb{R}^{n_i \times n_{i-1}} \times \mathbb{R}^{n_{i-1}}_{**}: \langle [\W_i]_j, \u_i \rangle = 0   \bigg\}} \nonumber \\
        & = \bigcup_{j=1}^{n_i} \bigg\{ (\W_i,\u_i) \in \mathbb{R}^{n_i \times n_{i-1}} \times  \mathbb{R}^{n_{i-1}}: \langle [\W_i]_j, \u_i \rangle = 0   \bigg\}. \nonumber
    \end{align}
    Then the image of $\overline{\mathcal{P}_i}$ under $\rho$ is $\bigcup_{j=1}^{n_i} \{ \u \in \mathbb{R}^{n_i} : \u_j =0 \} $. Hence $U_{i+1} = \rho(\mathcal{R}_i) \cong \mathbb{R}^{n_i}_{**}$, so $U_{i+1}$ is open in $ \mathbb{R}^{n_i}$ and has full Lebesgue measure in $ \mathbb{R}^{n_i}$. It only remains to show that $\mathcal{P}_{i+1}$ is a subset of the union of finitely many algebraic varieties. Recall from \eqref{reluinduct1} that
    \begin{align}
        \mathcal{R}_{i+1} 
 &=  \bigg\{ (\W_{i+1},\u_{i+1}) \in \mathbb{R}^{n_{i+1} \times n_i} \times  \mathbb{R}^{n_i}_{**}  \bigg\} \backslash \overline{\mathcal{P}_{i+1}}, \nonumber
 \end{align}
 where  
 \(
 \mathcal{P}_{i+1}  = \bigcup_{j=1}^{n_{i+1}}  \bigg\{ (\W_{i+1},\u_{i+1}) \in \mathbb{R}^{n_{i+1} \times n_i} \times  \mathbb{R}^{n_i}_{**}: \langle [\W_{i+1}]_j, \u_{i+1} \rangle = 0   \bigg\}
   \)
    is therefore  a subset of the union of finitely many algebraic varieties in $ \mathbb{R}^{n_{i+1} \times n_i} \times \mathbb{R}^{n_i} $, hence of 0 measure.
 \end{proof}
In the following lemma we treat $\rho$ as a function from $ \mathbb{R}^{n_{i-1}} \times \mathbb{R}^{n_i \times n_{i-1}} \to \mathbb{R}^{n_{i}}$ with $\rho(\u, \V) = \langle \V, \u\rangle_{\alpha}$.
\begin{lemma}\label{algebraicvarlem2}
    For any $i>0$ let $X_i \cong \mathbb{R}^{n_i \times n_{i-1}}$, $Y_i \cong \mathbb{R}^{n_{i-1}}$ where $Y_i \supset U_i$ and $U_i$ is as in \eqref{uidef1}. Consider the  Cartesian product map
    $$ \rho \times \mathrm{id} : Y_{i-1} \times X_{i-1} \times X_i \to  Y_i \times X_i$$ where $ \rho : Y_{i-1} \times X_{i-1} \to Y_i$. Let $A_i$ be any subset of a finite union of algebraic varieties in $Y_i \times X_i$. Then the pre-image of $A_i$ under  $ \rho \times \mathrm{id}$, namely $ \bigg(\rho \times \mathrm{id}\bigg)^{-1}(A_i)$ is a subset of a finite union of algebraic varieties in $Y_{i-1} \times X_{i-1} \times X_i$.
\end{lemma}
\begin{proof}
    From the definition of an algebraic variety in $Y_i \times X_i$, we have that
$$ A_i \subseteq \bigcup_{j=1}^t \bigg\{  (\y,\X) \in Y_i \times X_i : p_{k_j,j}(\X,\y) = \mathbf{0} \bigg\} $$
where $p_{k_j,j}(\cdot)$ is a non-trivial degree $k_j$ vector polynomial function and $t$ is finite. That is, at least one coefficient of the polynomial in at least one entry of the vector $p_{k_j,j}(\X,\y)$ is non-zero. Define $B_i:=\bigg(\rho \times \mathrm{id}\bigg)^{-1}(A_i) \subset Y_{i-1} \times X_{i-1} \times X_i$. Noting that $\y = \rho( \u,\V) $ for $\V \in X_{i-1}$, $\u \in Y_{i-1}$ we get\footnote{Here $ [\V]_l$ denotes the vector corresponding to the $l-$th row of $\V$.}:
\begin{align}
   \bigg( \rho \times \mathrm{id}\bigg)^{-1}( A_i) &\subseteq \bigcup_{j=1}^t \bigg\{  (\u,\V,\X) \in Y_{i-1} \times X_{i-1} \times X_i : p_{k_j,j}(\X, \rho( \u,\V)) = \mathbf{0} \bigg\}  \nonumber \\
    \iff  B_i &\subseteq \bigcup_{j=1}^t \bigg\{  (\u,\V,\X) \in Y_{i-1} \times X_{i-1} \times X_i : p_{k_j,j}(\X,\langle\V ,\u \rangle_{\alpha}) = \mathbf{0} \bigg\} \nonumber \\
    \iff  B_i &\subseteq \bigcup_{j=1}^t\bigg( \bigg\{  (\u,\V,\X) \in Y_{i-1} \times X_{i-1} \times X_i : p_{k_j,j}(\X,\langle\V ,\u \rangle_{\alpha} ) = \mathbf{0}; \langle\V ,\u \rangle_{\alpha} \in \mathbb{R}^{n_{i-1}}_{**}  \bigg\} \nonumber \\ & \hspace{1.4cm} \bigcup_{l=1}^{n_{i-1}} \bigg\{  (\u,\V,\X) \in Y_{i-1} \times X_{i-1} \times X_i : p_{k_j,j}(\X,\langle\V ,\u \rangle_{\alpha}) = \mathbf{0}; \langle [\V]_l ,\u \rangle =0 \bigg\}  \bigg)  \nonumber
\end{align}
Further relaxing the last inclusion yields:
\begin{align}
    B_i &\subseteq \bigg( \bigcup_{j=1}^t \underbrace{\bigg\{  (\u,\V,\X) \in Y_{i-1} \times X_{i-1} \times X_i : p_{k_j,j}(\X,\langle\V ,\u \rangle_{\alpha}) = \mathbf{0}  \bigg\}}_{F_j} \nonumber \\ & \hspace{1.4cm} \bigcup_{l=1}^{n_{i-1}} \underbrace{\bigg\{  (\u,\V,\X) \in Y_{i-1} \times X_{i-1} \times X_i :  \langle [\V]_l ,\u \rangle =0 \bigg\} }_{G_l} \bigg) \nonumber
\end{align}
where for any $j, l$ the sets $F_j, G_l$ are algebraic varieties in $ Y_{i-1} \times X_{i-1} \times X_i$. We have for any $j$,
\begin{align*}
    \bigg\{  (\u,\V,\X) \in Y_{i-1} \times X_{i-1} \times X_i : p_{k_j,j}(\X,\langle\V ,\u \rangle_{\alpha}) = \mathbf{0}  \bigg\} & \nonumber \\ & \hspace{-5cm} \subseteq \bigcup_{q=1}^{2^{n_{i-1}}}  \underbrace{\bigg\{  (\u,\V,\X) \in Y_{i-1} \times X_{i-1} \times X_i : p_{k_j,j}(\X, \langle\V ,\u \rangle \odot \boldsymbol{\alpha}_q ) = \mathbf{0}  \bigg\}}_{H_q} \nonumber
\end{align*}
where $\boldsymbol{\alpha}_q \in \mathbb{R}^{n_{i-1}} $ is a vector of $1$'s and $\alpha$'s with $q$ indexing the $2^{{n_{i-1}}}$ such vector possibilities. $H_q$ is an algebraic variety for any permutation index $q$ provided $\alpha \neq 0$ (see \RSnew{\Cref{relucounterexrem1}} below). 
Hence $B_i = \bigg( \rho \times \mathrm{id}\bigg)^{-1}( A_i) $ is a subset of finite union of algebraic varieties in $ Y_{i-1} \times X_{i-1} \times X_i$. 
\RSnew{We note that Lemma \ref{algebraicvarlem2} can be easily generalized to the map $ \rho \times  \prod_{j=i+1}^{M}\mathrm{id}$ for arbitrary $M$ where
   $$ \rho \times  \prod_{j=i+1}^{M}\mathrm{id}   : Y_i \times X_i \times X_{i+1} \cdots \times X_M  \xrightarrow{}  Y_{i+1} \times X_{i+1} \times \cdots \times X_M \, . $$
   Then if $A \subset Y_{i+1} \times X_{i+1} \times \cdots \times X_M $ is a subset of finite union of algebraic varieties in $Y_{i+1} \times X_{i+1} \times \cdots \times X_M$, we will have that $ (\rho \times  \prod_{j=i+1}^{M}\mathrm{id})^{-1}(A) $ is a subset of finite union of algebraic varieties in $Y_i \times X_i \times X_{i+1} \times \cdots \times X_M$. We omit the generalization proof for brevity.
}
\end{proof}
\begin{remark}[On \texttt{ReLU} activations.]\label{relucounterexrem1}
    Note that when $\alpha = 0$, there exists a $q$ for which $ \boldsymbol{\alpha}_q  = \mathbf{0}$. In that case  $H_q = {\bigg\{  (\u,\V,\X) \in Y_{i-1} \times X_{i-1} \times X_i : p_{k_j,j}(\X, \langle\V ,\u \rangle \odot \mathbf{0} ) = \mathbf{0}  \bigg\}}. $ If the polynomial  $p_{k_j,j}$ is homogeneous then $ H_q = {\bigg\{  (\u,\V,\X) \in Y_{i-1} \times X_{i-1} \times X_i   \bigg\}} $, which is no longer an algebraic variety but is the entire set $ Y_{i-1} \times X_{i-1} \times X_i$ and thus has full Lebesgue measure. Note that $\alpha=0$ implies the \texttt{ReLU} activation function and Lemma \ref{algebraicvarlem2} does not hold for \texttt{ReLU} activation. It is easy to construct a simple two layer example with \texttt{ReLU} activation where the set of non-smoothness has positive measure. For instance, when $\rho$ is the \texttt{ReLU} activation, the function $\rho(\W_1 \rho (\W_0\z))$ is not smooth on the set $ \{ (\W_1,\W_0,\z) : \langle [\W_0]_j ,\z\rangle \leq 0 \quad \forall \ j\}$ which has a positive Lebesgue measure. 
\end{remark}
\begin{theorem}\label{aesmoothlem1}
    The function $f : \mathbb{R}^{n_{M}} \times \mathbb{R}^{n_{M} \times n_{M-1}} \times\cdots \times \mathbb{R}^{n_i \times n_{i-1}}  \times \cdots \times \mathbb{R}^{n_0 \times d} \times \mathbb{R}^d  \to \mathbb{R}$ defined in \eqref{ermforge1} for $\alpha>0$ is $\mathcal{C}^{\infty}$ smooth a.e. on its domain.
\end{theorem}
\begin{proof}
  From \eqref{ermforge1}, $f$ acts on $\{\v,\W_0,\W_1,\cdots,\W_M,\z \}$ 
where $\v \in \mathbb{R}^{n_{M}}$, $ \W_i  \in \mathbb{R}^{n_i \times n_{i-1}} $ for all $ 0 \leq i \leq M$, $n_{-1} = d$ and $ \z \in \mathbb{R}^d$. Let $X_i \cong \mathbb{R}^{n_i \times n_{i-1}}$ for all $0 \leq i \leq M$, let $Y_i \cong \mathbb{R}^{ n_{i-1}}$ for all $1 \leq i \leq M$ and $Y_0 \cong \mathbb{R}^d$. Hence, for any $i$ we have $U_i \subseteq Y_i$ where $\{U_i\}_{i=0}^M$ are the admissible sets defined in \eqref{uidef1} with $U_0 \cong \mathbb{R}^d$. Then for all $0 \leq i \leq M$, the  leaky \texttt{ReLU} activation function 
$$ \rho :   Y_i \times  X_i  \to Y_{i+1}$$
with $\rho(\u, \V) = \langle \V, \u\rangle_{\alpha}$ for $\V\in  X_i$, $\u \in Y_i$.
 For $Y_{M+1} \cong  \mathbb{R}^{n_{M}}$ where $\v \in Y_{M+1}$, consider the Cartesian product of maps for any $0 \leq i < M$
    $$ \rho \times \bigg( \prod_{j=i+1}^{M}\mathrm{id} \bigg) \times \mathrm{id}  : Y_i \times X_i \times X_{i+1} \cdots \times X_M \times Y_{M+1} \xrightarrow{}  Y_{i+1} \times X_{i+1} \times \cdots \times X_M \times Y_{M+1}.  $$ 
    Since the last identity map takes $Y_{M+1}$ to itself for all $i$, we can factor it out  to get, 
    for any $0 \leq i < M$,
     $$ \rho \times  \prod_{j=i+1}^{M}\mathrm{id}   : Y_i \times X_i \times X_{i+1} \cdots \times X_M  \xrightarrow{}  Y_{i+1} \times X_{i+1} \times \cdots \times X_M .  $$
  Next, consider the  chain of Cartesian product of maps
    \begin{align}
        Y_0 \times X_0 \times X_1 \times \cdots \times X_M &\xrightarrow{ \rho \times  \prod_{j=1}^M\mathrm{id} } Y_1 \times X_1 \times \cdots \times X_M \xrightarrow{ \rho \times  \prod_{j=2}^M\mathrm{id} }  Y_2 \times X_2  \times \cdots \times X_M \xrightarrow{ \rho \times  \prod_{j=3}^M\mathrm{id} } \cdots \nonumber \\ & \hspace{-2cm} \cdots \xrightarrow{\rho \times  \prod_{j=i}^M\mathrm{id} }Y_i \times X_i \times \cdots \times X_M \xrightarrow{ \rho \times  \prod_{j=i+1}^M\mathrm{id} } Y_{i+1} \times X_{i+1} \times \cdots \times X_M \xrightarrow{ \rho \times  \prod_{j=i+2}^M\mathrm{id} } \cdots \nonumber \\ & \cdots \xrightarrow{\rho \times \mathrm{id}} Y_M \times X_M \xrightarrow{\rho} \mathbb{R}^{n_{M}} .\label{chaina1}
    \end{align}
For any given $i>0$ we write a triple sequence with the Cartesian product of maps:
\begin{align}
   Y_{i-1} \times X_{i-1} \times \cdots \times X_M  \xrightarrow{\rho \times  \prod_{j=i}^M\mathrm{id} }Y_i \times X_i \times \cdots \times X_M \xrightarrow{ \rho \times  \prod_{j=i+1}^M\mathrm{id} } Y_{i+1} \times X_{i+1} \times \cdots \times X_M.  \label{shortchain1}
\end{align}
We know that on the set $ \mathcal{P}_i \subset Y_i \times X_i$ where $ \mathcal{P}_i$ is defined in \eqref{pidefn}, the map $\rho : Y_i \times X_i  \to Y_{i+1}$ is non-smooth. Hence, the second Cartesian product of maps given by $\bigg(\rho \times  \prod_{j=i+1}^M\mathrm{id}\bigg) $ in  \eqref{shortchain1} is non-smooth on the product set $ \mathcal{P}_i \times X_{i+1} \times \cdots \times X_M$. For $\rho \times \mathrm{id}: Y_{i-1} \times X_{i-1} \times X_i \to Y_{i} \times X_i$, let
\begin{align}
   \bigg(\rho \times \mathrm{id}\bigg)^{-1} (\overline{\mathcal{P}_i} ) \times  \bigg(\prod_{j=i+1}^M\mathrm{id}\bigg)^{-1}(  X_{i+1} \times \cdots \times X_M)= \bigg(\rho \times  \prod_{j=i}^M\mathrm{id}\bigg)^{-1}( \overline{\mathcal{P}_i}\times X_{i+1} \times \cdots \times X_M) \label{shortchain2}
\end{align} 
be the pre-image of $\overline{ \mathcal{P}_i} \times X_{i+1} \times \cdots \times X_M$ in the set $ Y_{i-1} \times X_{i-1} \times X_i \times \cdots \times X_M $ where the above equality holds by the bijection of identity maps.
Next, factoring out the Cartesian product $ \prod_{j=i+1}^M\mathrm{id}$ from \eqref{shortchain1} yields the  triple sequence
\begin{align}
   Y_{i-1} \times X_{i-1} \times X_i   \xrightarrow{\rho \times \mathrm{id}  }Y_i \times X_i \xrightarrow{ \rho } Y_{i+1} .  \label{shortchain3}
\end{align}
Recall from Lemma \ref{algebraicvarlem1} that for any $i>0$ , $U_i \cong \mathbb{R}^{n_{i-1}}_{**}$ and hence $ U_i = Y_i \backslash E_i$ where $E_i$ is the union of all $n_{i-1} -1$ dimensional canonical hyperplanes of $ \mathbb{R}^{n_{i-1}}$. As this  is a finite union of algebraic varieties,  $E_i$ has zero Lebesgue measure in $Y_i \cong \mathbb{R}^{n_{i-1}}$. Also, recall from Lemma \ref{algebraicvarlem1} that the sets $\overline{\mathcal{P}_i}$ are subsets of algebraic varieties in $Y_i \times X_i$ and therefore have zero Lebesgue measure in $Y_i \times X_i$. Then the sets $\overline{\mathcal{P}_i}$ have zero Lebesgue measure in $U_i \times X_i$ since $ U_i = Y_i \backslash E_i$ and $E_i$ has zero Lebesgue measure in $Y_i$.
For certain projection maps $\pi : Y_i \to U_i$ and $ \gamma : U_i \times X_i \to \mathcal{R}_i$ for any $i$, where the set $\mathcal{R}_i$ is defined from \eqref{reluinduct1} with $\mathcal{R}_i = ( U_i \times X_i )\backslash \overline{\mathcal{P}_i} $ , consider the following diagram of maps:
    \[
\begin{tikzcd}
    & Y_{i-1} \times X_{i-1} \times X_i \arrow[d, "\pi \times \mathrm{id} \times \mathrm{id} "] \arrow[r, "\rho \times \mathrm{id}" ] & Y_i \times X_i \arrow[d, "\pi \times \mathrm{id}"] \arrow[r, "\rho"] & Y_{i+1} \arrow[d, "\pi"]  \\
  & U_{i-1} \times X_{i-1} \times X_i \arrow[d, "\gamma \times \mathrm{id} "] \arrow[r, "\rho \times \mathrm{id}"] & U_i \times X_i \arrow[d, "\gamma "] \arrow[r, "\rho"] & U_{i+1} \arrow[d, "\mathrm{id} "]\\
   & \mathcal{R}_{i-1} \times X_i  \arrow[d, "\cong "] \arrow[r, "\rho \times \mathrm{id}"] & \mathcal{R}_i  \arrow[d, "\cong "] \arrow[r, "\rho"] & U_{i+1}  \arrow[d, "\cong "] \\
   & \bigg(\bigg((Y_{i-1} \backslash E_{i-1})\times X_{i-1}\bigg)\backslash \overline{\mathcal{P}_{i-1}}\bigg) \times X_i  \arrow[r, "\rho \times \mathrm{id}"] & \bigg( (Y_{i} \backslash E_{i}) \times X_i \bigg)\backslash \overline{\mathcal{P}_{i}}  \arrow[r, "\rho"] & Y_{i+1}\backslash E_{i+1}
\end{tikzcd}
\]

Then in the bottom most row of the above diagram, the map $\rho$ is $\mathcal{C}^{\infty}$ smooth a.e. on $Y_i \times X_i$ due to the fact that the sets $ (E_i \cap Y_i) \times X_i$, $ \overline{\mathcal{P}_{i}}$ are subsets of finite unions of algebraic varieties in $Y_i \times X_i$ and hence these sets have zero Lebesgue measure (Lemma \ref{suplem3}). 
Similarly, the map $\rho \times \mathrm{id}$ is $\mathcal{C}^{\infty}$ smooth a.e. on $Y_{i-1} \times X_{i-1} \times X_{i}$ due to the fact that the sets $ (E_{i-1} \cap Y_{i-1}) \times X_{i-1} \times X_i$, $ \overline{\mathcal{P}_{i-1}} \times X_i$ are subsets of finite union of algebraic varieties in $Y_i \times X_i$ and hence have zero Lebesgue measure. 
Moreover, the composition $  \rho\circ(\rho \times \mathrm{id}) : Y_{i-1} \times X_{i-1} \times X_i \to  Y_{i+1} $ is non-smooth on the sets $ (E_{i-1} \cap Y_{i-1}) \times X_{i-1} \times X_i$, $ \overline{\mathcal{P}_{i-1}} \times X_i$ and also on the sets $(\rho \times \mathrm{id})^{-1} ((E_i \cap Y_i) \times X_i) $, $ (\rho \times \mathrm{id})^{-1} (\overline{\mathcal{P}_{i}})$ which are pre-images of the sets $ (E_i \cap Y_i) \times X_i$, $ \overline{\mathcal{P}_{i}}$ under $\rho \times \mathrm{id}$. But since the sets $ (E_i \cap Y_i) \times X_i$, $ \overline{\mathcal{P}_{i}}$ are subsets of finite union of algebraic varieties in $Y_i \times X_i$, their pre-images $(\rho \times \mathrm{id})^{-1} ((E_i \cap Y_i) \times X_i) $, $ (\rho \times \mathrm{id})^{-1} (\overline{\mathcal{P}_{i}})$ are also subsets of finite union of algebraic varieties in $Y_{i-1} \times X_{i-1} \times X_i$ from Lemma \ref{algebraicvarlem2} and thus have zero Lebesgue measure. Hence, the non-smooth support of the composition $  \rho\circ(\rho \times \mathrm{id}) $ in $ Y_{i-1} \times X_{i-1} \times X_i $ is a subset of a finite union of algebraic varieties in $Y_{i-1} \times X_{i-1} \times X_i$ which has zero Lebesgue measure. Hence, the composition $  \rho\circ(\rho \times \mathrm{id}) $ is $\mathcal{C}^{\infty}$ smooth a.e. on $ Y_{i-1} \times X_{i-1} \times X_i $ with the set of non-smoothness contained in a finite union of algebraic varieties. Since $i$ was arbitrary, for any $i$ and using the complete chain \eqref{chaina1} we can take the pre-images of these non-smooth supports recursively up to the set $ Y_0 \times X_0 \times X_1 \times \cdots \times X_M$. Then by recursively applying \RSnew{a generalized version of} Lemma \ref{algebraicvarlem2} we get that all such pre-images will be a subset of finite union of algebraic varieties in $ Y_0 \times X_0 \times X_1 \times \cdots \times X_M$. Hence the composite map from the complete chain \eqref{chaina1} given by 
$$ Y_0 \times X_0 \times X_1 \times \cdots \times X_M  \xrightarrow[]{ \rho\circ (\rho \times \mathrm{id}) \circ \cdots\circ(\rho \times  \prod_{j=2}^M\mathrm{id})\circ(\rho \times  \prod_{j=1}^M\mathrm{id})} \mathbb{R}^{n_{M}} $$
is $\mathcal{C}^{\infty}$ smooth a.e. on $Y_0 \times X_0 \times X_1 \times \cdots \times X_M $.
 Applying chain rule  to $f\bigg(\{\v,\W_0,\W_1,\cdots,\W_M \};\z\bigg)$ from \eqref{ermforge1} then yields that $f \in \mathcal{C}^{\infty} \hspace{0.2cm} a.e. \hspace{0.2cm} \text{on } \mathbb{R}^{n_{M}} \times \mathbb{R}^{n_{M} \times n_{M-1}} \times\cdots \times \mathbb{R}^{n_i \times n_{i-1}}  \times \cdots \times \mathbb{R}^{n_0 \times d} \times \mathbb{R}^d  \hspace{0.2cm}.$ 
\end{proof}

Thus, \textbf{Assumptions A1 and A2} hold for the loss function in \eqref{ermforge}. Note that the set of points of non-smoothness, denoted by $A$, within the domain $Y_0 \times X_0 \times X_1 \times \cdots \times X_M$ need not be closed. However, since we have shown that $A$ is contained in the union of finitely many algebraic varieties in $Y_0 \times X_0 \times X_1 \times \cdots \times X_M$, we may instead take its closure $\bar{A}$ as the set of non-smoothness. The closure $\bar{A}$ remains a subset of a finite union of algebraic varieties, and hence \textbf{Assumption A1} is satisfied.

\section{Applicability of \textbf{Assumption A3.}}\label{degenerateassumptionsec}
We verify that \textbf{Assumption A3} holds in standard models.  
\paragraph{Linear regression.}  
For $f(\x;(\z,y)) = \tfrac{1}{2}(\x^T\z-y)^2$, the mixed derivative is
\[
\nabla_{(\z,y)} \nabla_{\x} f(\x;(\z,y)) 
= (\x^T \z - y)\begin{bmatrix}\mathbf{I}_{d\times d} & \mathbf{0_{d\times 1}}\end{bmatrix} 
+ \z \begin{bmatrix}\x \\ -1\end{bmatrix}^T.
\]
This is the sum of a rank-$d$ and a rank-$1$ matrix, and thus has rank at least $d-1$ whenever $\x^T\z \neq y$. Since $\x^T \z =y$ is an algebraic variety in $\mathbb{R}^d \times \mathbb{R}^d \times \mathbb{R} $, the set $ \{(\x; (\z,\y)) \in \mathbb{R}^d \times \mathbb{R}^d \times \mathbb{R}  : \x^T \z =y \}$ has zero Lebesgue measure in $ \mathbb{R}^d \times \mathbb{R}^d \times \mathbb{R}$. For any fixed $\x$, the set $ \{ (\z,\y) \in  \mathbb{R}^d \times \mathbb{R}  : \x^T \z =y \}$ is a hyperplane in $\mathbb{R}^{d+1}$ hence of zero Lebesgue measure in $ \mathbb{R}^d \times \mathbb{R}$. Thus $ \nabla_{(\z,y)} \nabla_{\x} f(\x;(\z,y))$, when defined for any $\x$, is a least of rank $d-1$ a.e. on the data slice $\mathbb{R}^d \times \mathbb{R}$. \RSnew{Taking $D_1 \times D_2$ to be any closed ball in $\mathbb{R}^d \times \mathbb{R}^d \times \mathbb{R} $ such that its elements satisfy \( |\x^T \z -y |\ge \gamma >0 \)
 implies that there exists at least $d-1$ singular values of $\nabla_{(\z,y)} \nabla_{\x} f(\x;(\z,y)) $ greater than or equal to $\gamma$,} thereby satisfying \textbf{Assumption A3}. \RSnew{Since\footnote{\RSnew{Recall that $E_{<\gamma}(\z,y) $, defined in Lemma~\ref{seconvarlem2}, is the vector space spanned by the right singular vectors of $\nabla_{(\z,y)} \nabla_{\x} f(\x;(\z,y))$ where the singular values are less than $\gamma$.} } for any $ \w, (\z,y) \in D_1 \times D_2$ we have $ \dim(E_{<\gamma}(\z,y)) + \dim(E_{<\gamma}(\z,y)^\perp ) = d+1$ and $\dim(E_{<\gamma}(\z,y)^\perp ) \ge d-1 $, it must be that $r_{\gamma} \le 2 =r$ }, and since $f$ is analytic, \textbf{A1–A2} also hold.
\paragraph{One-layer neural networks.}  
Consider $f(\bm W,\bm v;(\bm x,y)) = \tfrac{1}{2}(\bm v^T\rho(\bm W\bm x)-y)^2$ with analytic, non-constant activation $\rho$. From Proposition~\ref{prop:NNex}, 
\[
\nabla_{\bm v} f = \big(\bm v^T\rho(\bm W\bm x)-y\big)\,\rho(\bm W\bm x).
\]
Differentiating with respect to $y$ gives
\[
\frac{\partial}{\partial y}\big(\nabla_{\bm v} f\big) = -\,\rho(\bm W\bm x).
\]
If $\rho$ is strictly positive (e.g. sigmoid), then $\rho(\bm W\bm x)\neq \bm 0$ for all $\bm x$, so \RSnew{for every compact set $D_1 \times D_2$ there  exists a $\gamma >0$ such that  $\nabla_{(\bm x,y)} \nabla_{\bm v} f$ has at least one singular value greater than $\gamma$ on any point in $D_1 \times D_2$ by continuity and strict positivity of $ \rho $. This singular value is from the component $ \frac{\partial}{\partial y}\big(\nabla_{\bm v} f\big)$.}  
If $\rho$ can vanish (e.g. $\tanh$), the zero set $\{\z : \rho(\W \z)= \bm 0\}$ is a proper real-analytic subset of $\mathbb{R}^d$ \RSnew{due to the facts that the function $g(\z) := \rho  \, ( \W\z)  := \tanh(\W \z)$ is a composition of analytic and affine function hence analytic in $\z$, $g$ is not zero everywhere for $\W \neq \mathbf{0}$, and that the zero set of non-constant analytic function is an analytic variety. Hence the zero set $\{\z : \rho(\W \z)= \bm 0\}$ is of Lebesgue measure zero.} \RSnew{Then by taking compact set $D_1 \times D_2$ away from the zero set $\{\z : \rho(\W \z)= \bm 0\}$, for some $\gamma>0$ the matrix operator $\nabla_{(\bm x,y)} \nabla_{\bm v} f$ has at least one singular value greater than $\gamma$ at any point in $D_1 \times D_2$ by continuity and strict positivity/ negativity of $\rho $.} Thus in either case $\nabla_{(\bm x,y)} \nabla_{\bm v} f$ has at least $1$ \RSnew{right singular vector with singular value greater than $\gamma$ for every $ (\W, \v), (\bm x,y) \in D_1 \times D_2$.}  
Therefore, \textbf{Assumption A3} is satisfied, and the same reasoning should extend to deeper networks with analytic activations.

\subsection{\RSnew{On the scaling limit of volume estimates for linear regression}}\label{degenerateassumptionsec_aux}

\RSnew{For linear regression, and  any fixed compact set $D_1 \times D_2$ containing 
$(\x_1,(\z_1,y_1))$ and $(\x_2,(\z_2,y_2))$, one may take
\(
L
=
2\sup_{(\x,(\z,y))\in D_1\times D_2}
\max\!\left\{\|\z\|,\sqrt{1+\|\x\|^2}\right\}
\)
as a Lipschitz constant. Indeed,
\begin{align}
&\bigl\|\nabla_{(\z,y)} \nabla_{\x} f(\x_1;(\z_1,y_1))
      - \nabla_{(\z,y)} \nabla_{\x} f(\x_2;(\z_2,y_2))\bigr\|_2 \nonumber\\
&\qquad=
\left\|
\bigl((\x_1^T \z_1 - y_1) - (\x_2^T \z_2 - y_2)\bigr)
\begin{bmatrix}\mathbf{I}_{d\times d} & \mathbf{0}_{d\times 1}\end{bmatrix}
+
\z_1 \begin{bmatrix}\x_1\\-1\end{bmatrix}^T
-
\z_2 \begin{bmatrix}\x_2\\-1\end{bmatrix}^T
\right\|_2 \nonumber\\
&\qquad\le
\left|
(\x_1^T \z_1 - y_1) - (\x_2^T \z_2 - y_2)
\right|
+
\left\|
\z_1 \begin{bmatrix}\x_1\\-1\end{bmatrix}^T
-
\z_2 \begin{bmatrix}\x_2\\-1\end{bmatrix}^T
\right\|_2 \nonumber\\
&\qquad\le
\|\z_1\|\,\|\x_1-\x_2\|
+
\|\x_2\|\,\|\z_1-\z_2\|
+
|y_1-y_2|
+
\|\z_1\|\,\|\x_1-\x_2\|
+
\sqrt{1+\|\x_2\|^2}\,\|\z_1-\z_2\| \nonumber\\
&\qquad\le
2\biggl(
\sup_{\x_2,\z_1\in D_1\times (D_2/y)}
\max\!\left\{\|\z_1\|,\sqrt{1+\|\x_2\|^2}\right\}
\biggr)
\bigl(\|\x_1-\x_2\|+\|\z_1-\z_2\|+|y_1-y_2|\bigr).
\end{align}
\\
For simplicity, let $D_1$ and the projection $(D_2/y)$ onto the $\z$ coordinates be cubes of side length $1$ in $\mathbb{R}^d$, centered at the origin, and suppose that for any $\x,(\z,y)\in D_1\times D_2$ we have
\(
|\x^T\z-y|\ge \gamma
\)
for some $\gamma>0$ independent of $d$. This condition is satisfied whenever
\(
|y|\ge \frac{d}{4}+\gamma,
\)
since
\(
|\x^T\z|
\le \|\x\|\,\|\z\|
\le \frac{d}{4}
\)
for $\x,\z$ in side-length-one cubes centered at the origin. Taking cubes rather than unit balls ensures that the volume of the compact set does not collapse to $0$ as $d\to\infty$.
}

\RSnew{
If, in addition,
\(
\frac{d}{4}+\gamma \le y \le \frac{d}{4}+\gamma+1,
\)
then $\operatorname{diam}(D_2)=\sqrt{d+1}$. Moreover, from the preceding estimate we may take
\(
L\le 2\sqrt{d+1}.
\)
Using the simplified bound \eqref{thm2simple1} from Theorem~\ref{seconvarthm1} with target data point $(\z^*,y^*)$, with fixed $\gamma\in(0,4)$ independent of $d$, and replacing $d$ by $d+1$, we obtain
\[
\begin{aligned}
\mu_2\bigl(S_{\epsilon}(\x,(\z^*,y^*))\bigr)
&\lesssim_{C_1(d)}
\sqrt{\pi(d+1)}
\left(
\frac{144\,L\,(d+1)\,(\operatorname{diam}(D_2))^2}{\pi e\gamma^2}
\right)^{(d+1)/2}
\epsilon^{\frac{d+1-r}{2}}\\
&\lesssim_{C_1(d)}
\sqrt{\pi}
\left(
\frac{288\,(d+1)^{\frac52+\frac{1}{d+1}}
\,\epsilon^{\frac{d-1}{d+1}}}{\pi e\gamma^2}
\right)^{(d+1)/2},
\end{aligned}
\]
where $r=2$ for linear regression and $C_1(d)\to 1$ as $d\to\infty$.
Next, replacing $d$ by $d+1$ in \eqref{eq:eps_bigO} gives
\[
\epsilon
=
\mathcal{O}\!\left(
(d+1)^{-\frac{3(1+a)}{2}}
(d+1)^{-\frac{(1+a)(d+2)}{d+1}}
\right)
=
\mathcal{O}\!\left(
(d+1)^{-\frac{5(1+a)}{2}}
(d+1)^{-\frac{1+a}{d+1}}
\right).
\]
Fix $a>0$, and set
\(
\epsilon
=
(d+1)^{-\frac{5(1+a)}{2}}
(d+1)^{-\frac{1+a}{d+1}}.
\)
Then
\[
\epsilon^{\frac{d-1}{d+1}}
=
(d+1)^{-\frac{5(1+a)}{2}}
(d+1)^{-\frac{1+a}{d+1}}
(d+1)^{\frac{5(1+a)}{d+1}}
(d+1)^{\frac{2(1+a)}{(d+1)^2}},
\]
and therefore
\[
\mu_2\bigl(S_{\epsilon}(\x,(\z^*,y^*))\bigr)
\lesssim_{C_1(d)}
\sqrt{\pi}
\left(
\frac{288}{\pi e\gamma^2}
(d+1)^{
-\frac{5a}{2}
+\frac{5+4a}{d+1}
+\frac{2(1+a)}{(d+1)^2}
}
\right)^{(d+1)/2}.
\]
Since
\(
-\frac{5a}{2}
+\frac{5+4a}{d+1}
+\frac{2(1+a)}{(d+1)^2}
\longrightarrow
-\frac{5a}{2}<0\)
as \(d\to\infty,
\)
the quantity inside parentheses tends to $0$. Hence
\(
\lim_{d\to\infty}\mu_2\bigl(S_{\epsilon}(\x,(\z^*,y^*))\bigr)=0.
\)
}

\section{Geometry of the set $K_1$ for a two layer neural network}\label{sectiongeomk1}

Consider the loss function in $\v,\W_1,\W_0,\z$ with the leaky ReLU activation function:
$$ f(\v,\W_1,\W_0;\z)  = \bigg(\v^T \rho(\W_1 \rho(\W_0 \z) ) -y\bigg)^2  .$$
The function $f : \mathbb{R}^{n_1} \times \mathbb{R}^{n_1 \times n_0} \times \mathbb{R}^{n_0 \times d} \times \mathbb{R}^{d} \to \mathbb{R}$ is non-smooth on the set given by
$$ V = \bigg( \bigcup_{i=1}^{n_0}  \bigg\{  ( \v,\W_1,\W_0,\z):  \langle [\W_0]_i, \z\rangle = 0 \bigg\}    \bigcup  \bigg( \bigcup_{i=1}^{n_1}\bigg\{  ( \v,\W_1,\W_0,\z): \langle  [\W_1]_i, \rho(\W_0\z)\rangle = {0} \bigg\} \bigg) \bigg)$$
and for any non-zero $\tilde{\v},\tilde{\W}_1,\tilde{\W}_0$, the restriction of $f$ on the slice $$   J_{s} =\{ (\tilde{\v},\tilde{\W}_1,\tilde{\W}_0, \z) : \z \in \mathbb{R}^d \}$$ is non-smooth on the closed subset $V_2$ of this slice $J_s$ where
\begin{align}
    V_2 &=  \bigg( \bigcup_{i=1}^{n_0}  \bigg\{  \z:  \langle [\tilde{\W}_0]_i, \z\rangle = 0 \bigg\}    \bigcup  \bigg(\bigcup_{i=1}^{n_1} \bigg\{ \z : \langle  [\tilde{\W}_1]_i, \rho(\tilde{\W}_0\z)\rangle = {0} \bigg\} \bigg) \bigg)  \nonumber \\
    & \subseteq  \bigg( \bigcup_{i=1}^{n_0}  \bigg\{  \z:  \langle [\tilde{\W}_0]_i, \z\rangle = 0 \bigg\}    \bigcup  \bigg(\bigcup_{q=1}^{2^{n_0}} \bigcup_{i=1}^{n_1} \bigg\{ \z : \langle  [\tilde{\W}_1]_i, \tilde{\W}_0\z \odot \boldsymbol{\alpha}_q \rangle = {0} \bigg\} \bigg) \bigg)  \nonumber
\end{align}
where $\boldsymbol{\alpha}_q \in \mathbb{R}^{n_{0}}  $ is a vector of the permutations of $1$'s and $\alpha$'s with permutations ranging from all $1$'s to all $\alpha$'s.
Then for compact, convex $D_2$ with non-empty interior and for any $\xi \in (0,R)$ where\footnote{Here $d_H(\cdot,\cdot)$ is the Hausdorff distance and $\z_c$ is the center of $D_2$.} $ R = d_H(\z_c, \partial D_2) $, we have $K_1 =  D_2 \backslash (V_2+ \mathcal{B}_{\xi}(\mathbf{0})) $. 
Moreover,  $V_2 \subset \mathbb{R}^d$ is the subset of union of at most $n_0+n_1 2^{n_0}$ hyperplanes passing through origin so the set $K_1$ is the complement of $\xi$ thickening of these hyperplanes.
$K_1$ is thus a subset of disjoint union of at most $n_0 + n_1 2^{n_0}$ cones embedded in the compact, convex set $D_2$. When $D_2$ is a closed ball with center at origin we have 
\begin{align}
    \text{vol}_{\mathbb{R}^d}((V_2 \cap D_2) + \mathcal{B}_{\xi}(\mathbf{0})) &\leq 2\xi\sum_{j=1}^{n_0+ n_1 2^{n_0}} \text{vol}_{\mathbb{R}^{d-1}}(\mathcal{B}_{R}(\mathbf{0})) - (n_0 +n_1 2^{n_0}-1) \text{vol}_{\mathbb{R}^{d}}(\mathcal{B}_{\xi}(\mathbf{0}))\nonumber\\
    &= \frac{2\xi (n_0+n_1 2^{n_0})\pi^{\frac{d-1}{2}} R^{d-1}}{\Gamma(\frac{d+1}{2})} -  \frac{(n_0 +n_1 2^{n_0}-1)\pi^{\frac{d}{2}} {\xi}^d}{\Gamma(\frac{d}{2}+1)} .\nonumber
\end{align}
Then,
\begin{align}
    \text{vol}_{\mathbb{R}^d}(K_1) & = \text{vol}_{\mathbb{R}^{d}}(\mathcal{B}_{R}(\mathbf{0})) - \text{vol}_{\mathbb{R}^d}((V_2 \cap D_2) + \mathcal{B}_{\xi}(\mathbf{0})) \nonumber \\
 \implies  \frac{\pi^{\frac{d}{2}} R^d}{\Gamma(\frac{d}{2}+1)} \geq   \text{vol}_{\mathbb{R}^d}(K_1)  &    \geq  \frac{\pi^{\frac{d}{2}} (R^d  {+}(n_0 +n_1 2^{n_0}-1) {\xi}^d )}{\Gamma(\frac{d}{2}+1)} -\frac{2\xi (n_0+n_1 2^{n_0})\pi^{\frac{d-1}{2}} R^{d-1}}{\Gamma(\frac{d+1}{2})}. \label{geomk1}
\end{align}
Since $ 0\leq \nu_1 < \text{vol}_{\mathbb{R}^{d}}(\mathcal{B}_{R}(\mathbf{0})) - \text{vol}_{\mathbb{R}^d}(K_1) = \text{vol}_{\mathbb{R}^d}((V_2 \cap D_2) + \mathcal{B}_{\xi}(\mathbf{0}))   $ we have the bound:
$$ \nu_1 <  \frac{2\xi (n_0+n_1 2^{n_0})\pi^{\frac{d-1}{2}} R^{d-1}}{\Gamma(\frac{d+1}{2})} -  \frac{(n_0 +n_1 2^{n_0}-1)\pi^{\frac{d}{2}} {\xi}^d}{\Gamma(\frac{d}{2}+1)} .$$
\RSnew{Letting $ \frac{2\xi (n_0+n_1 2^{n_0})\pi^{\frac{d-1}{2}} R^{d-1}}{\Gamma(\frac{d+1}{2})} -  \frac{(n_0 +n_1 2^{n_0}-1)\pi^{\frac{d}{2}} {\xi}^d}{\Gamma(\frac{d}{2}+1)} = 2\nu_1$ then gives a polynomial equation in $\xi$ as a function of $\nu_1$ and the positive root of this equation describes the explicit relation between $\xi$ and $\nu_1$.   }

\bibliographystyle{plainnat} 
\bibliography{reference} 

@article{dwork2014algorithmic,
  title={The algorithmic foundations of differential privacy},
  author={Dwork, Cynthia and Roth, Aaron and others},
  journal={Foundations and trends{\textregistered} in theoretical computer science},
  volume={9},
  number={3--4},
  pages={211--407},
  year={2014},
  publisher={Now Publishers, Inc.}
}

@inproceedings{thudi2022necessity,
  title={On the necessity of auditable algorithmic definitions for machine unlearning},
  author={Thudi, Anvith and Jia, Hengrui and Shumailov, Ilia and Papernot, Nicolas},
  booktitle={31st USENIX Security Symposium (USENIX Security 22)},
  pages={4007--4022},
  year={2022}
}

@book{rockafellar1998variational,
  title={Variational analysis},
  author={Rockafellar, R Tyrrell and Wets, Roger JB},
  year={1998},
  publisher={Springer}
}

@article{van2024weak,
  title={Weak convexity and approximate subdifferentials},
  author={van Ackooij, Wim and Atenas, Felipe and Sagastiz{\'a}bal, Claudia},
  journal={Journal of Optimization Theory and Applications},
  volume={203},
  number={2},
  pages={1686--1709},
  year={2024},
  publisher={Springer}
}

@article{alberti1994structure,
  title={On the structure of singular sets of convex functions},
  author={Alberti, Giovanni},
  journal={Calculus of Variations and Partial Differential Equations},
  volume={2},
  number={1},
  pages={17--27},
  year={1994},
  publisher={Springer}
}

@book{federer2014geometric,
  title={Geometric measure theory},
  author={Federer, Herbert},
  year={2014},
  publisher={Springer}
}

@book{lee2003smooth,
  title={Introduction to smooth manifolds},
  author={Lee, John M},
  year={2003},
  publisher={Springer}
}

@article{chien2024langevin,
  title={Langevin unlearning: A new perspective of noisy gradient descent for machine unlearning},
  author={Chien, Eli and Wang, Haoyu and Chen, Ziang and Li, Pan},
  journal={Advances in neural information processing systems},
  volume={37},
  pages={79666--79703},
  year={2024}
}

@article{sekhari2021remember,
  title={Remember what you want to forget: Algorithms for machine unlearning},
  author={Sekhari, Ayush and Acharya, Jayadev and Kamath, Gautam and Suresh, Ananda Theertha},
  journal={Advances in Neural Information Processing Systems},
  volume={34},
  pages={18075--18086},
  year={2021}
}

@book{nesterov2013introductory,
  title={Introductory lectures on convex optimization: A basic course},
  author={Nesterov, Yurii},
  volume={87},
  year={2013},
  publisher={Springer Science \& Business Media}
}

@article{mantelero2013eu,
  title={The EU Proposal for a General Data Protection Regulation and the roots of the ‘right to be forgotten’},
  author={Mantelero, Alessandro},
  journal={Computer Law \& Security Review},
  volume={29},
  number={3},
  pages={229--235},
  year={2013},
  publisher={Elsevier}
}

@inproceedings{bourtoule2021machine,
  title={Machine unlearning},
  author={Bourtoule, Lucas and Chandrasekaran, Varun and Choquette-Choo, Christopher A and Jia, Hengrui and Travers, Adelin and Zhang, Baiwu and Lie, David and Papernot, Nicolas},
  booktitle={2021 IEEE symposium on security and privacy (SP)},
  pages={141--159},
  year={2021},
  organization={IEEE}
}

@inproceedings{suliman2024data,
  title={Data Forging Is Harder Than You Think},
  author={Suliman, Mohamed and Kadhe, Swanand and Halimi, Anisa and Leith, Douglas and Baracaldo, Nathalie and Rawat, Ambrish},
  booktitle={Privacy Regulation and Protection in Machine Learning},
  year={2024}
}

@article{zhang2024verification,
  title={Verification of machine unlearning is fragile},
  author={Zhang, Binchi and Chen, Zihan and Shen, Cong and Li, Jundong},
  journal={arXiv preprint arXiv:2408.00929},
  year={2024}
}

@article{nguyen2022survey,
  title={A survey of machine unlearning},
  author={Nguyen, Thanh Tam and Huynh, Thanh Trung and Ren, Zhao and Nguyen, Phi Le and Liew, Alan Wee-Chung and Yin, Hongzhi and Nguyen, Quoc Viet Hung},
  journal={arXiv preprint arXiv:2209.02299},
  year={2022}
}

@inproceedings{neel2021descent,
  title={Descent-to-delete: Gradient-based methods for machine unlearning},
  author={Neel, Seth and Roth, Aaron and Sharifi-Malvajerdi, Saeed},
  booktitle={Algorithmic Learning Theory},
  pages={931--962},
  year={2021},
  organization={PMLR}
}

@article{gupta2021adaptive,
  title={Adaptive machine unlearning},
  author={Gupta, Varun and Jung, Christopher and Neel, Seth and Roth, Aaron and Sharifi-Malvajerdi, Saeed and Waites, Chris},
  journal={Advances in Neural Information Processing Systems},
  volume={34},
  pages={16319--16330},
  year={2021}
}

@inproceedings{cao2015towards,
  title={Towards making systems forget with machine unlearning},
  author={Cao, Yinzhi and Yang, Junfeng},
  booktitle={2015 IEEE symposium on security and privacy},
  pages={463--480},
  year={2015},
  organization={IEEE}
}

@inproceedings{baluta2023unforgeability,
  title={Unforgeability in stochastic gradient descent},
  author={Baluta, Teodora and Nikolic, Ivica and Jain, Racchit and Aggarwal, Divesh and Saxena, Prateek},
  booktitle={Proceedings of the 2023 ACM SIGSAC Conference on Computer and Communications Security},
  pages={1138--1152},
  year={2023}
}

@book{matousek2013lectures,
  title={Lectures on discrete geometry},
  author={Matousek, Jiri},
  volume={212},
  year={2013},
  publisher={Springer Science \& Business Media}
}

@article{matouvsek2001probabilistic,
  title={The probabilistic method},
  author={Jiri Matousek and Jan Vondrak},
  journal={Lecture Notes, Department of Applied Mathematics, Charles University, Prague},
  year={2001}
}

@article{pach2008state,
  title={State of the union (of geometric objects)},
  author={Pach, J{\'a}nos and Agarwal, Pankaj K and Sharir, Micha},
  journal={Surveys on Discrete and Computational Geometry-Twenty Years Later},
  pages={9--48},
  year={2008},
  publisher={American Mathematical Society}
}

@inproceedings{driemel2010approximating,
  title={Approximating the Fr{\'e}chet distance for realistic curves in near linear time},
  author={Driemel, Anne and Har-Peled, Sariel and Wenk, Carola},
  booktitle={Proceedings of the twenty-sixth annual symposium on Computational geometry},
  pages={365--374},
  year={2010}
}

@article{bertoin2021numerical,
  title={Numerical influence of ReLU’(0) on backpropagation},
  author={Bertoin, David and Bolte, J{\'e}r{\^o}me and Gerchinovitz, S{\'e}bastien and Pauwels, Edouard},
  journal={Advances in Neural Information Processing Systems},
  volume={34},
  pages={468--479},
  year={2021}
}

@inproceedings{berner2019towards,
  title={Towards a regularity theory for ReLU networks--chain rule and global error estimates},
  author={Berner, Julius and Elbr{\"a}chter, Dennis and Grohs, Philipp and Jentzen, Arnulf},
  booktitle={2019 13th International conference on Sampling Theory and Applications (SampTA)},
  pages={1--5},
  year={2019},
  organization={IEEE}
}

@article{blumenson1960derivation,
  title={A derivation of n-dimensional spherical coordinates},
  author={Blumenson, LE},
  journal={The American Mathematical Monthly},
  volume={67},
  number={1},
  pages={63--66},
  year={1960},
  publisher={JSTOR}
}

@article{pawelczyk2024machine,
  title={Machine unlearning fails to remove data poisoning attacks},
  author={Pawelczyk, Martin and Di, Jimmy Z and Lu, Yiwei and Sekhari, Ayush and Kamath, Gautam and Neel, Seth},
  journal={arXiv preprint arXiv:2406.17216},
  year={2024}
}

@article{chowdhury2024towards,
  title={Towards scalable exact machine unlearning using parameter-efficient fine-tuning},
  author={Chowdhury, Somnath Basu Roy and Choromanski, Krzysztof and Sehanobish, Arijit and Dubey, Avinava and Chaturvedi, Snigdha},
  journal={arXiv preprint arXiv:2406.16257},
  year={2024}
}

\end{document}